\documentclass{article}

\PassOptionsToPackage{numbers, compress}{natbib}
\usepackage[preprint]{neurips_2019}
\usepackage[utf8]{inputenc} % allow utf-8 input
\usepackage[T1]{fontenc}    % use 8-bit T1 fonts
\usepackage{hyperref}       % hyperlinks
\usepackage{url}            % simple URL typesetting
\usepackage{booktabs}       % professional-quality tables
\usepackage{amsfonts}       % blackboard math symbols
\usepackage{nicefrac}       % compact symbols for 1/2, etc.
\usepackage{microtype}      % microtypography
\usepackage{makecell}

\usepackage{algorithm}
\usepackage{algorithmicx}
\usepackage{algpseudocode}
\usepackage{subfigure}
\usepackage{hyperref}

\usepackage{amsmath,amssymb,amsfonts,amsthm,mathtools,bm}
\usepackage{enumitem}
\usepackage{multirow}
\usepackage{bm}
\usepackage{array}
\usepackage{color}
\usepackage{xcolor}
\usepackage{wrapfig}
\usepackage{multicol}
\usepackage{caption}
\usepackage{diagbox}

\usepackage{macro}

\title{Stochastic Gradient Methods ~\\with Block Diagonal Matrix Adaptation}

\author{%
  Jihun Yun \\
  KAIST\\
  \texttt{arcprime@kaist.ac.kr} \\
  \And
  Aur\'elie C. Lozano \\
  IBM T.J. Watson Research Center \\
  \texttt{aclozano@us.ibm.com} \\
  \And
  Eunho Yang \\
  KAIST, AITRICS \\
  \texttt{eunhoy@kaist.ac.kr}
}

\begin{document}

\maketitle

%%%%% Abstract %%%%%
\begin{abstract}
	Adaptive gradient approaches that automatically adjust the learning rate on a per-feature basis have been very popular for training deep networks. %
	This rich class of algorithms includes \textsc{Adagrad}, \textsc{RMSprop}, \textsc{Adam}, and recent extensions. %\textsc{Adadelta}, \textsc{Nadam}, and recent extensions. % such as \textsc{AMSGrad}, \textsc{AdaFom}, \textsc{Yogi}, and \textsc{AdaBound}. %which have been proposed to alleviate issues on convergence and out-of-sample generalization. %compared with \textsc{Sgd}. 
	All these algorithms have adopted diagonal matrix adaptation, due to the prohibitive computational burden of manipulating full matrices in high-dimensions. %
	In this paper, we show that block-diagonal matrix adaptation can be a practical and powerful solution that can effectively utilize structural characteristics of deep learning architectures, and significantly improve convergence and out-of-sample generalization. %
	%We present a general adaptive gradient framework with block-diagonal matrix updates via coordinate grouping that includes counterparts of the aforementioned algorithms,
	We present a general framework with block-diagonal matrix updates via coordinate grouping, which includes counterparts of the aforementioned algorithms, prove their convergence in non-convex optimization, highlighting benefits compared to diagonal versions. %, along with their differences with the diagonal version. %, and specialize our theorem for the block diagonal matrix case. 
	In addition, we propose an efficient spectrum-clipping scheme that benefits from superior generalization performance of \textsc{Sgd}. %and further reduce the computational cost.
	%An additional benefit is that one can update parameters in each sub-matrix in a parallel manner with GPU-friendly implementations. %
	Extensive experiments reveal that %that in terms of both training and generalization accuracy, 
	block-diagonal approaches achieve state-of-the-art results on several deep learning tasks, and can outperform adaptive diagonal methods, vanilla \textsc{Sgd}, as well as a modified version of full-matrix adaptation proposed very recently. %even for small group sizes.
	%Interestingly, our results also show that block diagonal matrix updates alleviate an oscillatory behavior of their diagonal counterparts.
	%
	%\textcolor{red}{[QUESTION] Do we need to write about oscillatory behavior in the abstract? } \textcolor{blue}{Jihun: I think It's enough to mention in the contribution!}
\end{abstract}

%%%%% introduction %%%%%%
\section{Introduction}

Stochastic gradient descent (\textsc{Sgd})~\cite{robbins1951} is a dominant approach for training large-scale machine learning models such as deep networks. At each iteration of this iterative method, the model parameters are updated in the opposite direction of the gradient of the objective function typically evaluated on a mini-batch, with step size controlled by a \emph{learning rate}.  
%Choosing the learning rate judiciously is critical. %:  too small a value may result in a painfully long training process that could get stuck, 
%(e.g. in numerous suboptimal local minima or in plateaus surrounding saddle points)
%while too large a value may make the training process highly unstable or produce sub-optimal parameters too quickly.
%As adjusting the learning rate is very difficult and scaling the gradient uniformly in all directions might not be desirable (e.g. when the data is sparse and features have different frequencies), 
While vanilla \textsc{Sgd} uses a common learning rate across coordinates (possibly varying across time), several \emph{adaptive learning rate} algorithms have been developed that scale the gradient coordinates by square roots of some form of average of the squared values of past gradients coordinates.  The first key approach in this class, \textsc{Adagrad}~\cite{duchi2011,mcmahan2010}, uses a per-coordinate learning rate based on squared past gradients, %yielding a large learning rate for coordinates that have not yet been updated enough and a small learning rate for coordinates that have been relatively more updated, 
and has been found to outperform vanilla \textsc{Sgd} on sparse data. However, in non-convex dense settings where gradients are dense, performance is degraded, since the learning rate shrinks too rapidly due to the accumulation of all past squared gradient in its denominator. 
To address this issue, variants of \textsc{Adagrad} have been proposed that use the exponential moving average (EMA) of past squared gradients to essentially restrict the window of accumulated gradients to only few recent ones. Examples of such methods include \textsc{Adadelta}~\cite{zeiler2012}, \textsc{RMSprop} ~\cite{tieleman2012}, \textsc{Adam}~\cite{kingma2015}, and \textsc{Nadam}~\cite{dozat2016}. %These methods have been successfully employed in various deep learning applications. 

Despite their popularity and great success in some applications, the above EMA-based adaptive approaches have raised several concerns. \cite{wilson2017} studied their out-of-sample generalization and observed that on several popular deep learning models their generalization is worse than vanilla \textsc{Sgd}. Recently ~\cite{reddi2018} showed that they may not converge to the optimum (or critical point) even in simple convex settings with constant minibatch size, and noted that the effective learning rate %(i.e. the learning rate parameter divided by square root of an exponential moving average of squared past gradients)
of EMA methods can increase fairly quickly while for convergence it should decrease or at least have a controlled increase over iterations. %, even if the learning rate parameter itself is decreased over iterations. 
%To fix this issue, ~\cite{reddi18} proposed \textsc{AMSGrad} that uses the maximum of past squared gradients for normalizing the running average of the gradient.  
\textsc{AMSGrad}, proposed in \cite{reddi2018} to fix this issue, did not yield conclusive improvements in terms of generalization ability. To simultaneously benefit from the generalization ability of vanilla \textsc{Sgd} and the fast training of adaptive approaches, \cite{luo2019}  recently proposed  \textsc{AdaBound} and \textsc{AMSBound} as variants of  \textsc{Adam} and \textsc{AMSGrad}, which employ dynamic bounds on learning rates to guard against extreme learning rates. % and illustrated via experiments that these approaches alleviate the generalization issues of EMA-based approaches. 
~\cite{chen2019} introduced \textsc{AdaFom} that only add momentum to the first moment estimate while using the same second moment estimate as \textsc{AdaGrad.}
\cite{zaheer2018} showed that increasing minibatch sizes enables convergence of \textsc{Adam}, and proposed \textsc{Yogi} which employs additive adaptive updates to prevent informative gradients from being forgotten too quickly.  %Compared with \textsc{Adam}, \textsc{Yogi} was shown to have similar theoretical convergence guarantees and superior performance in benchmark experiments.
%\textcolor{red}{[TODO: (Aurelie)] Introduce ADAFOM}.

We note that all the aforementioned adaptive algorithms deal with adaptation in a limited way, namely they only employ \emph{diagonal} information of Gradient of Outer-Product ($g_t g_t^T$ where $g_t$ is the stochastic gradient at time $t$, a.k.a. GOP). %adaptation, namely they employ the diagonal of the outer-product matrices formed by the gradient sequences. 
Though initially discussed in~\cite{duchi2011}, \emph{full} matrix adaptation has been mostly ignored due to its prohibitive computational overhead in high-dimensions. The only exception is the \textsc{GGT} algorithm~\cite{agarwal2018}; it uses a modified version of full-matrix \textsc{AdaGrad} with exponentially attenuated gradient history as in \textsc{Adam}, but truncated to a small window parameter so the preconditioning matrix becomes low rank thereby computing its inverse square root effectively. %\emph{Full} matrix adaptation has been discussed ~\cite{duchi11}, and investigating whether it can improve performance was left as future work.
%\textcolor{blue}{Preliminary experiments~\cite{agarwal18}  indicated that \textsc{GGT} trains faster than diagonal adaptation approaches, but with mixed results on generalization ability.} %\textcolor{red}{[TODO:Aurelie] Check how we can be firm but polite...}  %One concern is that the gradient truncation in \textsc{GGT} might exacerbate fast forgetting of informative past gradients, which can be an issue especially in sparse settings where gradients are rarely nonzero~\cite{reddi18}. 

\textbf{Contributions.}  In this paper, we revisit open questions on \textsc{AdaGrad} in ~\cite{duchi2011} and propose an extended form of \textsc{Sgd} learning with \emph{block-diagonal} matrix adaptation that can better utilize the structural characteristics of deep learning architectures. We also show that it can be a practical and powerful solution, which can actually outperform vanilla \textsc{Sgd} and achieve state-of-the-art results on several deep learning tasks. More specifically, the main contributions of this paper are as follows:
\begin{itemize}[topsep=0pt,itemsep=1mm, parsep=0pt] %[itemsep=1mm, parsep=0pt]
	\item We provide an EMA-based \textsc{Sgd} framework with \emph{block diagonal matrix adaptation} via \emph{coordinate grouping}. This framework takes advantage of richer information on interactions across different gradient coordinates, while significantly relaxing the expensive computational cost of full matrix adaptation in large-scale problems. In addition, we introduce several grouping strategies that are practically useful for deep learning problems.
	
	\item We provide the first convergence analysis of our framework in the non-convex setting, and highlight difference and benefits compared with diagonal versions. 
	%Our first theorem covers the fixed minibatch case. %Based on this, we empirically show that theoretical upper bound for block diagonal matrix converges faster than the upper bound for diagonal. 
	%Our second theorem guarantees convergence of versions with increasing minibatch size, including, e.g., block diagonal versions of \textsc{RMSprop} or \textsc{Adam}.
	%\textcolor{red}{[TODO: Aurelie] Be more precise on the ``benefits'' compared with diagonal case.}
	
	\item In addition, we introduce \emph{spectrum-clipping}, a non-trivial extension of \cite{luo2019} for our block-diagonal adaptation framework. Spectrum-clipping allows the block diagonal matrix to become a constant multiple of the identity matrix in the latter part of training, similarly to vanilla \textsc{Sgd}.
	% reduce computational cost and to benefit all the higher generalizations of vanilla \textsc{Sgd}. Spectrum-clipping allows the block diagonal matrix to be a constant multiple of the identity matrix in the latter of the training, giving the shape of the SGD. %This reduces the computational complexity of computing the matrix square root and inverse, and allows for a good generalization ability of SGD. Another benefit is that we can update each sub-block of the block diagonal matrix in a parallel manner with GPU-friendly implementations.
	
	\item We evaluate the training and generalization ability of our approaches on popular deep learning tasks. Our experiments reveal that block diagonal methods perform better than diagonal approaches, even for small grouping sizes, and can also outperform vanilla \textsc{Sgd} and the modified version of full-matrix adaptation \textsc{GGT}. Interestingly, our empirical studies also show that block diagonal matrix updates alleviate an oscillatory behavior present in diagonal versions. %\textcolor{red}{[Todo: Aurelie] State that we outperform SGD on test accuracy at tail.}
\end{itemize}
\textbf{Notation.} For any vectors $x,y \in \mathbb{R}^{d}$, we assume that all the operations are element-wise, such as $xy$, $x/y$, and $\sqrt{x}$. We denote $[x]_i$ to be the $i$-th coordinate of vector $x$. For a vector $x$, $\|x\|_p$ denotes the vector $p$-norm, and $\|x\|$ is $\|x\|_2$ if not specified. For a matrix $A$, $\matnorm{A}{p}$ indicates the matrix $p$-norm for matrix $A$, $\lambda(A)$ returns a eigenvalue list of $A$. $\lambda_{\mathrm{min}}(A)$ and $\lambda_{\mathrm{max}}(A)$ denote the minimum and maximum eigenvalue of $A$ respectively, and $\kappa(A)$ represents the condition number of $A$. The function $\mathrm{Clip}(x, a, b)$ represents clipping $x$ element-wise with the interval $I = [a, b]$.

% Recently, several first-order and second-order methods have been proposed that are guaranteed to converge to local minima under certain conditions [1, 4, 14, 28]. However, these methods are computationally expensive or exhibit slow convergence in practice, making them unsuitable for large-scale settings.

%\paragraph{Related Work.}  Aside from the aforementioned literature, pertinent work on analyzing \textsc{Sgd}-based approaches include \cite{ghadimi13} which provide an analysis of \textsc{Sgd} and accelerated variants for smooth nonconvex problems, and \cite{hazan15,bernstein18} which analyze normalized variants of \textsc{Sgd.} Aside from the analyses in~\cite{reddi18,zaheer18}, convergence studies of adaptive methods referenced above have mostly focused on the convex setting. %~\cite{duchi11,mcmahan10,zeiler12}}  [ 40, 34, 16].  
%\textcolor{red}{Do we need this related work paragraph?} \textcolor{red}{I think we should include this reference \url{https://arxiv.org/abs/1806.02958} at least. Is it better to make additional section or provide the references in the intro part?}

%%%%% setup %%%%%
\section{Adaptive Gradient Methods with Block Diagonal Matrix Adaptations via Coordinate Partitioning}

%\subsection{Background}
\begin{figure}[t]
	\centering
	\subfigure[input-neuron]{\includegraphics[width=0.22\linewidth]{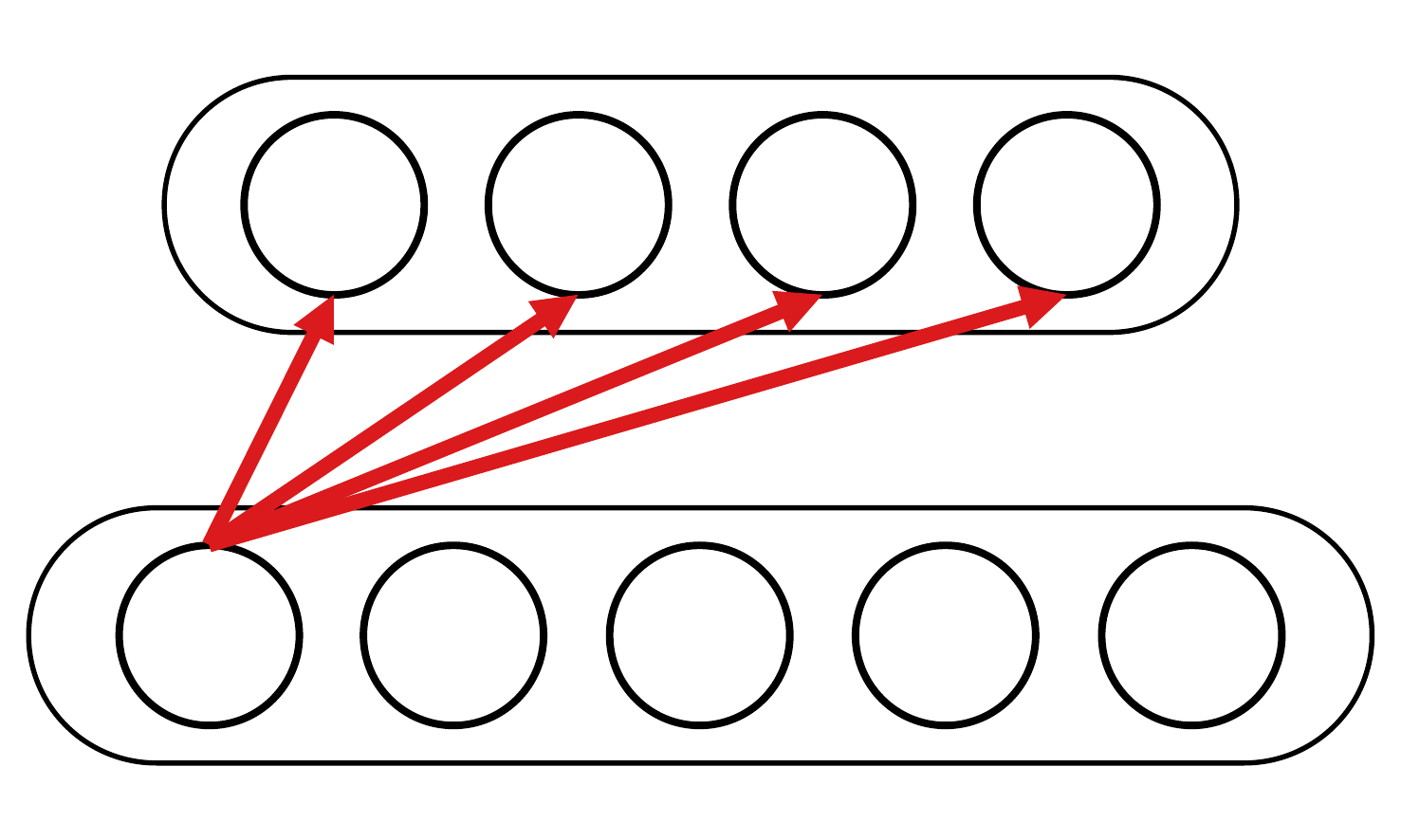}}
	\subfigure[output-neuron]{\includegraphics[width=0.22\linewidth]{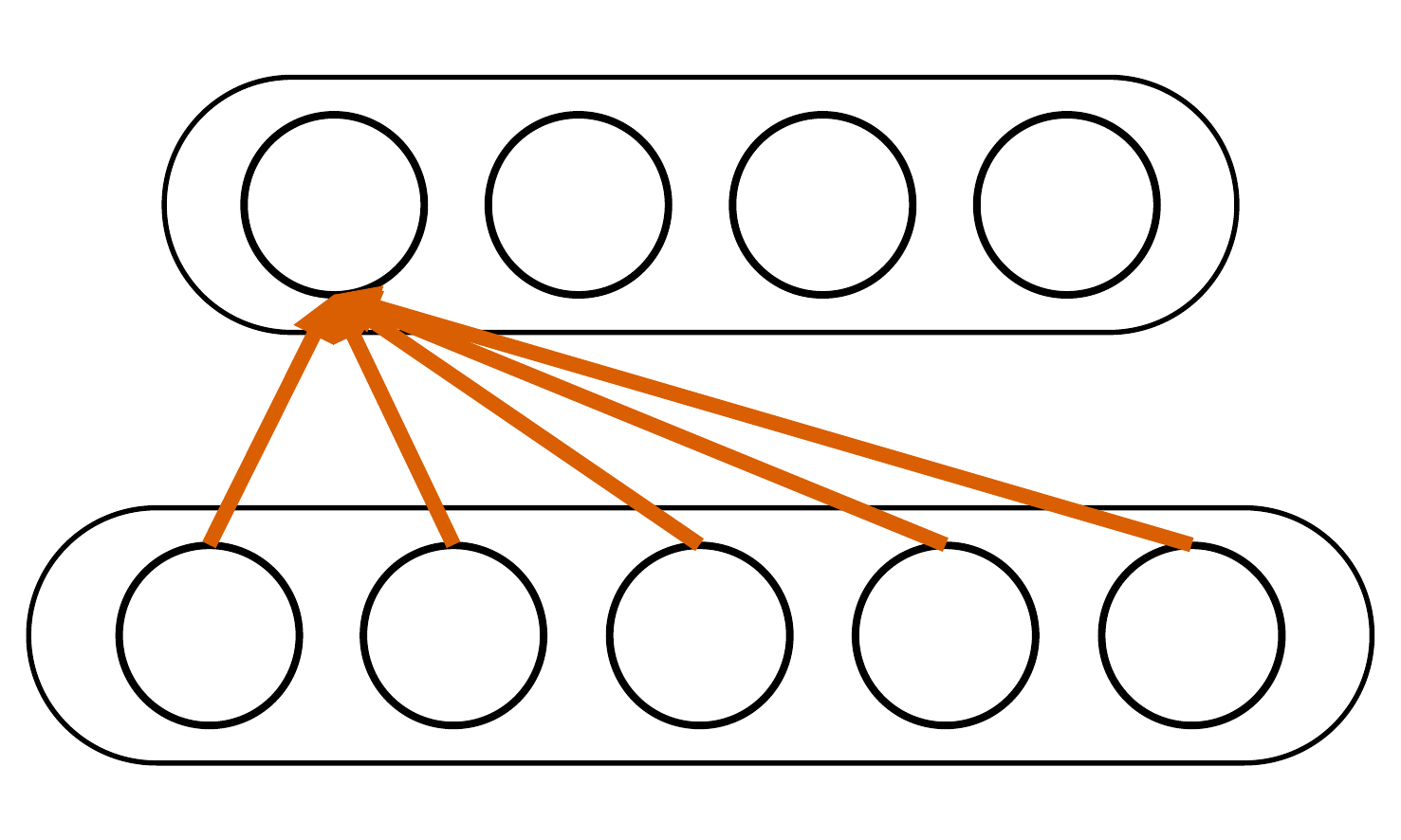}}
	\subfigure[partially group]{\includegraphics[width=0.22\linewidth]{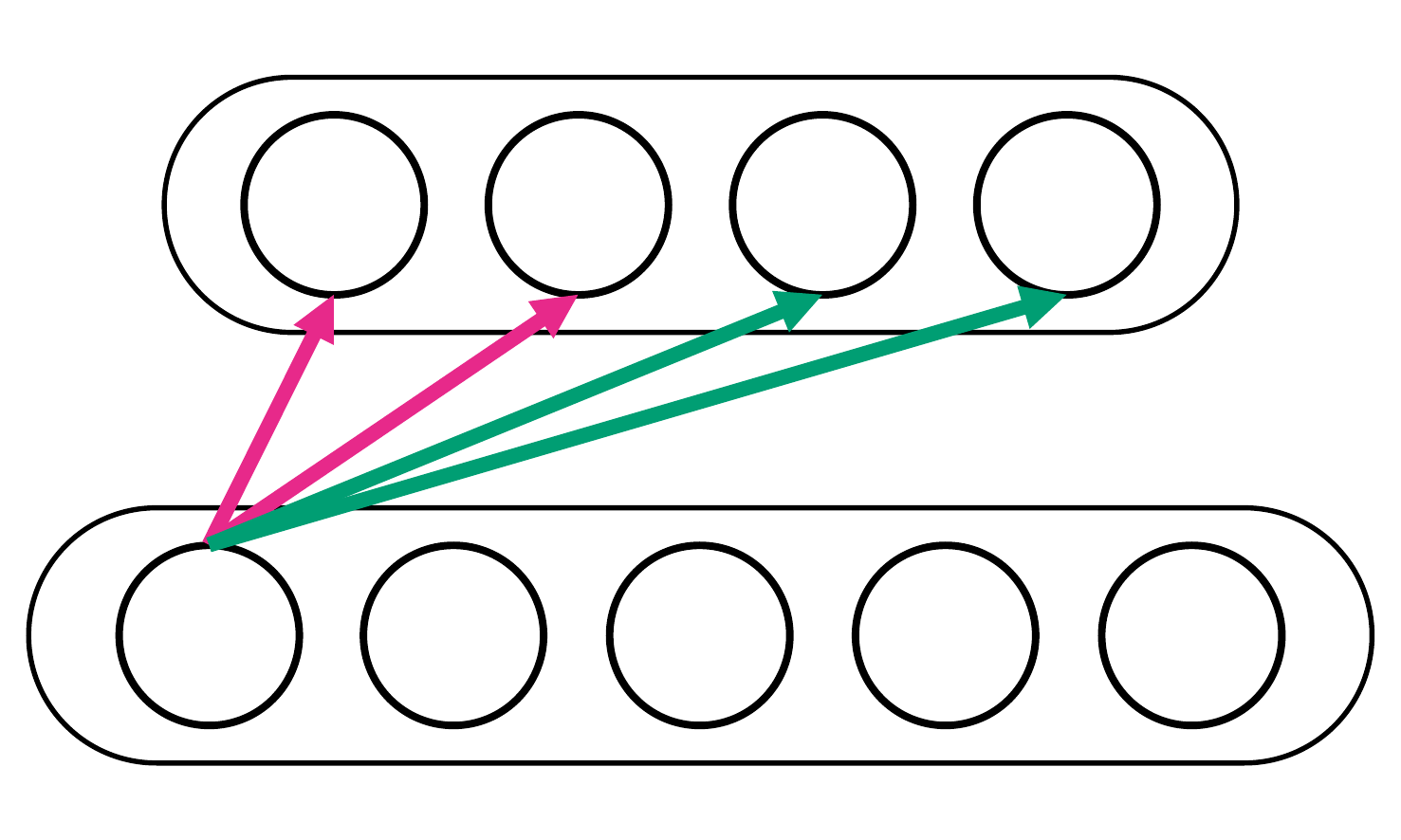}}\subfigure[filter-wise group]{\includegraphics[width=0.29\linewidth]{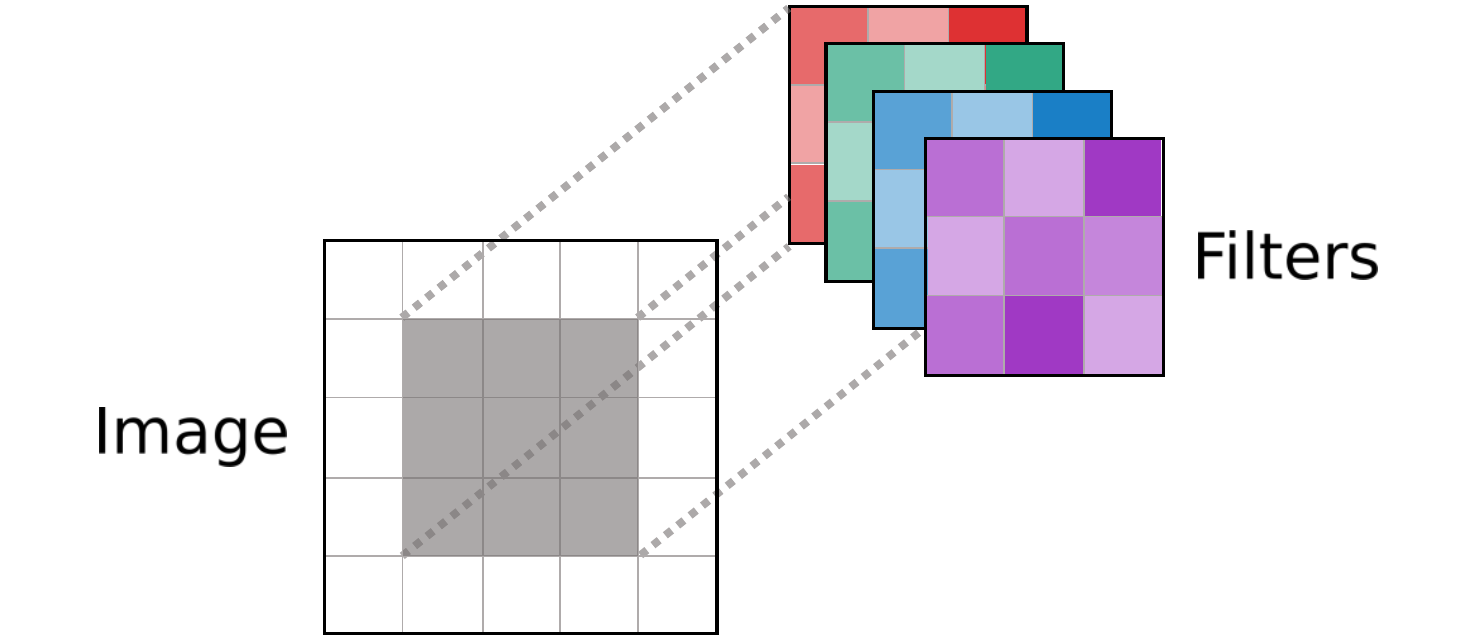}}
	\caption{Possible examples of grouping for fully-connected layers and convolutional layers.}
	\label{fig:grouping}
\end{figure}
%For deterministic optimization, a Hessian based second-order method such as Newton's method ({\color{red} add ref if possible}) is known to be effective when the variables have very different scales or when the curvature of the objective function varies widely in each direction. However, it is notoriously expensive since it involves the inverse of Hessian at every iterative step.
%
%for many high-dimensional problems where the number of parameters is huge, computing the Hessian is intractable, making it difficult to apply such methods to modern applications such as deep learning.
%
In the context of stochastic optimization, \cite{duchi2011} proposed a full-matrix variant of \textsc{AdaGrad}. This version employs a preconditioner which exploits first-order information only, via the sum of outer products of past gradients:
\begin{align}\label{eqn_adagrad}
g_t = \nabla f(x_t), \quad G_t = G_{t-1} + g_t g_t^T, \quad x_{t+1} = x_t - \alpha_t (G_t^{1/2} + \delta I)^{-1} g_t 
\end{align}
where $g_t$ is a stochastic gradient at time $t$, $\alpha_t$ is a step-size, and $\delta$ is a small constant for numerical stability. \cite{duchi2011} presented theoretical regret bounds for \eqref{eqn_adagrad} in the convex setting. However, this approach is quite expensive due to $G_t^{1/2}$ term, so they proposed to only use the diagonal entries of $G_t.$ %  so that one can trivially compute the square root of a huge matrix. 
Popular adaptive \textsc{Sgd} methods for training deep models such as \textsc{RMSprop/Adam} are based on such diagonal adaptation. Their general form and designs of the second-order momentum are given in the appendix. % due to space constraints.     
%
%
%Computing the exact GOP matrix requires prohibitive computational cost, the most of the methods to train deep networks today are approximating with diagonal matrix by taking only diagonal entries of the GOP matrix. This kind of algorithms can be expressed in a general form as algorithm \ref{alg:adaptive_diag}. We can classify each algorithm according to the design schemes for second momentum estimate $\widehat{v}_t$. The table \ref{tab:diag_general_framework} shows each algorithm depending on the function $h_t$. Based on diagonal approximation, similarly, the general framework using exact GOP matrix can be represented as algorithm \ref{alg:adaptive_full}. As in the diagonal case, the update rule changes according to the function $H_t$ which designs $\widehat{V}_t$, and the table \ref{tab:full_general_framework} demonstrates each algorithm.
%
%
%commonly employed for practical reasons as in \textsc{RMSProp} or \text{Adam}. 
%
%
%\textbf{Toy MLP example: Full GOP adaptation vs. Diagonal approximation.} 

\cite{duchi2011} also discussed the case where full-matrix adaptation can converge faster than its popular diagonal counterpart. Motivated by this, we first check through a toy MLP experiment whether preconditioning with exact GOP \eqref{eqn_adagrad} can be more effective even in the deep learning context. 
\iffalse
We consider a structured MLP (two nodes in two hidden layers followed by single output). For hidden units, we use ReLU activation \citep{nair2010} and the sigmoid unit for the binary output. We generate $n=10$ i.i.d. observations: $x_i \sim \mathcal{N}(0, I_2)$ and $y_i$ from this two layered MLP given $x_i$.  
\fi
Our experiment shows that one  can achieve faster convergence and better objective values by considering the interaction between gradient coordinates \eqref{eqn_adagrad}. Details are provided in appendix due to space constraint. The caveat here is that using full GOP adaptation in real deep learning optimization problems is computationally intractable due to the square root operator in \eqref{eqn_adagrad}. Nevertheless, is the best choice to simply use diagonal approximation given the available computation budget? What if we can afford to pay a little bit more for our computations?

\textbf{Main Algorithm: Adaptive \textsc{SGD} with Block Diagonal Adaptation.} %
\begin{figure}[t]
	\begin{minipage}[t]{0.52\textwidth}
		\centering
		\subfigure[From the same layer]{\includegraphics[width=0.49\linewidth,height=0.8in]{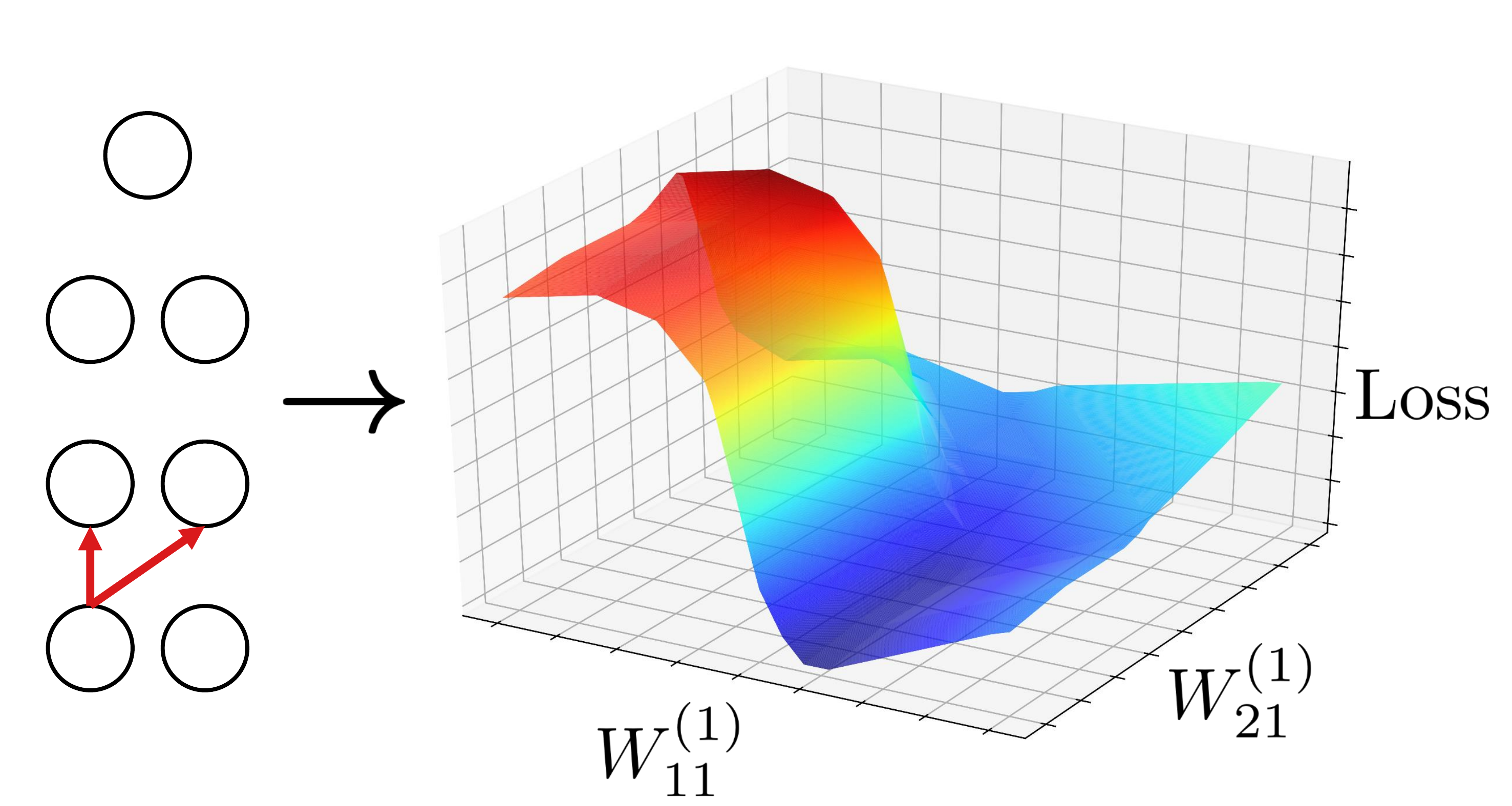}}
		\subfigure[From different layers]{\includegraphics[width=0.49\linewidth,height=0.8in]{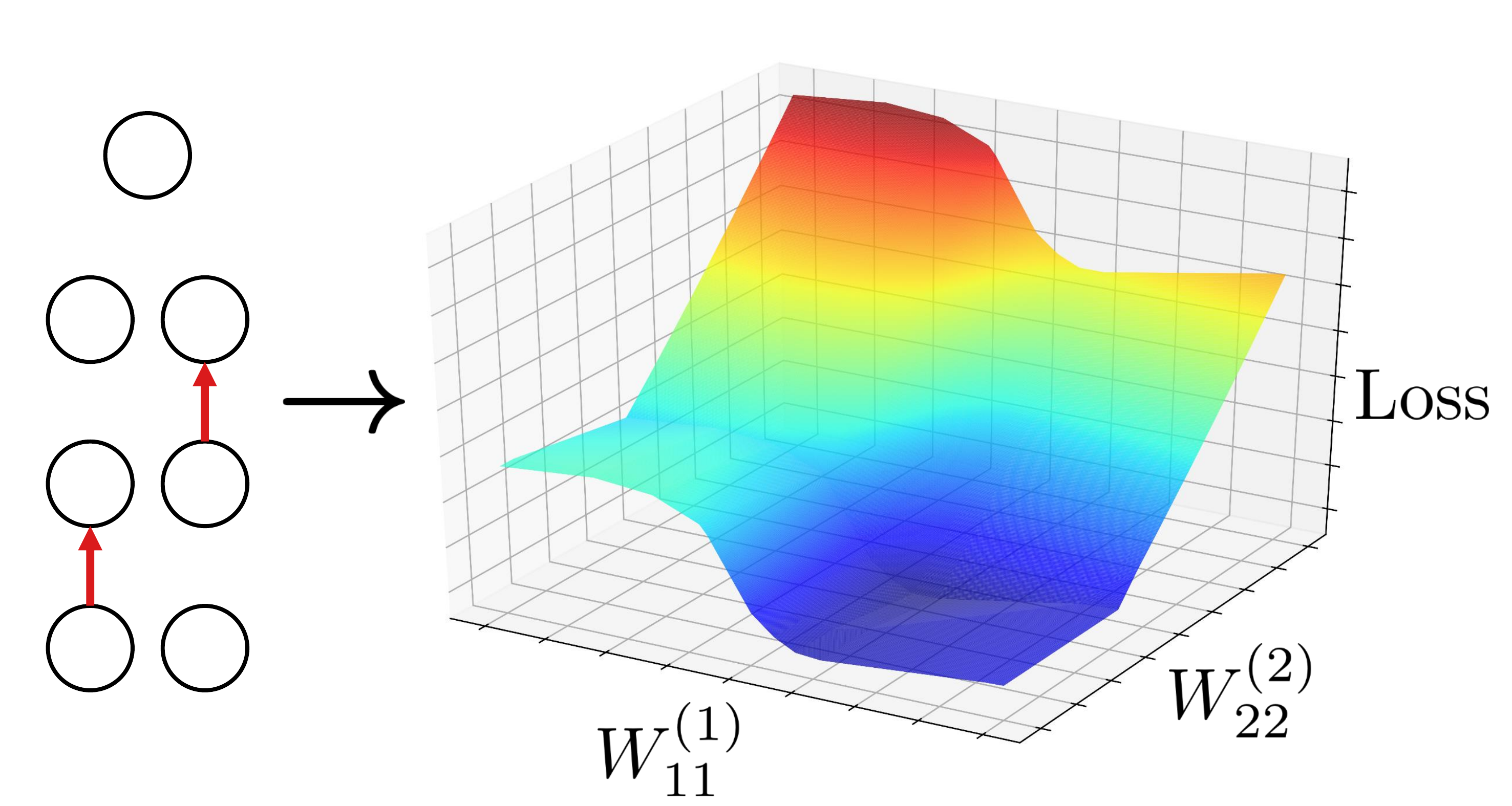}}
		\subfigure[\textsc{RMSprop}-Diag]{\includegraphics[width=0.49\linewidth]{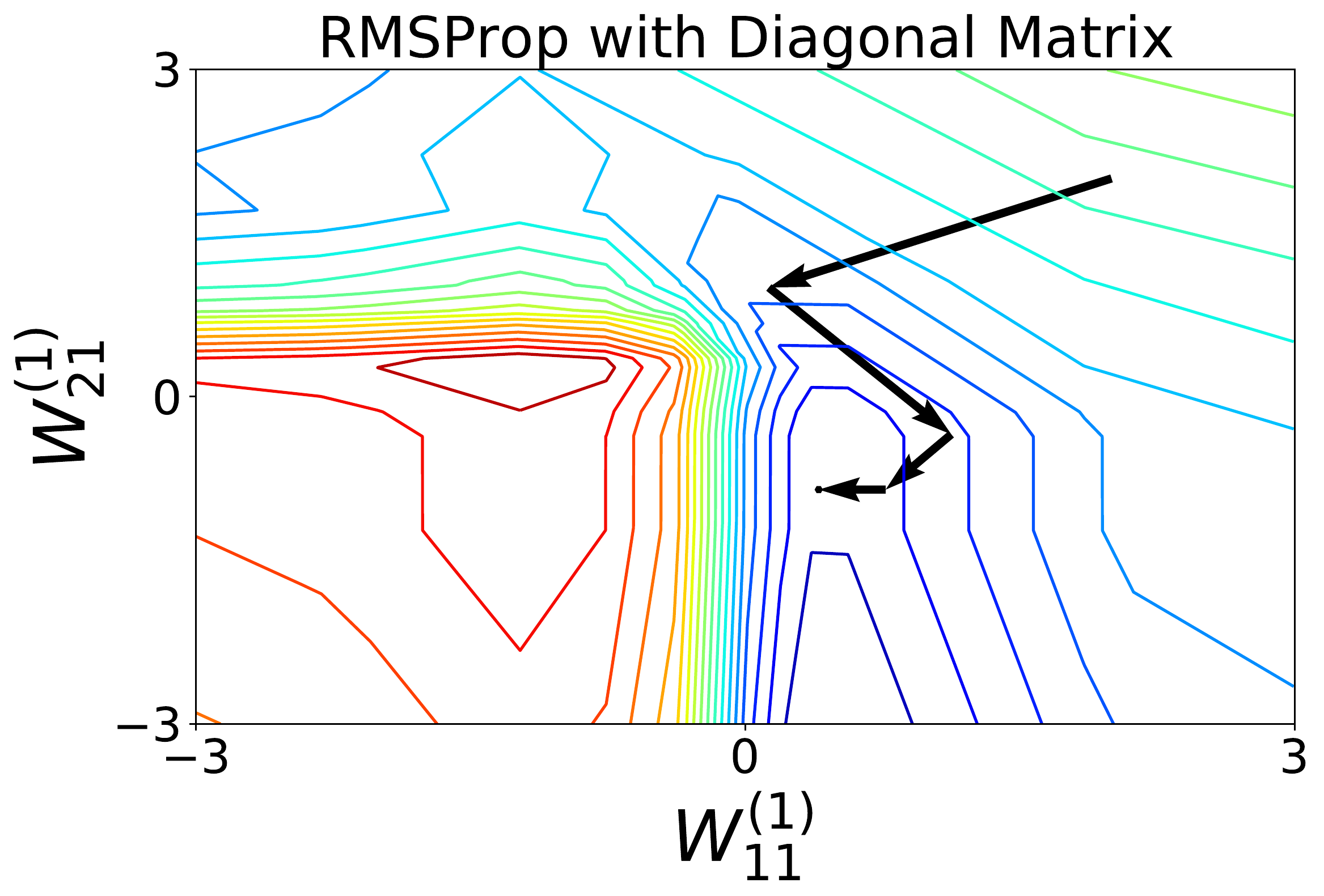}}
		\subfigure[\textsc{RMSprop}-Group]{\includegraphics[width=0.49\linewidth]{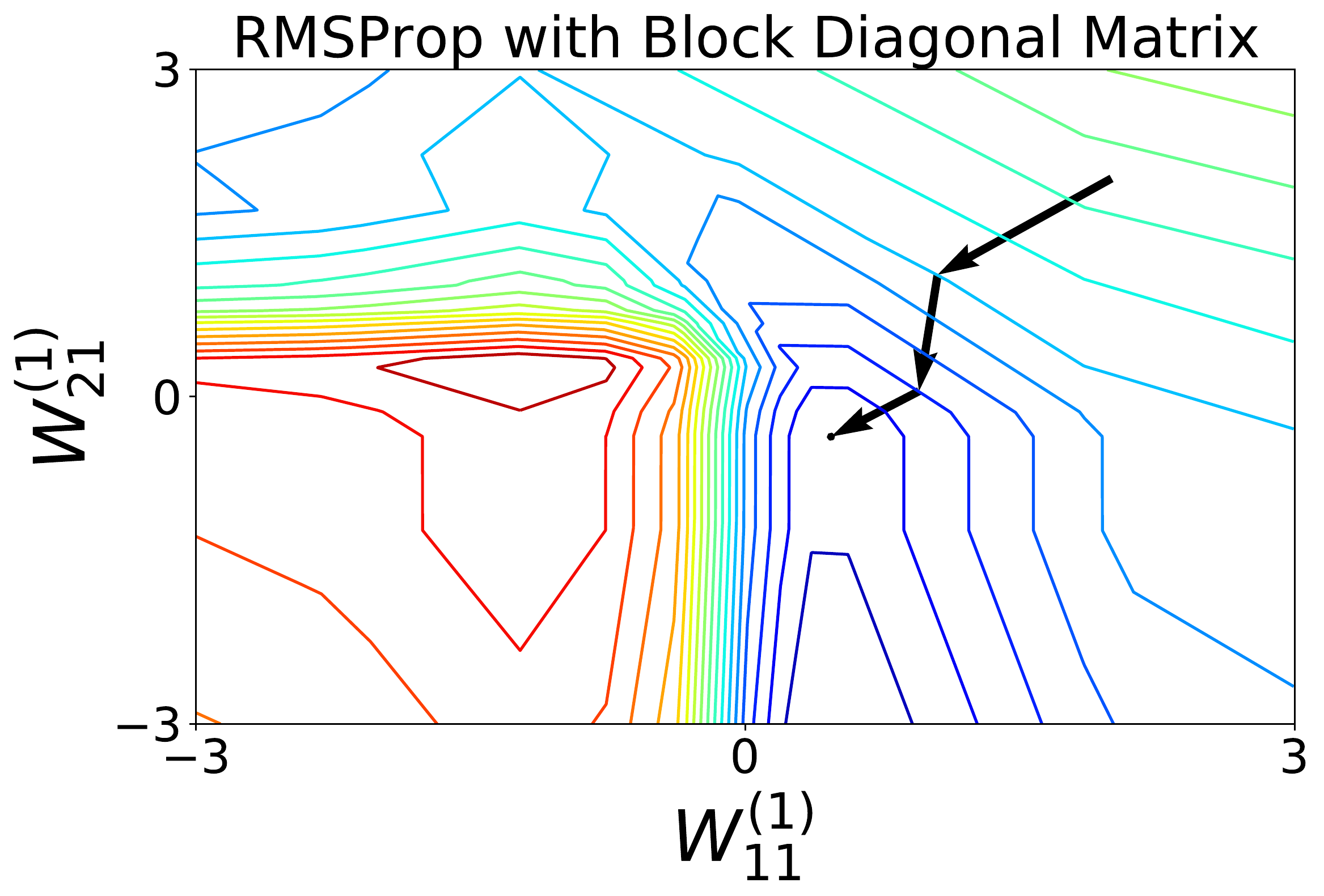}}
		\caption{Top: Loss landscape for different grouping methods; Bottom: optimization trajectories on the loss landscape (a).}
		\label{fig:landscape_and_trajectories}
	\end{minipage}
	\hfill
	\begin{minipage}[t]{0.48\textwidth}
		\begin{algorithm}[H]
			\caption{\small Adaptive Gradient Methods with Block Diagonal Matrix Adaptation}
			\label{alg:block_diagonal_update}{\footnotesize
				\begin{algorithmic}
					\State {\bfseries Input:} Stepsize $\alpha_t$, initial point $x_1 \in \mathbb{R}^{d}$, $\beta_1 \in [0, 1)$. The function $H_t$ designs $\widehat{V}_t$ with $r$ blocks.
					\State {\bfseries Initialize:} $m_0 = 0$, $\V_0 = 0$. Assume $\beta_1 \geq \beta_{1,t}$.
					\For{$t = 1, 2, \ldots, T$}
					\State{Draw a minibatch sample $\xi_t$ from $\mathbb{P}$}
					\State{$i \text{ (offset)} \gets 0, \quad  G_t \gets 0$} 
					\Let{$g_t$}{$\nabla f(x_t), ~~ m_t \gets \beta_{1,t} m_{t-1} + (1 - \beta_{1,t}) g_t$}
					\For{each group index $j = 1, 2, \ldots, r$}
					\Let{$g_t^{(j)}$}{$g_t[i:i + n_j]$}
					\Let{$G_t[i:i + n_j, i: i+ n_j]$}{$g_t^{(j)} \big(g_t^{(j)}\big)^T$}
					\Let{$i$}{$i + n_j$}
					\EndFor
					\Let{$\V_t$}{$H_t(G_1, \cdots, G_t)$}
					\Let{$x_{t+1}$}{$x_t - \alpha_t (\V_t^{1/2} + \delta I)^{-1} m_t$}
					\EndFor
			\end{algorithmic}}
		\end{algorithm}
	\end{minipage}
\end{figure}
%
%\begin{figure}[t]
%	\centering
%	\subfigure[From the same layer]{\includegraphics[width=0.27\linewidth]{figures/input_neuron_loss_surface.pdf}}
%	\subfigure[From different layers]{\includegraphics[width=0.27\linewidth]{figures/different_neuron_loss_surface.pdf}}
%	\subfigure[\textsc{RMSprop}-Diag]{\includegraphics[width=0.22\linewidth]{figures/rmsprop_diag_trajectories.pdf}}
%	\subfigure[\textsc{RMSprop}-Group]{\includegraphics[width=0.22\linewidth]{figures/rmsprop_full_trajectories.pdf}}
%	\caption{Loss landscape and optimization trajectories on the loss landscape (a).}
%	\label{fig:landscape_and_trajectories}
%	\end{minipage}
%\end{figure}
%
We address the above question and provide a family of adaptive \textsc{Sgd} bridging exact GOP adaptation and its diagonal approximation, via coordinate partitioning.  
%\textcolor{red}{In the original adagrad paper, the authors proposed the same formulation. See Section 7 or their paper ``To conclude we would like to underscore a possible elegant %generalization that interpolates between full-matrix proximal '' and derived regret bounds for convex optimization. Am I right? If so, I will acknowledge this fact.} \textcolor{blue}{Jihun: %Yes. This is also right. In my understanding, they only give ``simple notes'' on the possibility of block diagonal matrix adaptations while we give in-depth study.}
Given a coordinate partition, we simply ignore the interactions of coordinates between different groups. For instance, given a gradient $g$ in a 6-dimensional space, %
%The followings are possible examples of constructing block diagonal matrices via coordinate partitioning,
one example of constructing block diagonal matrices via coordinate partitioning is $
g = (\underbrace{g_1, g_2}_{\mathcal{G}_1}, \underbrace{g_3, g_4, g_5}_{\mathcal{G}_2}, \underbrace{g_6}_{\mathcal{G}_3}) \rightarrow 
[g_{_{\mathcal{G}_1}} g_{_{\mathcal{G}_1}}^T \mid 0 \mid 0~\bm{;}~ 0 \mid g_{_{\mathcal{G}_2}} g_{_{\mathcal{G}_2}}^T \mid 0 ~\bm{;}~ 0 \mid 0 \mid g_{_{\mathcal{G}_3}} g_{_{\mathcal{G}_3}}^T] $
where $\mathcal{G}_i$ represents each group and $g_{_{\mathcal{G}_i}}$ denotes the collection of entries corresponding to group $\mathcal{G}_i$. %Note that the groups need not be formed by consecutive entries, nor do they all need to have the same size. %For instance, we can partition the coordinates of gradient $g = (g_1, \ldots, g_6)$ into groups such as $\big\{\{g_1, g_3, g_4\}, \{g_2, g_5\}, \{g_6\}\big\}$. 
Both exact GOP and diagonal approximation are special cases of our family. Exploring the use of block-diagonal matrices was suggested as future work in~\cite{duchi2011}, and our work therefore provides an in-depth study of this proposal. 
Our main algorithm, Algorithm \ref{alg:block_diagonal_update}, formalizes our approach for a total $r$ groups where each group $\mathcal{G}_i$ has a size of $n_i$ for $i \in [r]$. The Algorithm \ref{alg:block_diagonal_update} can handle arbitrary coordinate grouping with appropriate reordering of entries, and groups of unequal sizes. %The algorithm details are provided in the appendix due to space constraint.

\textbf{Effect of grouping on optimization.} Figure \ref{fig:grouping} shows some grouping examples in the context of deep learning models: grouping the weights with the same color in a neural network can approximate the exact GOP matrix with a block diagonal matrix of several small full matrices. 
%Since there are several ways of grouping parameters, we would like to get insight on developing meaningful grouping strategies. Figure \ref{fig:grouping} shows some grouping examples. 
To see which grouping could be more effective in terms of optimization, we revisit our MLP toy example. Figure \ref{fig:landscape_and_trajectories}-(a,b) show the loss landscape for different grouping strategies (weights other than shown are fixed as true model values). It can be seen that the loss landscape when grouping weights in the same layer has a much more dynamic curvature than when grouping weights in different layers. In this context, we expect that a preconditioner based on block-diagonal matrices is effective in terms of optimization and illustrate this empirically by comparing the grouping version for the loss landscape with dynamic curvature (Figure \ref{fig:landscape_and_trajectories}-(a)), and its diagonal counterpart. To figure out the effect of the block-diagonal based matrix preconditioner only, we compare both approaches using \textsc{RMSprop} which does not consider the first-order momentum. Figure \ref{fig:landscape_and_trajectories}-(c,d) illustrate the optimization trajectories. The block diagonal version of \textsc{RMSprop} converges to a stationary point in fewer steps than the diagonal approximation and shows a more stable trajectory. 

Computations and memory considerations compared to full matrix adaptation as well as its modified version GGT are discussed in the appendix.
\section{Convergence Analysis}

In this section, we provide a theoretical analysis of the convergence of Algorithm \ref{alg:block_diagonal_update}. We consider the following non-convex optimization problem, 
$\min f(x) \coloneqq \mathbb{E}_{\xi \sim \mathbb{P}} \big[f(x; \xi)\big]$
where $x$ is an optimization variable and $\xi$ is a random variable representing randomly selected data sample from $\mathcal{D}$. While $f$ is assumed to be continuously differentiable with Lipschitz continuous gradient, it can be non-convex. In non-convex optimization, we study convergence to ``stationarity'' and hence derive upper bounds for $\|\nabla f(x)\|^2$ as in \citep{ghadimi2013, ghadimi2016}. We assume $\widehat{V}_t$ in Algorithm \ref{alg:block_diagonal_update} has $r$ blocks, $\{\widehat{B}_{t, [j]}\}_{j=1}^{r}$. Our analysis covers two settings: when the minibatch size $M$ is fixed and when $M$ is increasing during training.

\paragraph{Convergence for Fixed Minibatch size.}

First, we provide sufficient conditions for our algorithms to converge and difference with diagonal counterpart, for fixed minibatch size. We make the following assumptions.

\begin{assumption}{\label{assumption1}}
	\textbf{(a)} $f$ is differentiable and has $L$-Lipschitz gradients. $f$ is also lower bounded. 
	\textbf{(b)} At time $t$, the algorithm can access a bounded noisy gradient. We assume the true gradient and noisy gradient are both bounded, i.e. $\|\nabla f(x_t)\|_\infty, \|g_t\|_\infty \leq G_\infty$ for all $t$. 
	\textbf{(c)} The noisy gradient $g_t$ is unbiased and the noise is independent, i.e. $g_t = \nabla f(x_t) + \zeta_t$ where $\mathbb{E}[\zeta_t] = 0$ and $\zeta_i$ is independent of $\zeta_j$ for $i \neq j$. 
	\textbf{(d)} $\beta_1 \geq \beta_{1,t}, \beta_{1,t} \in [0,1)$ is non-increasing.
	\textbf{(e)} For some constant $D_\infty > 0$, $\|\alpha_t \V_t^{-1/2} m_t \| \leq D_\infty$.
\end{assumption}
Here, we assume $\delta I$ is absorbed in $\widehat{V}_t$. Assumption \ref{assumption1} are also needed for the diagonal case~\cite{chen2019}. Condition (a) is a key assumption in general non-convex optimization analysis, and (b)-(d) are standard ones in this line of work. The last condition (e) states that the final step vector $\alpha_t \widehat{V}_t^{-1/2} m_t$ should be finite, which is a mild condition. We are now ready to state our first theorem.

\begin{theorem}{\label{thm1}}
	For the Algorithm \ref{alg:block_diagonal_update}, define the quantity $Q_t \coloneqq \matnorm{\alpha_t \widehat{V}_t^{-1/2} - \alpha_{t-1} \widehat{V}_{t-1}^{-1/2}}{2} = \max_{j \in [r]} \{\matnorm{\alpha_t \widehat{B}_{t,[j]}^{-1/2} - \alpha_{t-1} \widehat{B}_{t-1, [j]}^{-1/2}}{2}\}$
	which measures the maximum difference in effective spectrums over all blocks $\widehat{B}_{t,[j]}$.
	Under the Assumption \ref{assumption1}, the Algorithm \ref{alg:block_diagonal_update} yields 
	\begin{align}\label{suff_uppbound}
	\mathbb{E}\Bigg[\sum\limits_{t=1}^{T} \alpha_t \Big\langle \nabla f(x_t), \V_t^{-\frac{1}{2}} \nabla f(x_t) \Big\rangle\Bigg] \leq \mathbb{E}\Bigg[C_1 \sum\limits_{t=1}^{T} \underbrace{\|\alpha_t \widehat V_t^{-\frac{1}{2}} g_t\|^2}_{\text{Term A}} +~ C_2 \sum\limits_{t=2}^{T} \underbrace{Q_t}_{\text{Term B}} +~ C_3 \sum\limits_{t=2}^{T-1} Q_t^2\Bigg] + C_4
	\end{align}
	where $C_1, C_2,$ and $C_3$ are constants independent of $d$ and $T$, $C_4$ is a constant independent of $T$. The expectation is taken with respect to all the randomness corresponding to $\{g_t\}_{t=1}^{T}$. ~\\
	Further, we let $\gamma_t \coloneqq \min_{\{g_t\}_{t=1}^{T}} \lambda_{min}(\alpha_t \widehat V_t^{-1/2})$ denote the possible minimum effective spectrum over all past gradients. Then, we have $\min_{t \in [T]}\mathbb{E}\big[\|\nabla f(x_t)\|^2\big] = O\big(\frac{s_1(T)}{s_2(T)}\big)$
	where $s_1(T)$ is defined as the upper bound in (\ref{suff_uppbound}), namely, $\mathcal{O}\big(s_1(T)\big)$, and $\sum_{t=1}^{T} \gamma_t = \Omega\big(s_2(T)\big)$.
\end{theorem}

\textbf{Remarks.} Our first theorem provides sufficient conditions, $s_1(T) = o\big(s_2(T)\big)$, for convergence as in the diagonal case \cite{chen2019}. The convergence of block-diagonal and diagonal versions depend on the dynamics of Term A and Term B as noted in \cite{chen2019} and in our theorem. %In Table \ref{tab:convergence_dynamics}, 
The Term A for block diagonal version is $\|\alpha_t \widehat{V}_t^{-1/2} g_t\|^2$ and for diagonal version is $\|\alpha_t g_t / \sqrt{\widehat{v}_t}\|^2$. The Term B for block diagonal version is $\matnorm{\alpha_t \widehat{V}_t^{-1/2} - \alpha_{t-1} \widehat{V}_{t-1}^{-1/2}}{2}$ and for diagonal version is $\|\alpha_t / \sqrt{\widehat{v}_t} - \alpha_{t-1} / \sqrt{\widehat{v}_{t-1}}\|_1$. We can see that the main difference is in Term B. 

If we assume grouping size of 1 in our Algorithm \ref{alg:block_diagonal_update}, i.e. $r = d$, $\widehat{V}_t$ becomes a diagonal matrix. In this case, Term B for the block diagonal version %, $\matnorm{\alpha_t \widehat{V}_t^{-1/2} - \alpha_{t-1} \widehat{V}_{t-1}^{-1/2}}{2}$, 
represents the maximum difference of effective learning rate while Term B %$\|\alpha_t / \sqrt{\widehat{v}_t} - \alpha_{t-1} / \sqrt{\widehat{v}_{t-1}}\|_1$
for the diagonal version means the sum of differences of effective learning rate over all coordinates. Therefore, our bound improves upon prior results even for the diagonal case. The difference comes from the proofs. The proofs of prior analysis for diagonal case depends on the coordinate-wise effective stepsize, but this term is absent in our case. Our proofs rely on the matrix norm which allows smaller bound.
%\textcolor{red}{Block diagonal case with group size 1 boils down to diagonal case, right?  So do you mean that our results when specialized for group size 1 (which boils down to diagonal case) are superior than \emph{prior} results on the diagonal case? Namely our analysis improve upon prior analysis for diagonal case? Or do you mean that our results for block diagonal case in general, are superior to diagonal case? } 
%\textcolor{blue}{Jihun: Actually, I mean our analysis improve prior diagonal case. Based on this observation, I mean we can expect improving bound even for general block diagonal matrix case via empirical study since the Term B for prior diagonal case depends on the ``sum''. What about being stronger for improvement of prior diagonal case?} 
%\textcolor{red}{[TODO: All] Clarify that the remark based on group size 1 (which is equivalent to diagonal case) shows that our results for the diagonal case are superior to prior results for the diagonal case.  ALL: TRY TO BE STRONGER AND DISCUSS HOW OUR RESULTS ARE BETTER EVEN FOR GENERAL BLOCK DIAGONAL CASE}
To investigate the difference for general group size, we perform an empirical evaluation through MNIST classification task using 784-100-10 MLP with fixed minibatch size 128, and group size of 10. Figure \ref{fig:dynamics_of_terms} illustrates the dynamics of Term A and Term B, supporting our observations by showing a big difference for Term B. As we observe significant improvement in Term B, we expect that block diagonal matrix approximation 
can alleviate the oscillatory behavior of diagonal version, and we corroborate this by empirical studies on more architectures in Section~\ref{sec:exp}.

\begin{figure}[t]
	\centering
	\subfigure[Term A]{\includegraphics[width=0.49\linewidth]{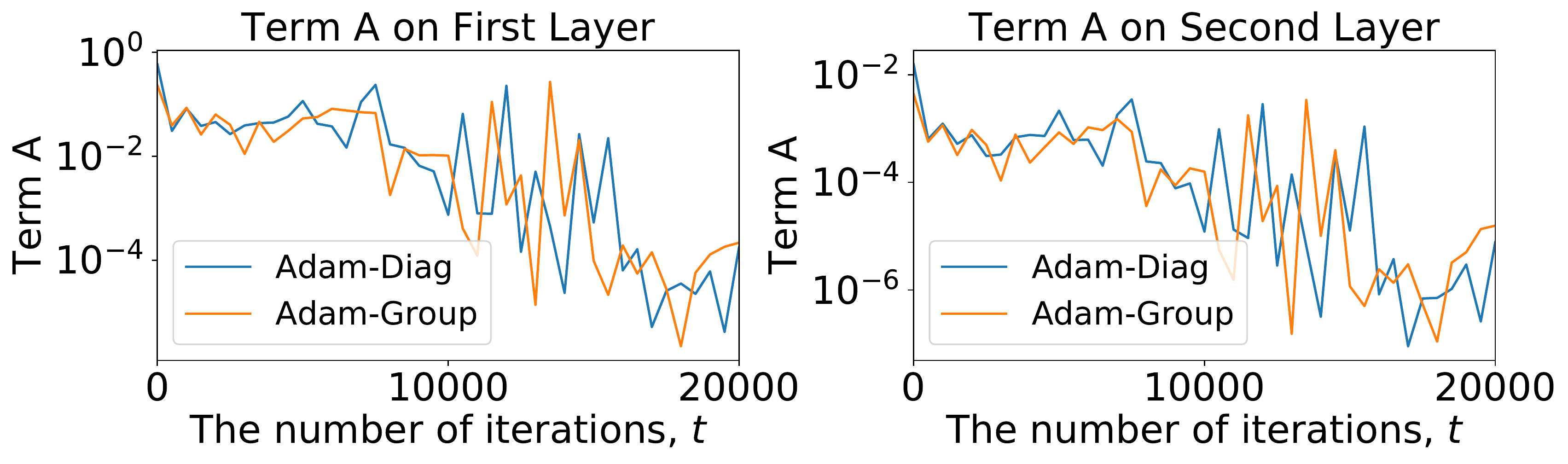}}
	\subfigure[Term B]{\includegraphics[width=0.49\linewidth]{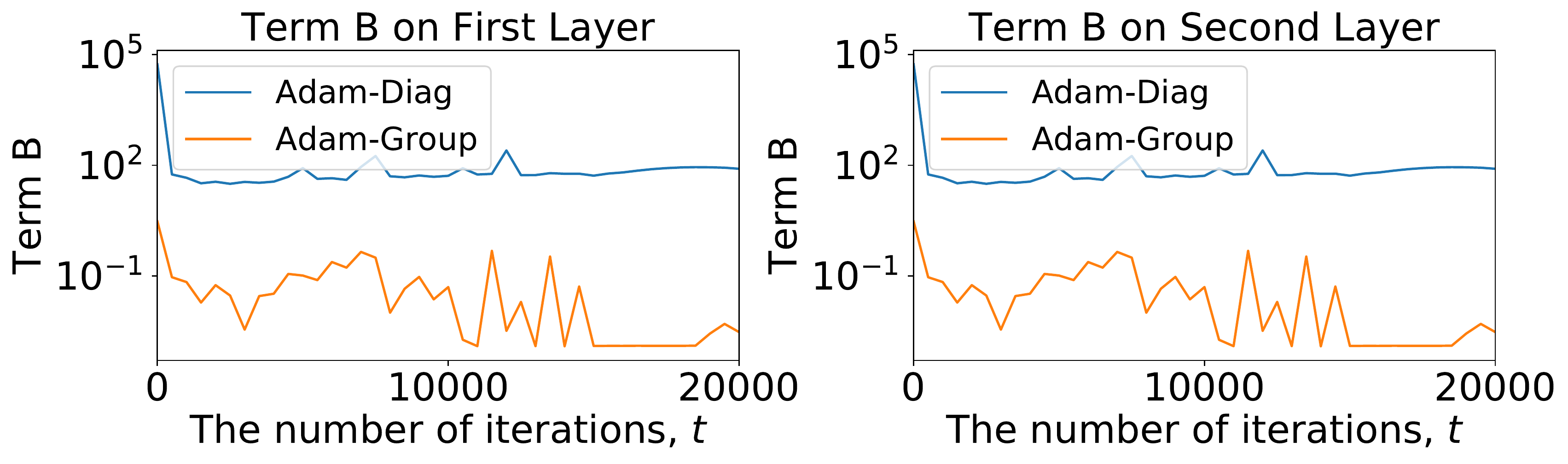}}
	\caption{Dynamics of Term A and Term B for block diagonal version with group size of $10$ and diagonal version. Both approaches have similar scale of Term A, but there are big differences in Term B for both layers. For fair comparisons, we set $\delta = \epsilon = 10^{-4}$.}
	\label{fig:dynamics_of_terms}
\end{figure}

We instantiate our theorem to \textsc{AdaGrad/AdaFom}\citep{chen2018, zou2018} satisfying the sufficient conditions.% \textsc{AdaFom}~\citep{chen2018convergence, zou2018convergence} is recently proposed, which combines \textsc{AdaGrad} and first-order momentum.

\begin{corollary}{\label{cor1}} \textsc{(AdaGrad/AdaFom)}
	If we set $\alpha_t = 1/\sqrt{t}$ and $\beta_{1,t} \leq \beta_1 \in [0,1)$ and $\beta_{1,t}$ is non-increasing ($\beta_1 = 0$ for \textsc{AdaGrad}), \textsc{AdaGrad/AdaFom} with block diagonal matrix adaptations achieve $\min_{t \in [T]} \mathbb{E}\big[\|\nabla f(x_t)\|^2\big] = \mathcal{O}(\log T/\sqrt{T})$.
\end{corollary}

This is the same convergence rate as diagonal versions~\citep{chen2018, zou2018}. As discussed in remarks, we expect that it is possible to remove the $\log T$ term since block diagonal matrix adaptation can have smaller bound, but this is an open question.

\paragraph{Convergence of Block-diagonal \textsc{RMSprop} and \textsc{Adam} with Increasing Minibatch Size.}

Moving on to the case of increasing minibatch size, we show that algorithms based on EMA such as \textsc{RMSprop/Adam} with block diagonal matrix adaptation converge to a stationary point. This is the case of our Algorithm \ref{alg:block_diagonal_update} with $H_t = \beta_2 \widehat{V}_{t-1} + (1 - \beta_2) g_t g_t^T$. We assume bounded variance of stochastic gradients, $\mathbb{E}\big[\|\nabla f(x; \xi) - \nabla f(x)\|^2\big] \leq \sigma^2$, a standard assumption in analyses of stochastic gradient methods \citep{ghadimi2013, ghadimi2016}. We are ready to state our second theorem.
\begin{theorem}{\label{thm2}}
	\textsc{(RMSprop/Adam)} For the Algorithm \ref{alg:block_diagonal_update}, define the quantity $\kappa_t \coloneqq \kappa(\beta_2^{1/2} \widehat{V}_t^{1/2} + \delta I)$.
	Under the Assumption \ref{assumption1} (without (e)) and bounded variance of gradients, we suppose that $\kappa_t$ is bounded above by $\kappa_{\text{max}}$ and $\alpha_t = \alpha$ ~for all $t$. Furthermore, we assume that all blocks $\{\widehat{B}_{t, [j]}\}_{j=1}^{r}$ are full-rank when $t \geq t_0$ for some $t_0 \in \mathbb{N}$ and $\beta_{1,t} = \beta_1 \lambda^{t-1}$ for some $\lambda \in (0,1)$. If the parameters $\alpha$, $\beta_2$, and $\delta$ are chosen such that $1 - \beta_2 \leq \frac{\delta^2}{9G_\infty^2 \kappa_{\text{max}}^2}, $ and  $\alpha \leq \frac{2\delta}{3L \kappa_{\text{max}}}$.
	Then, the iterates $x_t$ generated by the Algorithm \ref{alg:block_diagonal_update} satisfies 
	\begin{align*}
	\min\limits_{t \in [T]} \mathbb{E}\big[\|\nabla f(x_t)\|^2\big] & \leq \frac{3(\sqrt{\beta_2} G_\infty + \delta)}{1 - \beta_1} \Bigg[ \frac{f(x_{t_0}) - f(x^{*})}{\alpha T} + \frac{\sigma^2}{M}\Big(\frac{G_\infty \sqrt{1 - \beta_2}}{\delta^2} + \frac{\alpha L}{2 \delta^2}\Big) + \frac{C}{T}\Bigg]
	\end{align*}
	which is $\mathcal{O}(1/T + \sigma^2/M)$, where $C$ is a constant independent of $T$ and $x^{*}$ is an optimal solution. Additionally, we can obtain the bound for \textsc{RMSprop} if we set $\beta_1 = 0$.
\end{theorem}

\textbf{Remarks.}  If we consider exact GOP, i.e. $r = 1$, $\widehat{V}_t$ is an exact full GOP and the condition that $\widehat{V}_t$ becomes full-rank within finite time $t_0$ may not be satisfied. For instance, consider the least-square problem with $y = X\beta + \epsilon$ where $(X,y)$ is a given data, $\theta$ is a parameter we should optimize, and $\epsilon$ is a noise. For this case, the GOP matrix contains $X^TXX^TX$, so it is always rank-deficient in the high-dimensional setting. In contrast, we emphasize that Theorem \ref{thm2} can be applied to block diagonal matrix adaptations, because we only require the full-rankness of each small sub-matrix $\widehat{B}_{t, [j]}$. 

The condition on $\kappa_{\text{max}}$ states that $\kappa_{\text{max}}$ should be bounded above by some constant. Therefore, the condition number of $(\beta_2^{1/2}\widehat{V}_{t}^{1/2} + \delta I)$ at time $t$ should not be ``too'' large (i.e. not too ill-conditioned), so we choose $\delta$ not too small. On the other hand, we cannot unconditionally increase $\delta$ since the first term $\frac{3(\sqrt{\beta_2} G_\infty + \delta)}{(1 - \beta_1)} \times \frac{f(x_{t_0}) - f(x_T)}{\alpha T}$ tends to diverge as $\delta$ increases. Balancing these two, we use $\delta = 10^{-4}$ for our experiments, instead of $\delta = 10^{-8}$ which is recommended for diagonal case. 

Lastly, we need $M = \mathcal{O}(T)$ to guarantee convergence. However, this condition is not stringent. As a concrete example, consider a problem with sample size $N$ and minibatch size $M$ with maximum 200 epochs. Since the minibatch size is $M$, $T$ should be $\mathcal{O}(\frac{200N}{M})$ resulting in $M = \mathcal{O}(10\sqrt{2N}) = \mathcal{O}(\sqrt{N})$, which is practical in real cases.

%%%%% extension %%%%%
\section{Interpolation with SGD via Spectrum-Clipping} \label{sec:clip}%\textcolor{blue}{TODO: shorten this section.}

%We now provide a variant of our framework that can smoothly transform into vanilla \textsc{Sgd} for even better generalization. This is motived by the following line of recent work. 
It has been shown in \cite{wilson2017} that adaptive methods are better than vanilla \textsc{Sgd} in the early stage but get worse as the learning process matures. To address this, \cite{keskar2017} suggests training networks with \textsc{Adam} at the beginning and switching to \textsc{Sgd} later. \cite{luo2019} proposes methods \textsc{AdaBound/AMSBound} which clip the effective learning rate $\alpha_t / (\sqrt{\widehat{v}_t} + \epsilon)$ of \textsc{Adam} by decreasing sequence of intervals $I_t = [\eta_l(t), \eta_u(t)]$ every iteration which converges to some point, thereby resembling \textsc{Sgd} in the end. However, this type of extension is not obvious in our framework due to the absence of \emph{effective} learning rate in our case. %
%
%
%which can exploit both reducing computational cost and benefitting from higher generalization of SGD. 
%
%
%Recent study \cite{wilson17} show that adaptive methods have poor generalization than vanilla SGD, theoretically for small problems and empirically with deep learning tasks. Adaptive methods are faster than vanilla SGD in the early stage of training, but they become worse than vanilla SGD in the end in terms of generalization. To alleviate this problem, 
%
%
%
Instead, we observe that the \emph{spectral property} is important in our convergence analysis (In Theorem \ref{thm1}: convergence heavily depends on Term B, maximum changes in \emph{effective spectrum}, $\matnorm{\alpha_t \widehat{V}_t^{-1/2} - \alpha_{t-1}\widehat{V}_{t-1}^{-1/2}}{2}$. In Theorem \ref{thm2}, we need conditions on $\kappa(\beta_2^{1/2}\widehat{V}_t^{1/2} + \delta I)$). Motivated on them, we propose a \emph{spectrum-clipping} scheme which clips the spectrum of $\alpha_t (\widehat{V}_t^{1/2} + \delta I)^{-1}$ by decreasing sequence of intervals. For spectrum-clipping, we use the following modified update rule in Algorithm \ref{alg:block_diagonal_update} after constructing $\widehat{V}_t$: \textbf{(i)} $\widehat{U}_t, \widehat{\Sigma}_t^{1/2}, \_ \leftarrow \mathrm{SVD}(\widehat{V}_t^{1/2})$, \textbf{(ii)} $\widetilde{\Sigma}_t^{-1/2} \leftarrow \mathrm{Clip}(\lambda(\alpha_t (\widehat{\Sigma}_t^{1/2} + \delta I)^{-1}), \lambda_l(t), \lambda_u(t)\big)$, and \textbf{(iii)} $x_{t+1} \leftarrow x_t - \widehat{U}_t^T \widetilde{\Sigma}_t^{-1/2} \widehat{U}_t m_t$. We schedule the sizes of clipping intervals converging to a single point uniformly over all coordinates  so that $\alpha_t (\widehat{V}_t^{1/2} + \delta I)^{-1}$ can be easily computed in the form of constant times identity matrix and effectively behaves like vanilla \textsc{Sgd}. In all our experiments, 
%we tune hyper-parameter $I_t$. %As \textsc{AdaBound/AMSBound}~\citep{luo2019adaptive} employs $\eta_l(t) = \big(1 - \frac{1}{\gamma t + 1}\big) \alpha^{*}$ and $\eta_u(t) = \big(1 + \frac{1}{\gamma t}\big) \alpha^{*}$, 
we use  $\lambda_l(t) = (1 - \frac{1}{\gamma t + 1}) \alpha^{*}$ and $\lambda_u(t) = (1 + \frac{1}{\gamma t}) \alpha^{*}$ where $\gamma$ reflects the clipping speed and $\alpha^{*}$ represents the final learning rate of vanilla \textsc{Sgd}, as in~\citep{luo2019}. We specify how to choose $\gamma$ and $\alpha^{*}$ in Section~\ref{sec:exp}.

%induce $I_t$ to be a point $\lambda^{*}$ by imposing the relation $\lambda_l(t) < \lambda_l(t+1) < \cdots < \lambda^{*} < \cdots < \lambda_u(t+1) < \lambda_u(t)$, thereby making $\alpha_t (\widehat{V}_t^{1/2} + \delta I)^{-1}$ to be a constant times identity matrix.

%%%%% experiments %%%%%
\section{Experiments}\label{sec:exp}

\begin{figure}[t]
	\centering
	\subfigure[MLP for \textsc{Adam}]{\includegraphics[width=0.49\linewidth]{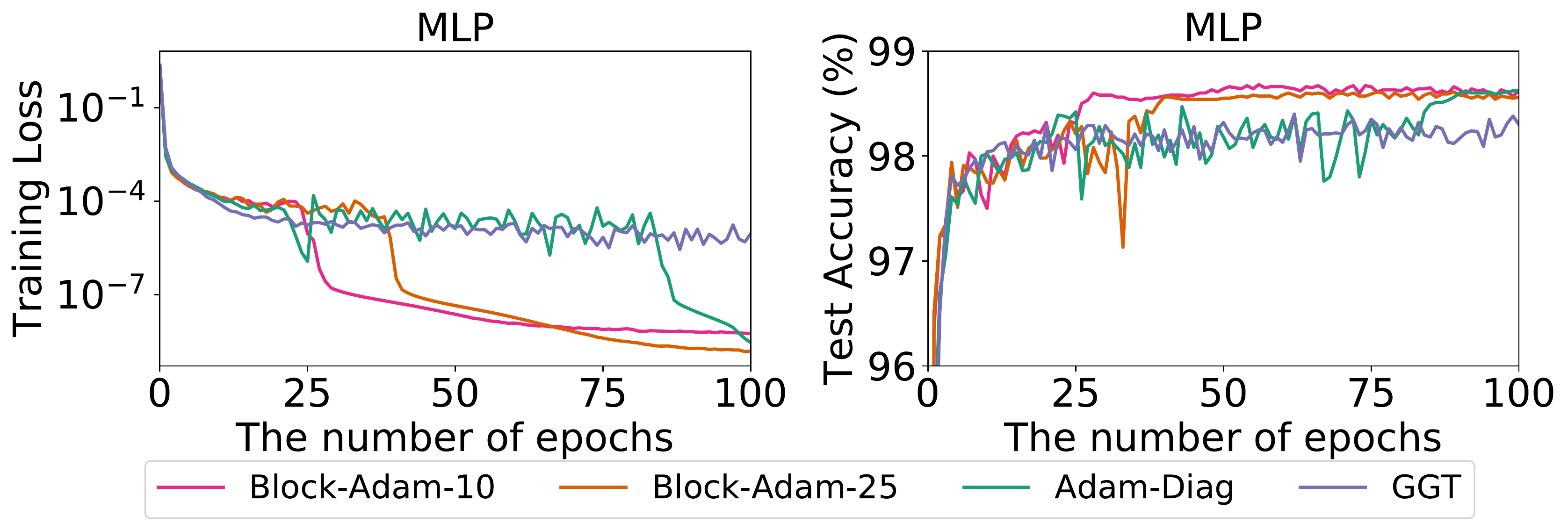}}
	\subfigure[LeNet-5-Caffe for \textsc{Adam}]{\includegraphics[width=0.49\linewidth]{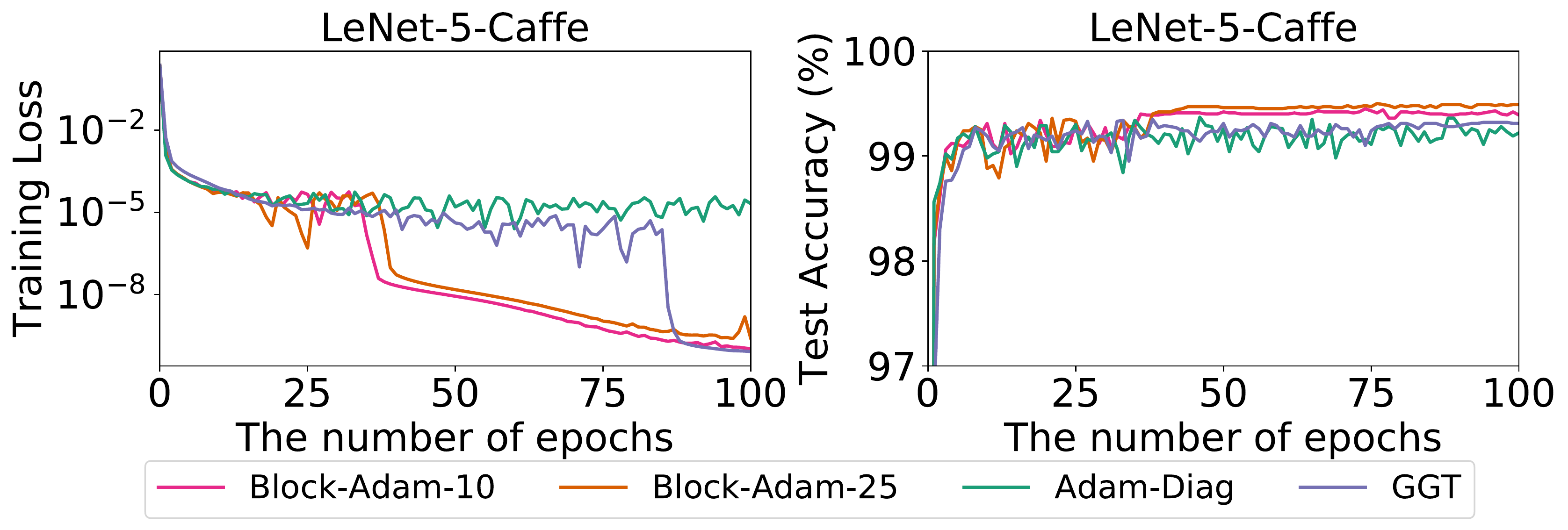}}
	\caption{Training loss and test accuracy for MLP/LeNet-5-Caffe on MNIST dataset.}
	\label{fig:shallow_models}
\end{figure}

\begin{figure}[t]
	\centering
	\subfigure[Training loss and test accuracy for DenseNet-BC-100-12 on CIFAR-100.]{\includegraphics[width=\linewidth]{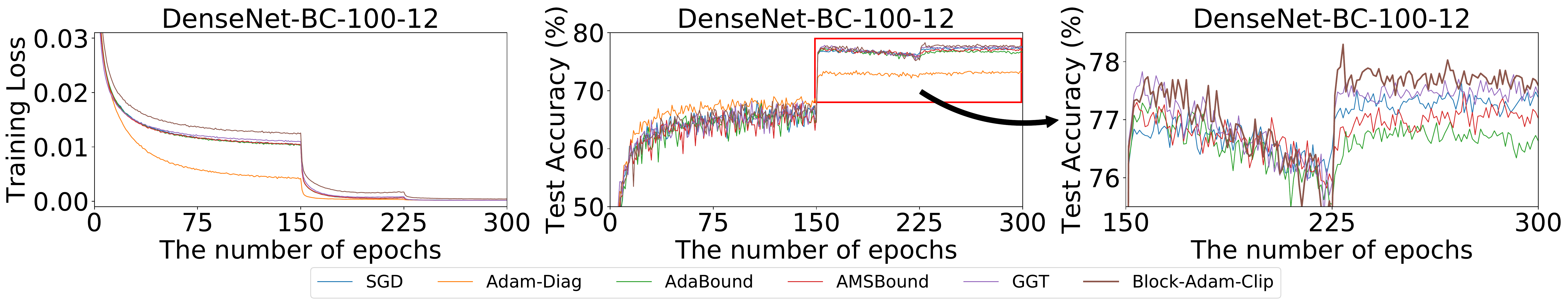}}
	\subfigure[Training loss and test perplexity for 3-layer deep LSTM on PTB]{\includegraphics[width=0.645\linewidth]{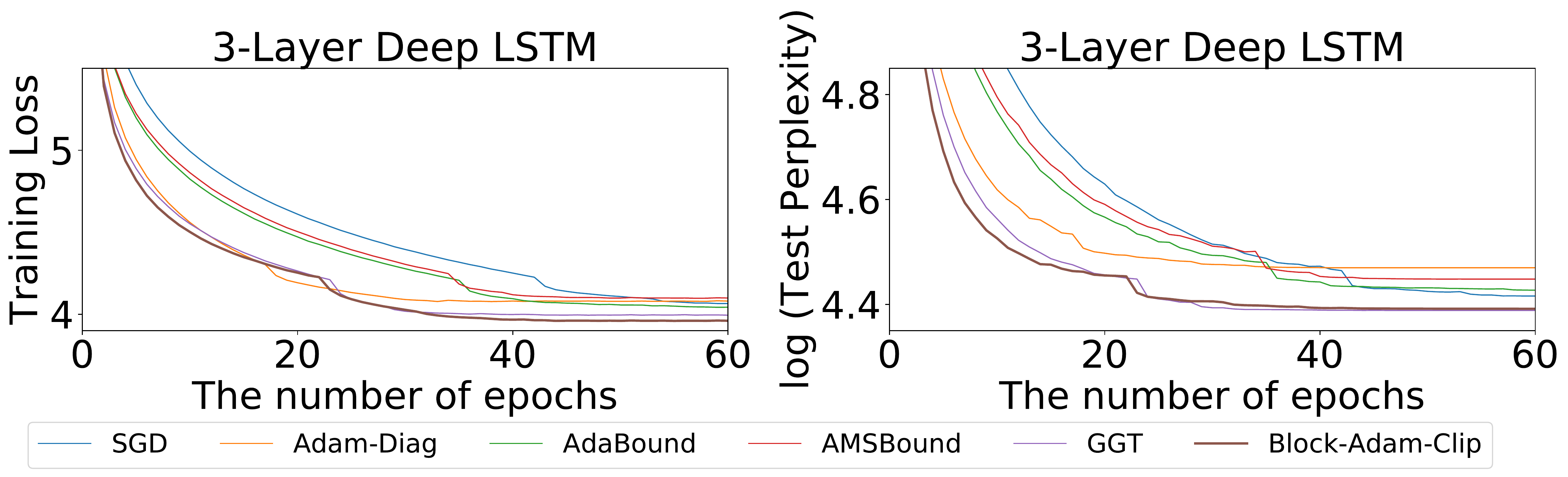}}
	\subfigure[MIG Score on $\beta$-TCVAE]{\includegraphics[width=0.348\linewidth]{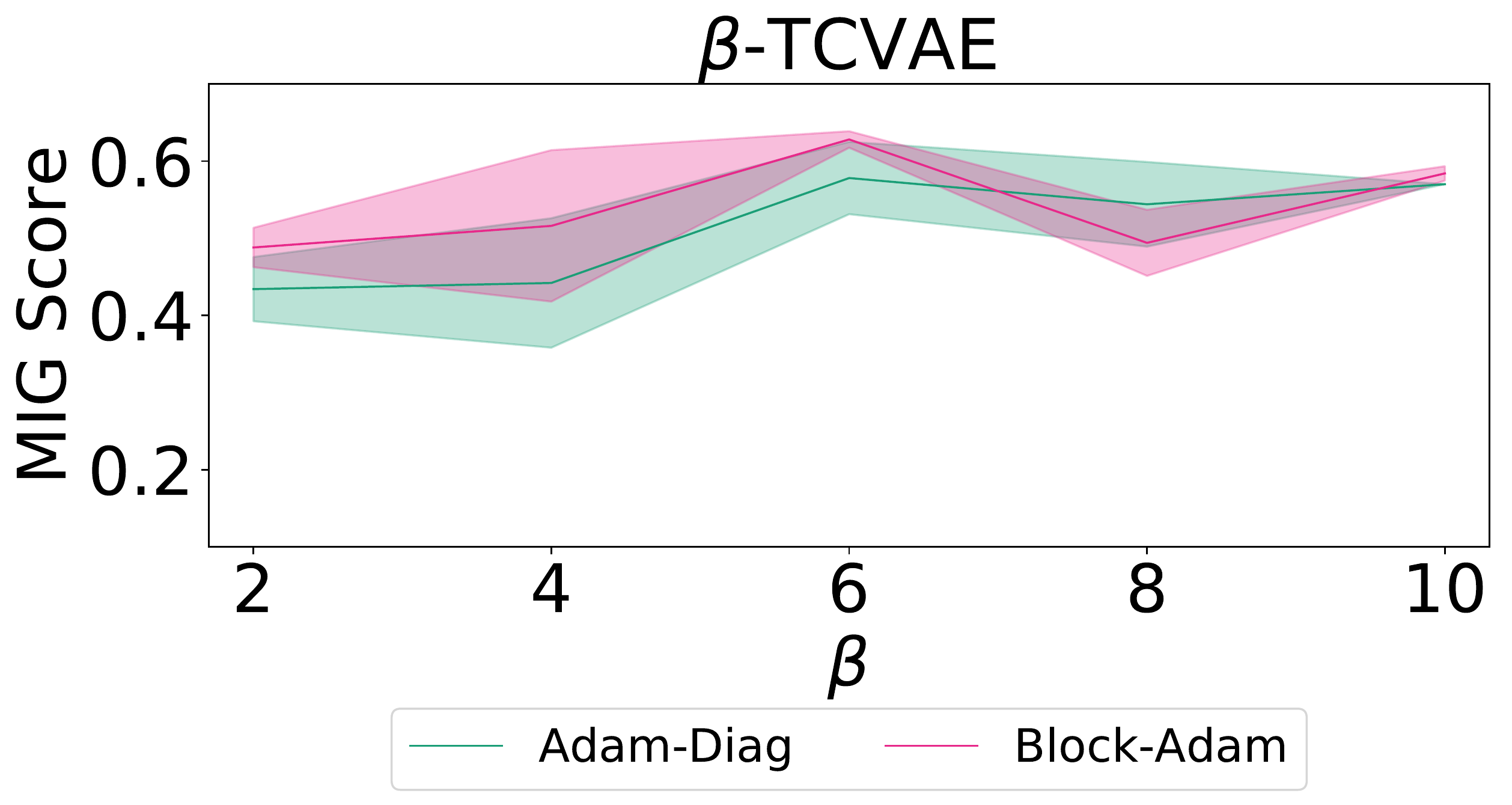}}
	\caption{Results on (a) CIFAR classification, (b) language models, (c) generative models.}
	\label{fig:sota}
\end{figure}

We consider two sets of experiments. The first shows the differences between block-diagonal and diagonal versions. The second investigates whether block diagonal matrix adaptation can achieve state-of-the-art performance on benchmark architecture/dataset for various important deep learning problems. For the first set, we do not consider the spectrum-clipping of Section~\ref{sec:clip}, to clearly assess the effect of coordinate partitioning. In our Algorithm \ref{alg:block_diagonal_update}, coordinate grouping can be done in a number of ways. Given our insight that grouping weights in the same layer could be more effective, we consider Figure \ref{fig:grouping}-(c) with grouping 10 or 25 weight parameters connected to input-neuron for fully-connected layer, and we consider filter-wise grouping for convolutional layers as in Figure \ref{fig:grouping}-(d). We add a suffix \textsc{Block} for our optimizer such as \textsc{Block-Adam}, representing the \textsc{Adam} with block diagonal matrix adaptations. Details on hyperparameter choices for each experiment are provided in the appendix.
%\textcolor{green}{We use the recommended step size or tune it in the range $[10^{-4}, 10^{2}]$ for all comparison algorithms. For $\textsc{Adam}$ based algorithms, we use default decay parameters $(\beta_1, \beta_2) = (0.9, 0.999)$. For a diagonal version of \textsc{Adam} variant algorithm, we choose numerical stability parameter $\epsilon = 10^{-4}$ (to match our choice of $\delta = 10^{-4}$) for fair comparison since the larger value of $\epsilon$ can improve the generalization performance as discussed in~\citep{zaheer18}. For $\gamma$ and $\alpha^{*}$ in clipping bound functions, we consider $\gamma \in \{10^{-3}, 10^{-4}\}$ and choose $\alpha^{*}$ as the best-performing initial learning rate for vanilla \textsc{Sgd}. As in \cite{luo2019adaptive}, our results are also not sensitive to choice of $\gamma$ and $\alpha^{*}$. With these hyperparameters, we consider maximum 300 epochs training time, and mini-batch size or learning rate scheduling are introduced in each experiment description. Our Algorithm \ref{alg:block_diagonal} requires $\mathrm{SVD}$ procedures to compute the square root of a block diagonal matrix. We apply $\mathrm{SVD}$ efficiently to all small sub-matrices simultaneously through batch mode of $\mathrm{SVD}$.}

\textbf{Investigating Grouping Effect.}  We investigate the effect of coordinate partitioning on MNIST classification and generative models. 

\emph{MNIST Classification.} We consider a fully connected model, LeNet-300-100~\citep{lecun1998}, and a simple convolutional network, LeNet-5-Caffe\footnote{\url{https://github.com/BVLC/caffe/tree/master/examples/mnist}}. We use 128 mini-batch size and train networks with maximum 100 epochs. To see the effect of coordinate grouping, we compare \textsc{RMSProp/Adam} with block diagonal matrix version and diagonal counterpart. Figure \ref{fig:shallow_models} illustrates the results for \textsc{Adam}, and the results for \textsc{RMSprop} are in the appendix. The learning curve looks similar in the early stage of training, but our methods converge without oscillatory behavior in the latter part of training, which corroborates our observations on Theorem \ref{thm1}. The generalization of block-diagonal approaches also becomes more stable than diagonal variant and GGT, and overall superior across epochs.
%Though \textsc{GGT} converges to a lower training loss on LeNet-5-Caffe, our methods outperform \textsc{GGT} in terms of generalization as well as the diagonal version.

\emph{Generative Models.} %For generative models such as variational autoencoder (VAE) or generative adversarial networks (GAN), the \textsc{Adam} optimizer with tuning $(\beta_1, \beta_2)$ shows better performance than \textsc{Sgd}. Therefore, 
We conduct experiments on very recent variant of VAE called $\beta$-TCVAE~\citep{chen18}. The goal of this model is to make the encoder $q(z|x)$ give disentangled representation $z$ of input images $x$ by additionally forcing $q(z) = \int q(z|x)p(x) dx$ to be factorized, which can be achieved by giving heavier penalty on total correlation. We evaluate our optimizer with the Mutual Information Gap (MIG) score they proposed, to measure disentanglement of the latent code. Following implementation in~\citep{chen18}, we use convolutional encoder-decoder for $\beta$-TCVAE on 3D faces dataset~\citep{paysan2009}. Figure \ref{fig:sota}-(c) illustrates the results over 5 random simulations with $95\%$ confidence region. The block diagonal version outperforms diagonal version except at $\beta = 8$, and we can achieve the best performance at $\beta = 6$, which is a recommended value for $\beta$-TCVAE~\citep{chen18}.

\textbf{Improving Performance with Spectrum-Clipping.} We demonstrate the superiority of our algorithms using more complex benchmark architecture/dataset for two popular tasks in deep learning: image classification and language modeling. For both tasks, vanilla \textsc{Sgd} with proper learning rate scheduling has enjoyed state-of-the-art performance. Therefore, we compare algorithms using our spectrum-clipping methods that can exploit higher generalization ability of vanilla \textsc{Sgd}. 

\begin{wraptable}{r}{0.6\linewidth}
	\caption{Test error (\%) for CIFAR dataset.}
	\centering
	{\scriptsize
		\begin{tabular}{ccccccc}
			\toprule
			& \textsc{SGD} & \textsc{Adam} & \makecell{\textsc{Ada}- \\ \textsc{Bound}} & \makecell{\textsc{Ams}- \\ \textsc{Bound}} & \textsc{GGT} & \makecell{\textsc{Block}- \\ \textsc{Adam}} \\
			\midrule
			CIFAR-10 & 4.51 & 6.07 & 4.78 & 4.77 & 6.33 & \textbf{4.34} \\
			\midrule
			CIFAR-100 & 22.27 & 26.51 & 22.5 & 22.52 & 22.17 & \textbf{21.7} \\
			\bottomrule
	\end{tabular}}
	\label{tab:cifar_performance}
\end{wraptable}

\emph{CIFAR Classification.} We conduct experiments using DenseNet architecture~\citep{huang2017}. %According to experiment settings in~\citep{huang2017densely}, we use mini-batch size 64 and consider maximum 300 epochs. Also, we use a \emph{step-decay} learning rate scheduling in which the learning rate is divided by $10$ at $50\%$ and $75\%$ of the total number of training epochs. With this setting, vanilla \textsc{Sgd} with a momentum factor 0.9 performs best with initial learning rate $\alpha^{*} = 0.1$, so we use this value for our bound functions of spectrum-clipping, $\lambda_l(t)$ and $\lambda_u(t)$.
Figure \ref{fig:sota}-(a) illustrates our results on CIFAR-100 datasets, and the figure for CIFAR-10 is in appendix. In both cases, the training speed of our algorithm at the early stage is similar or slightly slower, but we can arrive at the \emph{state-of-the-art} generalization performance in the end among all comparison algorithms. Specifically, we can achieve great improvement in generalization about $0.5\%$ for CIFAR-100 dataset as in Table \ref{tab:cifar_performance}. Note that, our spectrum-clipping method consistently achieves higher performance.%, while the performance of GGT varies across problems. % as can be seen in Table \ref{tab:cifar_performance}.

\emph{Language Models.} We use recurrent networks~\citep{zaremba2014}, base architectures still frequently used today for language modeling. While \citep{zaremba2014} uses only two layers maximum, we add one more layer to consider more complex and deeper networks. To consider similar model capacity as~\citep{zaremba2014}, we use 500 hidden units on each layer. Based on this architecture, we build a word-level language model using 3-layer LSTM~\cite{HochSchm1997} on Penn TreeBank (PTB) dataset~\cite{marcus1994}. %In this experiment, we use a \emph{dev-decay} learning rate scheduling~\cite{wilson17} where we reduce learning rate by a constant factor if the model does not attain a new best validation performance at each epoch as in~\cite{zaremba2014recurrent}. Under this setting, vanilla \textsc{Sgd} performs best when the initial learning rate $\alpha^{*} = 5$.
Figure \ref{fig:sota}-(b) shows the experimental results: the optimizer with spectrum-clipping of \textsc{Adam} outperforms all the other algorithms w.r.t. learning curve. It achieves similar perplexity as GGT and outperforms the other methods.

%%%%% conclusion %%%%%
\section{Concluding Remarks}

We proposed a general adaptive gradient framework that approximates exact GOP with block diagonal matrices via coordinate grouping, and showed that it can be a practical and powerful solution that can effectively utilize structural characteristics of deep learning architectures. 
We analyzed %provided theoretical 
convergence %results
for our approach, showed that they can lead to a smaller upper bound than its popular diagonal counterpart, and confirmed our findings empirically. We also proposed a spectrum-clipping algorithm which achieved state-of-the-art generalization performance on popular deep learning tasks. %, with various mini-batch sizes and different learning rate scheduling. 
As future work, we plan to explore additional strategies for setting the clipping parameters in our approach to strike the best balance between training speed and generalization ability, and to develop novel computationally efficient methods that generalize well.

%\bibliographystyle{unsrt}
%\bibliography{fullopt,adaopt,neuralnet,noncvxopt}

\small
\newpage
\appendix

%%%%% supplementary materials %%%%%
\section*{Supplementary Materials}

\section{Toy MLP example: Full GOP adaptation vs. Diagonal approximation}

We consider a structured MLP (two nodes in two hidden layers followed by single output). For hidden units, we use ReLU activation \citep{nair2010} and the sigmoid unit for the binary output. We generate $n=10$ i.i.d. observations: $x_i \sim \mathcal{N}(0, I_2)$ and $y_i$ from this two layered MLP given $x_i$.  
The results of our toy experiment are depicted in Figure \ref{fig:alg_compare}.

\begin{figure}[!h]
	\centering
	\subfigure{\includegraphics[width=0.32\linewidth]{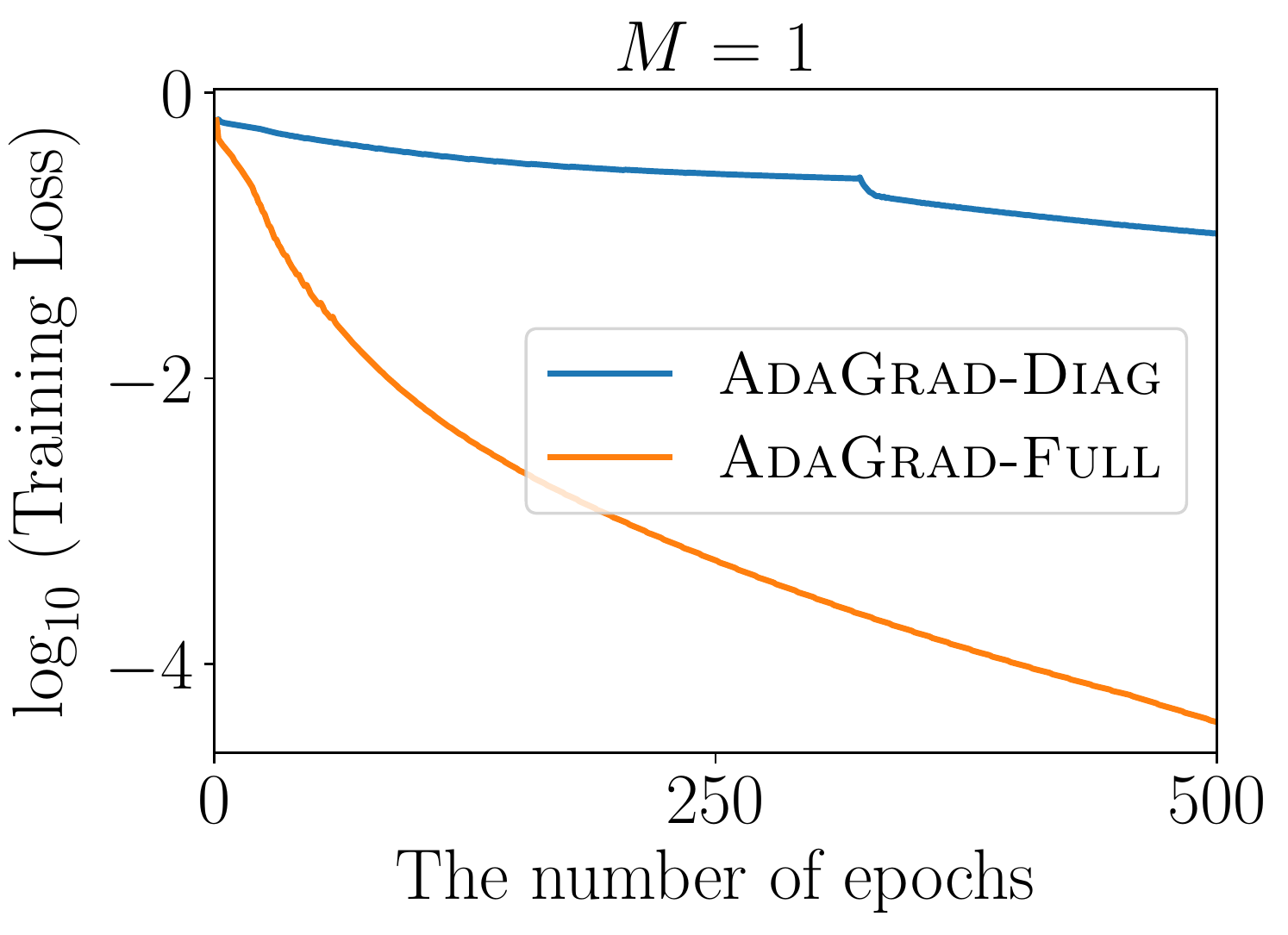}}
	\subfigure{\includegraphics[width=0.32\linewidth]{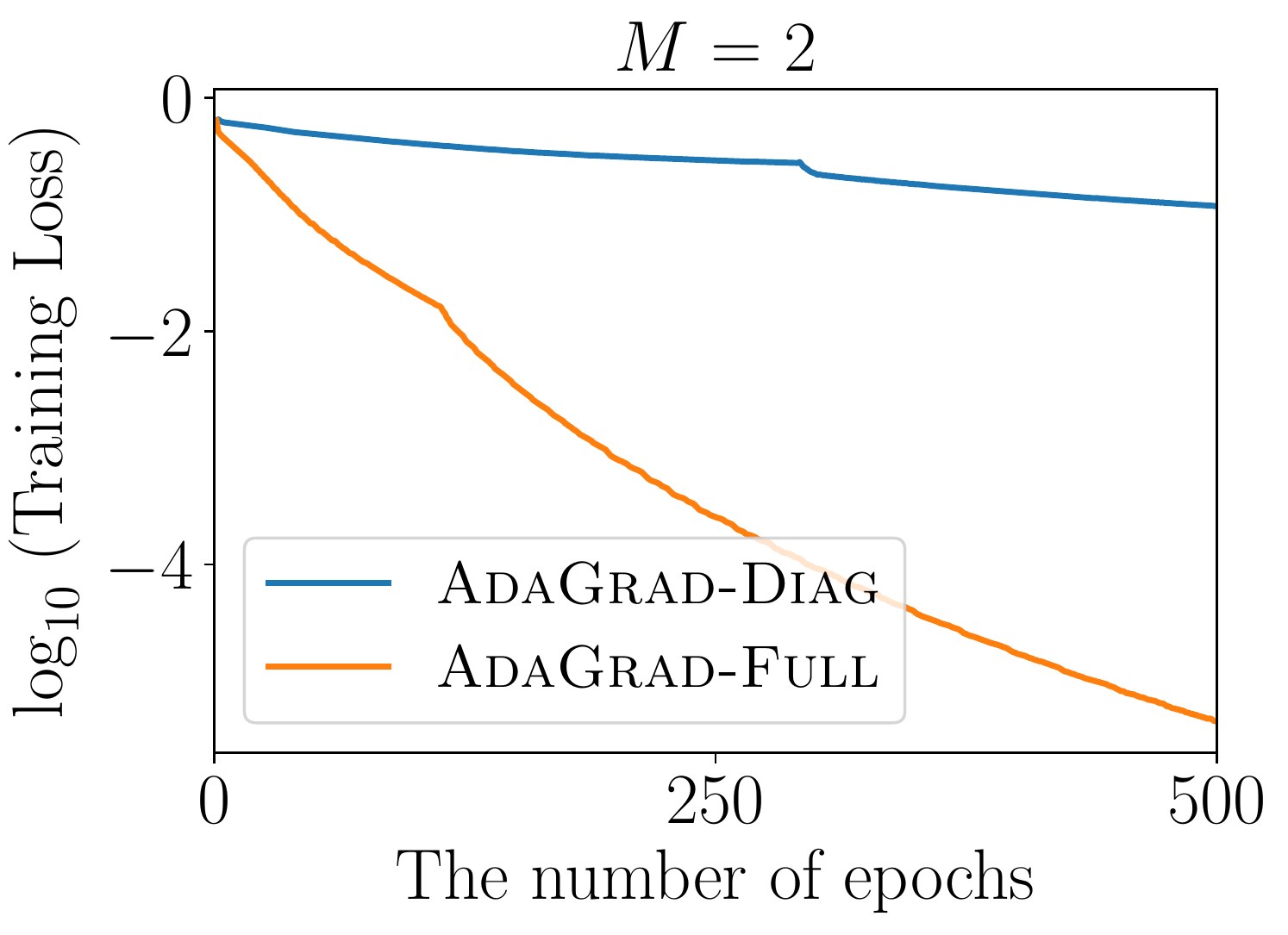}}
	\subfigure{\includegraphics[width=0.33\linewidth]{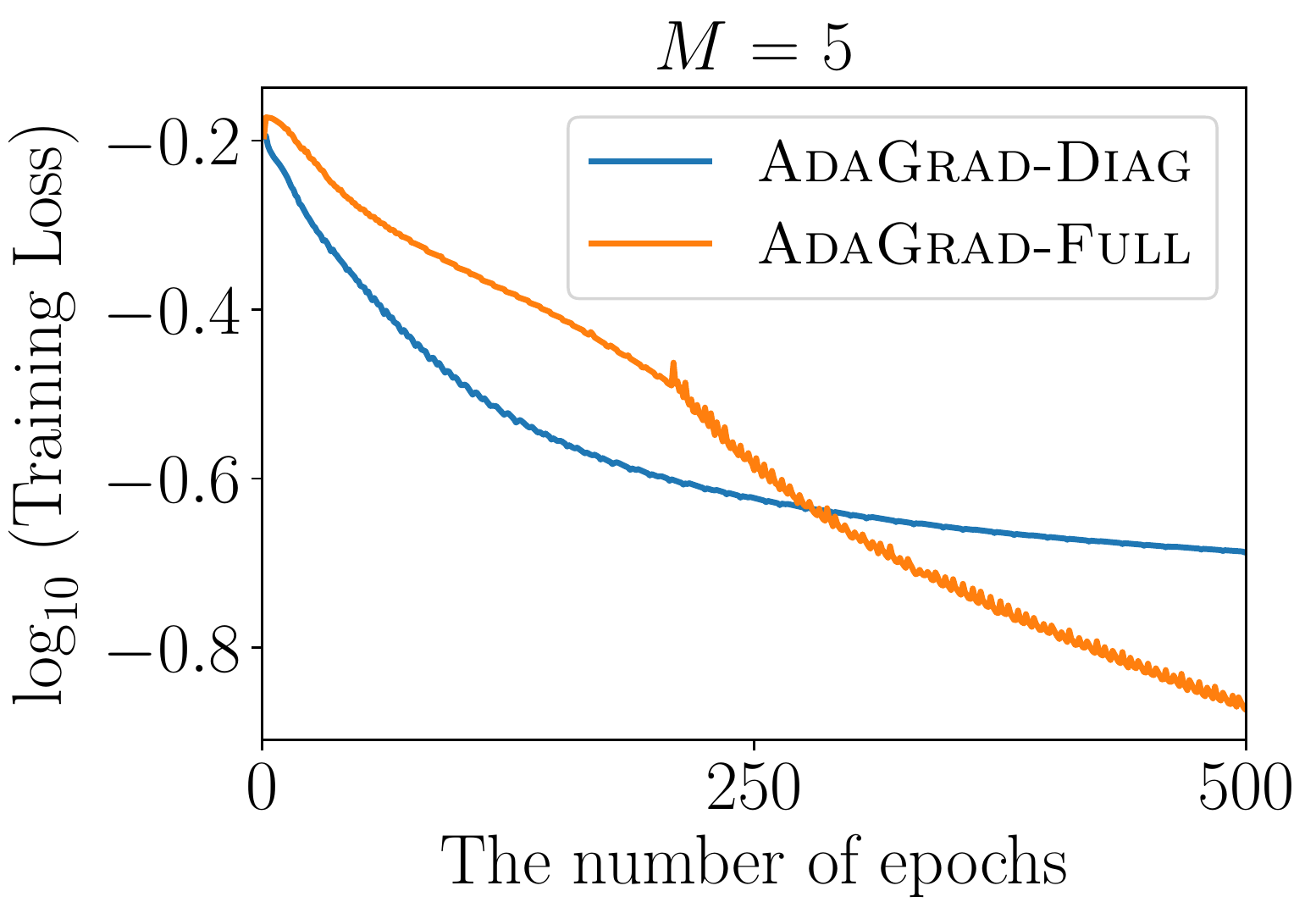}}
	\caption{Comparison of \textsc{AdaGrad} diagonal version and full matrix version varying the minibatch size $M$.} 
	\label{fig:alg_compare}
\end{figure}

\section{Computations and memory considerations} 
Compared with full matrix adaptation, working with a block diagonal matrix is computationally more efficient as it allows for decoupling computations with respect to each small full sub-matrix. %\textcolor{blue}{In addition, the ''inverse'' of a block diagonal matrix is much more efficient in terms of computations since we only compute the ''inverse'' of each small sub-matrix.} 
%\textcolor{red}{Rather than saying inverse, we might want to say computing update (1), b/c we never actually have to compute the inverse per se?} \textcolor{blue}{Jihun: Yeah, that will be better since the main computational bottleneck is in computing square root of a matrix.} 
In Algorithm \ref{alg:block_diagonal_update}, the procedures for constructing the block diagonal matrix and for updating parameters for each block by computing the ``inverse'' square root of each sub-matrix can be done in a parallel manner. As the group size increases, the block diagonal matrix becomes closer to the full matrix, resulting in greater computational cost. Therefore, we consider small group size for our numerical experiments. Although the wall clock time of \textsc{Block-Adam} for experiments on CIFAR dataset is about two times more than the diagonal counterpart, our results show great improvement in generalization.
In terms of memory, our method is more efficient than \textsc{GGT}~\cite{agarwal2018} (the modified version of full-matrix \textsc{Adagrad}). For example, consider models with a total of $d$ parameters. For Algorithm \ref{alg:block_diagonal_update}, assume that $\widehat{V}_t$ is a block diagonal matrix with $r$ sub-matrices, and each block has size $m \times m$ (so, $rm = d$). Also, assume that the truncated window size for GGT is $w$. GGT needs a memory size of $\mathcal{O}(wd)$, and our algorithm requires  $\mathcal{O}(rm^2) = \mathcal{O}(md)$. We consider small group size $m=10$ or $25$ for our experiments while the recommended window size of GGT is 200~\cite{agarwal2018}. Therefore, our algorithm is more memory-efficient and the benefit is more pronounced as the number of model parameters $d$ is large, which is the case in  popular deep learning models/architectures.

\section{Hyperparameters and Additional Experimental Results}

We use the recommended step size or tune it in the range $[10^{-4}, 10^{2}]$ for all comparison algorithms. For $\textsc{Adam}$ based algorithms, we use default decay parameters $(\beta_1, \beta_2) = (0.9, 0.999)$. For a diagonal version of \textsc{Adam} variant algorithm, we choose numerical stability parameter $\epsilon = 10^{-4}$ (to match our choice of $\delta = 10^{-4}$) for fair comparison since the larger value of $\epsilon$ can improve the generalization performance as discussed in~\citep{zaheer2018}. For $\gamma$ and $\alpha^{*}$ in clipping bound functions, we consider $\gamma \in \{10^{-3}, 10^{-4}\}$ and choose $\alpha^{*} \in \{\alpha_{\textrm{SGD}}, 10\alpha_{\textrm{SGD}}\}$ where $\alpha_{\textrm{SGD}}$ is the best-performing initial learning rate for vanilla \textsc{Sgd} (These hyperparameter candidates are based on the empirical studies in \cite{luo2019}). As in \cite{luo2019}, our results are also not sensitive to choice of $\gamma$ and $\alpha^{*}$. With these hyperparameters, we consider maximum 300 epochs training time, and mini-batch size or learning rate scheduling are introduced in each experiment description. Our Algorithm \ref{alg:block_diagonal} requires $\mathrm{SVD}$ procedures to compute the square root of a block diagonal matrix. We apply $\mathrm{SVD}$ efficiently to all small sub-matrices simultaneously through batch mode of $\mathrm{SVD}$.

\paragraph{Generative Models.} For experiments on generative models $\beta$-TCVAE, we use the author's implementation only replacing the \textsc{Adam} optimizer with our \textsc{Block-Adam}. We use convolutional networks for encoder-decoder and mini-batch size 2048. 

\paragraph{CIFAR classification.} According to experiment settings in~\citep{huang2017}, we use mini-batch size 64 and consider maximum 300 epochs. Also, we use a \emph{step-decay} learning rate scheduling in which the learning rate is divided by $10$ at $50\%$ and $75\%$ of the total number of training epochs. With this setting, vanilla \textsc{Sgd} with a momentum factor 0.9 performs best with initial learning rate $\alpha^{*} = 0.1$, so we use this value for our bound functions of spectrum-clipping, $\lambda_l(t)$ and $\lambda_u(t)$.

\paragraph{Language Models.} In this experiment, we use a \emph{dev-decay} learning rate scheduling~\cite{wilson2017} where we reduce learning rate by a constant factor if the model does not attain a new best validation performance at each epoch as in~\cite{zaremba2014}. Under this setting, vanilla \textsc{Sgd} performs best when the initial learning rate $\alpha^{*} = 5$.

\begin{figure}[!h]
	\centering
	\subfigure{\includegraphics[width=0.495\linewidth]{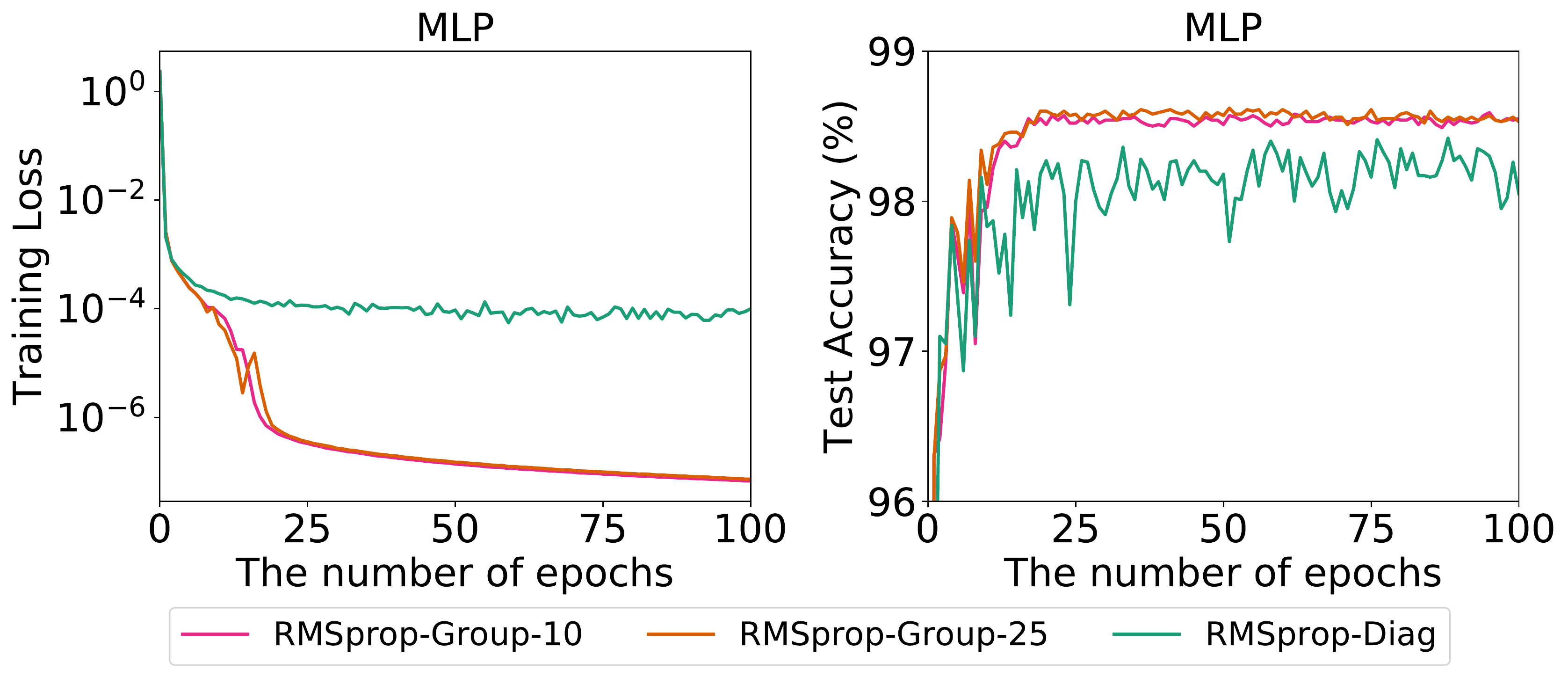}}
	\subfigure{\includegraphics[width=0.495\linewidth]{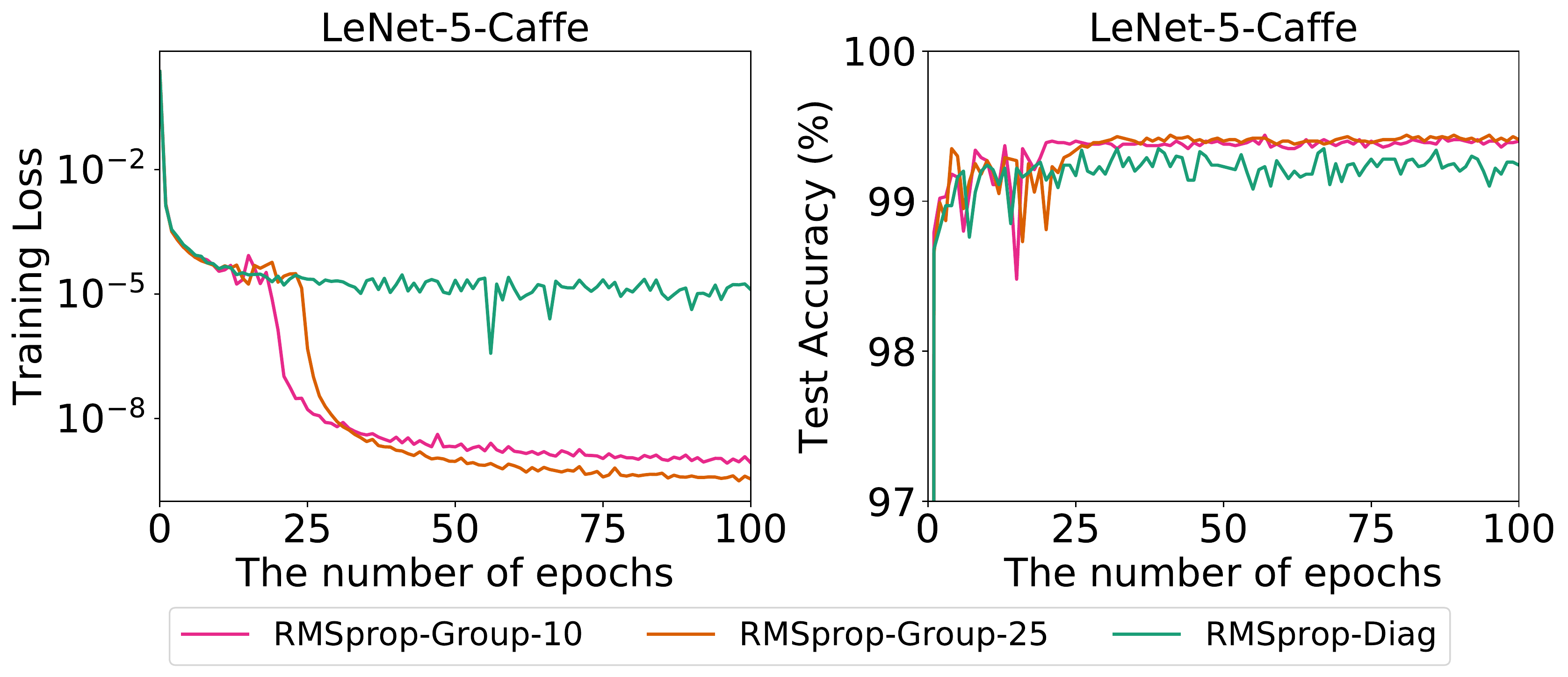}}
	\caption{Training curve and test accuracy for MLP/LeNet-5-Caffe with \textsc{RMSprop}.} 
	\label{fig:shallow_models_rmsprop}
\end{figure}

\begin{figure}[!h]
	\centering
	\subfigure{\includegraphics[width=\linewidth]{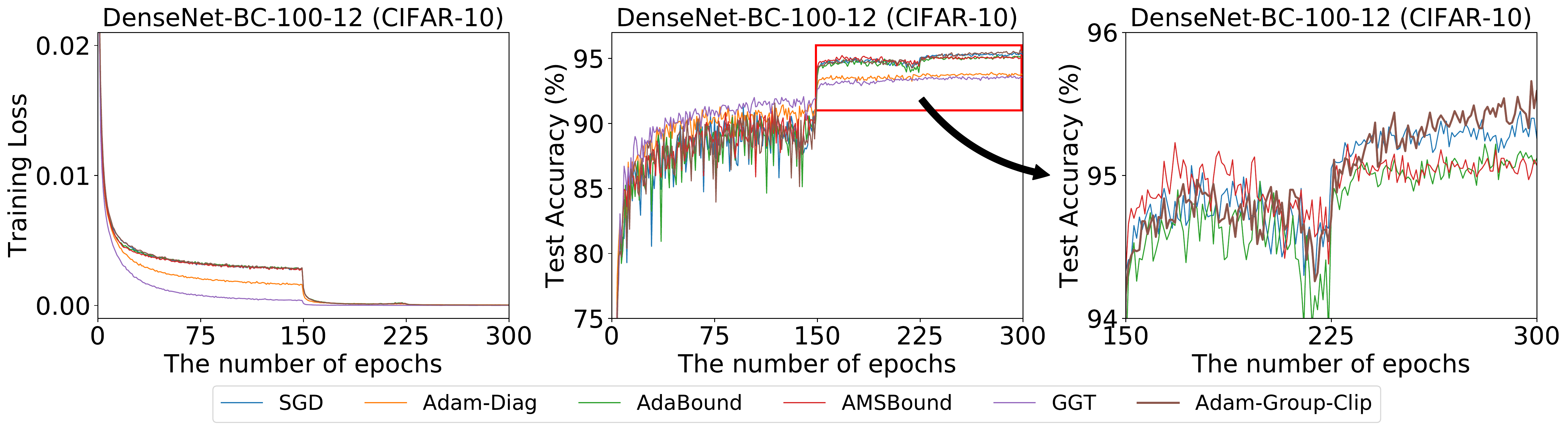}}
	\caption{Training curve and test accuracy for DenseNet-BC-100-12 on CIFAR-10 dataset.} 
	\label{fig:densenet_cifar10}
\end{figure}

\section{General Frameworks}

\begin{algorithm}[!h]
	\caption{Adaptive Gradient Methods with Block Diagonal Matrix Adaptations via Grouping}
	\label{alg:block_diagonal}
	\begin{algorithmic}
		\State {\bfseries Input:} Stepsize $\alpha_t$, initial point $x_1 \in \mathbb{R}^{d}$, $\beta_1 \in [0, 1)$, and the function $H_t$ which designs $\widehat{V}_t$.
		\State {\bfseries Initialize:} $m_0 = 0$, $\V_0 = 0$, and $t=0$.
		\State {\bfseries Assumption:} We have $r$ blocks with each size $n_i \times n_i$ and $n_1 + \cdots + n_r = d$, and $\beta_{1,t} \geq \beta_{1, t+1}$
		\For{$t = 1, 2, \ldots, T$}
		\State{Draw a minibatch sample $\xi_t$ from $\mathbb{P}$}
		\Let{offset}{0}
		\Let{$G_t$}{0}
		\Let{$g_t$}{$\nabla f(x_t)$}
		\Let{$m_t$}{$\beta_{1,t} m_{t-1} + (1 - \beta_{1,t}) g_t$}
		\For{each group index $j = 1, 2, \ldots, r$}
		\Let{$g_t^{(j)}$}{$g_t[\text{offset}:{\text{offset} + n_j}]$}
		\Let{$G_t[\text{offset}:{\text{offset} + n_j}, \text{offset}:{\text{offset} + n_j}]$}{$g_t^{(j)} \big(g_t^{(j)}\big)^T$}
		\Let{offset}{offset + $n_j$}
		\EndFor
		\Let{$\V_t$}{$H_t(G_1, \cdots, G_t)$}
		\Let{$x_{t+1}$}{$x_t - \alpha_t (\V_t^{1/2} + \delta I)^{-1} m_t$}
		\EndFor
	\end{algorithmic}
\end{algorithm}

We provide the general frameworks of adaptive gradient methods with exact full matrix adaptations. The Algorithm \ref{alg:adaptive_diag} and \ref{alg:adaptive_full} represent the general framework for each case. We can identify algorithms according to the functions $h_t$ (Table \ref{tab:diag_general_framework}) and $H_t$ (Table \ref{tab:full_general_framework}) which determine the dynamics of $\widehat{v}_t$ and $\widehat{V}_t$ respectively. Also, the Algorithm \ref{alg:block_diagonal} is a detail version of the Algorithm \ref{alg:block_diagonal_update}.

\begin{figure*}[!h]
	\begin{minipage}[t]{0.49\textwidth}
		\begin{algorithm}[H]
			\caption{General Adaptive Gradient Methods approximating $g_t g_t^T$ via \textsc{Diagonal} Matrix}
			\label{alg:adaptive_diag}
			\begin{algorithmic}
				\State {\bfseries Input:} Initial point $x_1 \in \mathbb{R}^d$, stepsize $\{\alpha_t\}_{t=1}^{T}$, decay parameters $\beta_{1,t}, \beta_2 \in [0, 1]$, and $\epsilon > 0$. 
				\State {\bfseries Initialize:} $m_0 = 0$, $\widehat v_0 = 0$.
				\For{$t = 1, 2, \ldots, T$}
				\State{Draw a minibatch sample $\xi_t$ from $\mathbb{P}$}
				\Let{$g_t$}{$\nabla f(x_t; \xi_t)$}
				\Let{$G_t$}{$\mathrm{diag}(g_t g_t^T)$}
				\Let{$m_t$}{$\beta_{1,t} m_{t-1} + (1 - \beta_{1,t}) g_t$} 
				\vspace*{0.09cm}
				\Let{$\widehat v_t$}{$h_t(G_1, G_2, \ldots, G_t)$}
				\vspace*{0.09cm}
				\Let{$x_{t+1}$}{$x_t - \alpha_t m_t / (\sqrt{\widehat v_t} + \epsilon)$}
				\EndFor
				\State {\bfseries Output:} $\widehat x$.
			\end{algorithmic}
		\end{algorithm}
	\end{minipage}%
	\hfill
	\begin{minipage}[t]{0.49\textwidth}
		\begin{algorithm}[H]
			\caption{General Adaptive Gradient Methods with the exact $g_t g_t^T$ (\textsc{Full} Matrix)}
			\label{alg:adaptive_full}
			\begin{algorithmic}
				\State {\bfseries Input:} Initial point $x_1 \in \mathbb{R}^d$, stepsize $\{\alpha_t\}_{t=1}^{T}$, decay parameters $\beta_{1,t}, \beta_2 \in [0, 1]$, and $\delta > 0$. 
				\State {\bfseries Initialize:} $m_0 = 0$, $\widehat V_0 = 0$.
				\For{$t = 1, 2, \ldots, T$}
				\State{Draw a minibatch sample $\xi_t$ from $\mathbb{P}$}
				\Let{$g_t$}{$\nabla f(x_t; \xi_t)$}
				\Let{$G_t$}{$g_t g_t^T$}
				\Let{$m_t$}{$\beta_{1,t} m_{t-1} + (1 - \beta_{1,t}) g_t$}
				\Let{$\widehat V_t$}{$H_t(G_1, G_2, \ldots, G_t)$}
				\Let{$x_{t+1}$}{$x_t - \alpha_t (\widehat V_t^{1/2} + \delta I)^{-1} m_t$}
				\EndFor
				\State {\bfseries Output:} $\widehat x$.
			\end{algorithmic}
		\end{algorithm}
	\end{minipage}
\end{figure*}

\begin{table}[!h]
	\caption{Variants of diagonal matrix adaptations}
	\label{tab:diag_general_framework}
	\centering
	\begin{tabular}{ccc}
		\toprule
		%\multicolumn{2}{c}{Part}                   \\
		%\cmidrule(r){1-2}
		\backslashbox{$\widehat v_t$}{$\beta_{1,t}$} & $\beta_{1,t} = 0$ & $\beta_{1,t} = \beta_1$ \\
		\midrule
		$1$ & \textsc{Sgd} & - \\
		\midrule
		$(1/t)\sum_{t=1}^{T} g_t^2$ & \textsc{AdaGrad} & \textsc{AdaFom} \\
		\midrule
		$\beta_2 \widehat v_{t-1} + (1 - \beta_2) g_t^2 $ & \textsc{RMSProp} & \textsc{Adam} \\
		\midrule
		$v_t = \beta_2 v_{t-1} + (1 - \beta_2) g_t^2$, ~\\$\widehat{v}_t = \max\{\widehat{v}_{t-1}, v_t\}$ & - & \textsc{AMSGrad} \\ 
		\bottomrule
	\end{tabular}
\end{table}

\begin{table}[!h]
	\caption{Variants of full matrix adaptations}
	\label{tab:full_general_framework}
	\centering
	\begin{tabular}{ccc}
		\toprule
		%\multicolumn{2}{c}{Part}                   \\
		%\cmidrule(r){1-2}
		\backslashbox{$\widehat V_t$}{$\beta_{1,t}$} & $\beta_{1,t} = 0$ & $\beta_{1,t} = \beta_1$ \\
		\midrule
		$\widehat V_t = I$ & \textsc{Sgd} & - \\
		\midrule
		$\widehat V_t = \frac{1}{T}\sum_{t=1}^{T} g_t g_t^T$ & \textsc{AdaGrad} & \textsc{AdaFom} \\
		\midrule
		$\widehat V_t = \beta_2 \widehat V_{t-1} + (1 - \beta_2) g_t g_t^T $ & \textsc{RMSProp} & \textsc{Adam} \\
		\midrule
		$V_t = U_t \Sigma_t U_t^T$, ~\\ $\widehat V_t = U_t \max\{\widehat \Sigma_{t-1}, \Sigma_t \} U_t^T$ & - & \textsc{AMSGrad} \\ 
		\bottomrule
	\end{tabular}
\end{table}

%%%%% proofs of the theorem 1 %%%%%
\section{Proofs of Main Theorems}

We study the following minimization problem,
\begin{align*}
\min f(x) \coloneqq \mathbb{E}_{\xi}[f(x; \xi)]
\end{align*}
under the assumption \ref{assumption1}. The parameter $x$ is an optimization variable, and $\xi$ is a random variable representing randomly selected data sample from $\mathcal{D}$. We study the convergence analysis of the algorithm \ref{alg:block_diagonal}. 
For analysis in stochastic convex optimization, one can refer to \cite{duchi2011}. For analysis in non-convex optimization with full matrix adaptations, we follow the arguments in the paper \cite{chen2019}. As we will show, the convergence of the adaptive full matrix adaptations depends on the changes of \emph{effective spectrum} while the diagonal counterpart depends on the changes of \emph{effective stepsize}. We assume that $\delta I$ is absorbed into $\V^{-1/2}$ for convenience of notations. Note that, our proof can be applied to exact full matrix adaptations, algorithm \ref{alg:adaptive_full}.

\subsection{Technical Lemmas for Theorem \ref{thm1}}
\begin{lemma}{\label{lemma1}}
	Consider the sequence
	\begin{align*}
	z_t = x_t + \frac{\beta_{1,t}}{1 - \beta_{1,t}} (x_t - x_{t-1})
	\end{align*}
	Then, the following holds true
	\begin{align*}
	z_{t+1} - z_t = - \Bigg(\frac{\beta_{1, t+1}}{1 - \beta_{1, t+1}} - \frac{\beta_{1,t}}{1 - \beta_{1,t}} \Bigg)\alpha_t \V_t^{-1/2}m_t -\frac{\beta_{1,t}}{1 - \beta_{1,t}} \Bigg(\alpha_t \V_t^{-1/2} - \alpha_{t-1} \V_{t-1}^{-1/2}\Bigg)m_{t-1} - \alpha_t \V_t^{-1/2} g_t
	\end{align*}
\end{lemma}

\begin{proof}
	By our update rule, we can derive
	\begin{align*}
	x_{t+1} - x_t =~ & -\alpha_t \V_t^{-1/2} m_t \\
	\overset{(i)}{=}~ & -\alpha_t \V_t^{-1/2} (\beta_{1,t} m_{t-1} + (1 - \beta_{1,t}) g_t) \\
	=~ & -\alpha_t \beta_{1,t} \V_t^{-1/2} m_{t-1} - \alpha_t (1 - \beta_{1,t}) \V_t^{-1/2} g_t \\
	\overset{(ii)}{=} & -\alpha_t \beta_{1,t} \V_t^{-1/2} \Big(-\frac{1}{\alpha_{t-1}} \V_{t-1}^{1/2}(x_t - x_{t-1})\Big) - \alpha_t (1 - \beta_{1,t}) \V_t^{-1/2} g_t \\
	=~ & \frac{\alpha_t}{\alpha_{t-1}} \beta_{1,t} (\V_t^{-1} \V_{t-1})^{1/2} (x_t - x_{t-1}) - \alpha_t (1 - \beta_{1,t})\V_t^{-1/2} g_t \\ 
	=~ & \beta_{1,t} (x_t - x_{t-1}) + \beta_{1,t} \Bigg(\frac{\alpha_t}{\alpha_{t-1}}(\V_t^{-1}\V_{t-1})^{1/2} - I_d\Bigg) (x_t - x_{t-1}) - \alpha_t (1 - \beta_{1,t})\V_t^{-1/2} g_t \\
	\overset{(iii)}{=}~ & \beta_{1,t} (x_t - x_{t-1}) - \beta_{1,t} \Bigg(\alpha_t \V_t^{-1/2} - \alpha_{t-1} \V_{t-1}^{-1/2}\Bigg) m_{t-1} - \alpha_t (1 - \beta_{1,t})\V_t^{-1/2} g_t \\
	\end{align*}
	The reasoning follows
	\begin{enumerate}[label=(\roman*)]
		\item By definition of $m_t$.
		\item Since $x_t = x_{t-1} - \alpha_{t-1} \V_{t-1}^{-1/2} m_{t-1}$, we can solve as $m_{t-1} = -\frac{1}{\alpha_{t-1}}\V_{t-1}^{1/2}(x_t - x_{t-1})$.
		\item Similarly to (ii), we can have $\V_{t-1}^{1/2}(x_t - x_{t-1})/\alpha_{t-1} = -m_{t-1}$.
	\end{enumerate}
	Since $x_{t+1} - x_t = (1 - \beta_{1,t}) x_{t+1} + \beta_{1,t}(x_{t+1} - x_t) - (1 - \beta_{1,t}) x_t$, we can further derive by combining the above,
	\begin{align*}
	& (1 - \beta_{1,t}) x_{t+1} + \beta_{1,t}(x_{t+1} - x_t) \\
	=~& (1 - \beta_{1,t}) x_t + \beta_{1,t} (x_t - x_{t-1}) - \beta_{1,t} \Bigg(\alpha_t \V_t^{-1/2} - \alpha_{t-1} \V_{t-1}^{-1/2}\Bigg) m_{t-1} - \alpha_t (1 - \beta_{1,t})\V_t^{-1/2} g_t
	\end{align*}
	By dividing both sides by $1 - \beta_{1,t}$,
	\begin{align*}
	& x_{t+1} + \frac{\beta_{1,t}}{1 - \beta_{1,t}} (x_{t+1} - x_t) \\
	=~& x_t + \frac{\beta_{1,t}}{1 - \beta_{1,t}} (x_t - x_{t-1}) - \frac{\beta_{1,t}}{1 - \beta_{1,t}} \Bigg(\alpha_t \V_t^{-1/2} - \alpha_{t-1} \V_{t-1}^{-1/2}\Bigg)m_{t-1} - \alpha_t \V_t^{-1/2} g_t 
	\end{align*}
	Define the sequence
	\begin{align*}
	z_t = x_t + \frac{\beta_{1,t}}{1 - \beta_{1,t}} (x_t - x_{t-1})
	\end{align*}
	Then, we obtain
	\begin{align*}
	z_{t+1} & = z_t + \Bigg(\frac{\beta_{1, t+1}}{1 - \beta_{1, t+1}} - \frac{\beta_{1,t}}{1 - \beta_{1,t}} \Bigg) (x_{t+1} - x_t) \\
	&~~~~ -\frac{\beta_{1,t}}{1 - \beta_{1,t}} \Bigg(\alpha_t \V_t^{-1/2} - \alpha_{t-1} \V_{t-1}^{-1/2}\Bigg)m_{t-1} - \alpha_t \V_t^{-1/2} g_t \\
	& = z_t - \Bigg(\frac{\beta_{1, t+1}}{1 - \beta_{1, t+1}} - \frac{\beta_{1,t}}{1 - \beta_{1,t}} \Bigg)\alpha_t \V_t^{-1/2}m_t \\
	&~~~~ -\frac{\beta_{1,t}}{1 - \beta_{1,t}} \Bigg(\alpha_t \V_t^{-1/2} - \alpha_{t-1} \V_{t-1}^{-1/2}\Bigg)m_{t-1} - \alpha_t \V_t^{-1/2} g_t \\
	\end{align*}
	By putting $z_t$ to the left hand side, we can get desired relations.
\end{proof}

\begin{lemma}{\label{lemma2}}
	Suppose that the assumptions in Theorem \ref{thm1} hold, then
	\begin{align*}
	\mathbb{E}[f(z_{t+1}) - f(z_1)] \leq \sum\limits_{i=1}^{6} T_i
	\end{align*}
	where
	\begin{align*}
	T_1 & = -\mathbb{E}\Bigg[\sum\limits_{i=1}^{t} \Bigg\langle \nabla f(z_i), \frac{\beta_{1,i}}{1 - \beta_{1,i}} \Bigg(\alpha_i \V_i^{-1/2} - \alpha_{i-1}\V_{i-1}^{-1/2}\Bigg)m_{i-1} \Bigg\rangle \Bigg] \\
	T_2 & = -\mathbb{E}\Bigg[\sum\limits_{i=1}^{t} \alpha_i \Bigg\langle \nabla f(z_i), \V_i^{-1/2}g_i \Bigg\rangle \Bigg] \\
	T_3 & = -\mathbb{E}\Bigg[\sum\limits_{i=1}^{t} \Bigg\langle \nabla f(z_i), \Bigg(\frac{\beta_{1, i+1}}{1 - \beta_{1, i+1}} - \frac{\beta_{1,i}}{1 - \beta_{1,i}}\Bigg) \alpha_i \V_i^{-1/2}m_i  \Bigg\rangle \Bigg] \\
	T_4 & = \mathbb{E}\Bigg[\sum\limits_{i=1}^{t} \frac{3}{2}L \Bigg\Vert \Bigg(\frac{\beta_{1, i+1}}{1 - \beta_{1, i+1}} - \frac{\beta_{1,i}}{1 - \beta_{1,i}} \Bigg)\alpha_t \V_i^{-1/2}m_i \Bigg\Vert^2\Bigg] \\
	T_5 & = \mathbb{E}\Bigg[\sum\limits_{i=1}^{t} \frac{3}{2}L \Bigg\Vert \frac{\beta_{1,i}}{1 - \beta_{1,i}} \Bigg(\alpha_i \V_i^{-1/2} - \alpha_{i-1} \V_{i-1}^{-1/2}\Bigg)m_{i-1} \Bigg\Vert^2\Bigg] \\
	T_6 & = \mathbb{E}\Bigg[\sum\limits_{i=1}^{t} \frac{3}{2}L \Bigg\Vert \alpha_i \V_i^{-1/2} g_i \Bigg\Vert^2\Bigg]
	\end{align*}
\end{lemma}

\begin{proof}
	By $L$-Lipschitz continuous gradients, we get the following quadratic upper bound,
	\begin{align*}
	f(z_{t+1}) \leq f(z_t) + \langle \nabla f(z_t), z_{t+1} - z_t \rangle + \frac{L}{2} \|z_{t+1} - z_t\|^2
	\end{align*}
	Let $d_t = z_{t+1} - z_t$. The lemma \ref{lemma1} yields
	\begin{align*}
	d_t = - \Bigg(\frac{\beta_{1, t+1}}{1 - \beta_{1, t+1}} - \frac{\beta_{1,t}}{1 - \beta_{1,t}} \Bigg)\alpha_t V_t^{-1/2}m_t -\frac{\beta_{1,t}}{1 - \beta_{1,t}} \Bigg(\alpha_t V_t^{-1/2} - \alpha_{t-1} V_{t-1}^{-1/2}\Bigg)m_{t-1} - \alpha_t V_t^{-1/2} g_t
	\end{align*}
	Combining with Lipschitz continuous gradients, we have
	\begin{align*}
	\mathbb{E}[f(z_{t+1}) - f(z_1)] & = \mathbb{E}\Bigg[\sum\limits_{i=1}^{t} f(z_{i+1}) - f(z_i)\Bigg] \\
	& \leq \mathbb{E}\Bigg[\sum\limits_{i=1}^{t} \langle \nabla f(z_i), d_i \rangle + \frac{L}{2} \|d_i\|^2\Bigg] \\
	& = -\mathbb{E}\Bigg[\sum\limits_{i=1}^{t} \Bigg\langle \nabla f(z_i), \frac{\beta_{1,i}}{1 - \beta_{1,i}} \Bigg(\alpha_i V_i^{-1/2} - \alpha_{i-1}V_{i-1}^{-1/2}\Bigg) m_{i-1} \Bigg\rangle \Bigg] \\
	&~~~ -\mathbb{E}\Bigg[\sum\limits_{i=1}^{t} \alpha_i \Bigg\langle \nabla f(z_i), V_i^{-1/2}g_i \Bigg\rangle \Bigg] \\
	&~~~ -\mathbb{E}\Bigg[\sum\limits_{i=1}^{t} \Bigg\langle \nabla f(z_i), \Bigg(\frac{\beta_{1, i+1}}{1 - \beta_{1, i+1}} - \frac{\beta_{1,i}}{1 - \beta_{1,i}}\Bigg) \alpha_i V_i^{-1/2}m_i  \Bigg\rangle \Bigg] \\
	&~~~ +\mathbb{E}\Bigg[\sum\limits_{i=1}^{t} \frac{L}{2} \|d_i\|^2\Bigg] = T_1 + T_2 + T_3 + \mathbb{E}\Bigg[\sum\limits_{i=1}^{t} \frac{L}{2} \|d_i\|^2\Bigg]
	\end{align*}
	With $\|a + b + c\|^2 \leq 3(\|a\|^2 + \|b\|^2 + \|c\|^2)$, we can finally bound by
	\begin{align*}
	\mathbb{E}[f(z_{t+1}) - f(z_1)] \leq \sum\limits_{i=1}^{6} T_i
	\end{align*}
\end{proof}

\begin{lemma}{\label{lemma3}}
	Suppose that the assumptions in Theorem \ref{thm1} hold, $T_1$ can be bound as
	\begin{align*}
	T_1 \leq G_\infty^2 \frac{\beta_1}{1 - \beta_1} \mathbb{E}\Bigg[\sum\limits_{i=1}^{t} \matnormBigg{\alpha_i V_i^{-1/2} - \alpha_{i-1}V_{i-1}^{-1/2}}{2} \Bigg]
	\end{align*}
\end{lemma}

\begin{proof}
	From the definition of quantity $T_1$,
	\begin{align*}
	T_1 =~ & -\mathbb{E}\Bigg[\sum\limits_{i=1}^{t} \Bigg\langle \nabla f(z_i), \frac{\beta_{1,i}}{1 - \beta_{1,i}} \Bigg(\alpha_i V_i^{-1/2} - \alpha_{i-1}V_{i-1}^{-1/2}\Bigg)m_{i-1} \Bigg\rangle \Bigg] \\
	\overset{(i)}{\leq}~ & \mathbb{E}\Bigg[\sum\limits_{i=1}^{t} \|\nabla f(z_i)\|_2 \Bigg\|\frac{\beta_{1,i}}{1 - \beta_{1,i}}\Bigg(\alpha_i V_i^{-1/2} - \alpha_{i-1}V_{i-1}^{-1/2}\Bigg)m_{i-1}\Bigg\|_2\Bigg] \\
	\overset{(ii)}{\leq}~ & \frac{\beta_1}{1 - \beta_1}\mathbb{E}\Bigg[\sum\limits_{i=1}^{t} \|\nabla f(z_i)\|_2 \matnormBig{\alpha_i V_i^{-1/2} - \alpha_{i-1}V_{i-1}^{-1/2}}{2} \|m_{t-1}\|_2 \Bigg] \\
	\overset{(iii)}{\leq}~ & G_\infty^2 \frac{\beta_1}{1 - \beta_1} \mathbb{E}\Bigg[\sum\limits_{i=1}^{t} \matnormBigg{\alpha_i V_i^{-1/2} - \alpha_{i-1}V_{i-1}^{-1/2}}{2} \Bigg]
	\end{align*}
	The reasoning follows
	\begin{enumerate}[label=(\roman*)]
		\item By Cauchy-Schwarz inequality.
		\item For a matrix norm, we have $\|Ax\|_2 \leq \matnorm{A}{2} \|x\|_2$. Also, $\frac{\beta_{1,i}}{1 - \beta_{1,i}} = \frac{1}{1 - \beta_{1,i}} - 1 \leq \frac{1}{1 - \beta_1} - 1 = \frac{\beta_1}{1 - \beta_1}$.
		\item By definition of $m_t$, we have $m_t = \beta_{1,t} m_{t-1} + (1 - \beta_{1,t}) g_t$. Therefore, we use a triangle inequality by $\|m_t\|_2 \leq \beta_{1,t} \|m_{t-1}\|_2 + (1 - \beta_1) \|g_t\|_2 \leq (\beta_{1,t} + 1 - \beta_{1,t}) \max \{\|m_{t-1}\|_2, \|g_t\|_2\}$. Since we have $m_0 = 0$ and $\|g_t\| \leq G_\infty$, we also have $\|m_t\| \leq G_\infty$ by the mathematical induction.
	\end{enumerate}
\end{proof}

\begin{lemma}{\label{lemma4}}
	Suppose that the assumptions in Theorem \ref{thm1} hold, then $T_3$ can be bound as
	\begin{align*}
	T_3 \leq \Bigg(\frac{\beta_1}{1 - \beta_1} - \frac{\beta_{1, t+1}}{1 - \beta_{1, t+1}} \Bigg) (G_\infty^2 + D_\infty^2)
	\end{align*}
\end{lemma}

\begin{proof}
	By the definition of $T_3$,
	\begin{align*}
	T_3 =~ & -\mathbb{E}\Bigg[\sum\limits_{i=1}^{t} \Bigg\langle \nabla f(z_i), \Bigg(\frac{\beta_{1, i+1}}{1 - \beta_{1, i+1}} - \frac{\beta_{1,i}}{1 - \beta_{1,i}}\Bigg) \alpha_i V_i^{-1/2}m_i  \Bigg\rangle \Bigg] \\
	\overset{(i)}{\leq}~ & \mathbb{E}\Bigg[\sum\limits_{i=1}^{t} \Bigg|\frac{\beta_{1, i+1}}{1 - \beta_{1, i+1}} - \frac{\beta_{1, i}}{1 - \beta_{1,i}} \Bigg| \frac{1}{2} \Big(\|\nabla f(z_i)\|^2 + \|\alpha_i V_i^{-1/2} m_i\|^2 \Big) \Bigg] \\
	\overset{(ii)}{\leq}~ & \mathbb{E}\Bigg[\sum\limits_{i=1}^{t} \Bigg|\frac{\beta_{1, i+1}}{1 - \beta_{1, i+1}} - \frac{\beta_{1, i}}{1 - \beta_{1,i}} \Bigg| \frac{1}{2} \Big(G_\infty^2 + D_\infty^2\Big) \Bigg] \\
	=~ & \sum\limits_{i=1}^{t} \Bigg(\frac{\beta_{1, i}}{1 - \beta_{1,i}} - \frac{\beta_{1, i+1}}{1 - \beta_{1, i+1}}\Bigg)  \frac{1}{2} \Big(G_\infty^2 + D_\infty^2\Big) \\
	\overset{(iii)}{\leq}~ & \Bigg(\frac{\beta_1}{1 - \beta_1} - \frac{\beta_{1, t+1}}{1 - \beta_{1, t+1}} \Bigg) (G_\infty^2 + D_\infty^2)
	\end{align*}
	The reasoning follows
	\begin{enumerate}[label=(\roman*)]
		\item Use Cauchy-Schwarz inequality and $ab \leq \frac{1}{2} (a^2 + b^2)$ for $a, b \geq 0$.
		\item By our assumptions on bounded gradients and bounded final step vectors.
		\item The sum over $i = 1$ to $T$ can be done by telescoping.
	\end{enumerate}
\end{proof}

\begin{lemma}{\label{lemma5}}
	Suppose that the assumptions in Theorem \ref{thm1} hold, $T_4$ can be bound as
	\begin{align*}
	T_4 \leq \Big(\frac{\beta_1}{1 - \beta_1} - \frac{\beta_{1, t+1}}{1 - \beta_{1, t+1}}\Big)^2 D_\infty^2
	\end{align*}
\end{lemma}

\begin{proof}
	By the definition of $T_4$,
	\begin{align*}
	\frac{2}{3L}T_4 =~ & \mathbb{E}\Bigg[\sum\limits_{i=1}^{t} \Bigg\Vert \Bigg(\frac{\beta_{1, i+1}}{1 - \beta_{1, i+1}} - \frac{\beta_{1,i}}{1 - \beta_{1,i}} \Bigg)\alpha_i V_i^{-1/2}m_i \Bigg\Vert^2\Bigg] \\
	\overset{(i)}{\leq}~ & \mathbb{E}\Bigg[\sum\limits_{i=1}^{t}\Big(\frac{\beta_{1, i+1}}{1 - \beta_{1, i+1}} - \frac{\beta_{1, i}}{1 - \beta_{1, i}}\Big)^2 D_\infty^2 \Bigg] \\
	\overset{(ii)}{\leq}~ & \Big(\frac{\beta_1}{1 - \beta_1} - \frac{\beta_{1, t+1}}{1 - \beta_{1, t+1}}\Big)\sum\limits_{i=1}^{t}\Big(\frac{\beta_{1, i+1}}{1 - \beta_{1, i+1}} - \frac{\beta_{1, i}}{1 - \beta_{1, i}}\Big) D_\infty^2 \\
	\overset{(iii)}{\leq}~ & \Big(\frac{\beta_1}{1 - \beta_1} - \frac{\beta_{1, t+1}}{1 - \beta_{1, t+1}}\Big)^2 D_\infty^2
	\end{align*}
	The reasoning follows
	\begin{enumerate}[label=(\roman*)]
		\item From our assumptions on final step vector $\|\alpha_i \widehat{V}_i^{-1/2} m_i\|^2 \leq D_\infty$.
		\item We use the relation $\beta_1 \geq \beta_{1, t} \leq \beta_{1, t+1}$.
		\item By telescoping sum, we can get the final result.
	\end{enumerate}
\end{proof}

\begin{lemma}{\label{lemma6}}
	Suppose that the assumptions in Theorem \ref{thm1} hold, $T_5$ can be bound as
	\begin{align*}
	\frac{2}{3L}T_5 \leq \Big(\frac{\beta_1}{1 - \beta_1}\Big)^2 G_\infty^2 \mathbb{E}\Bigg[\sum\limits_{i=2}^{t} \matnormBigg{\alpha_i \V_i^{-1/2} - \alpha_{i-1} \V_{i-1}^{-1/2}}{2}\Bigg]
	\end{align*}
\end{lemma}

\begin{proof}
	By the definition of $T_5$,
	\begin{align*}
	\frac{2}{3L}T_5 =~ & \mathbb{E}\Bigg[\sum\limits_{i=2}^{t} \Bigg\Vert \frac{\beta_{1,i}}{1 - \beta_{1,i}} \Bigg(\alpha_i V_i^{-1/2} - \alpha_{i-1} V_{i-1}^{-1/2}\Bigg)m_{i-1} \Bigg\Vert^2\Bigg] \\
	\overset{(i)}{\leq}~ & \mathbb{E}\Bigg[\sum\limits_{i=2}^{t} \frac{\beta_{1,i}}{1 - \beta_{1,i}}\matnormBigg{\alpha_i \V_i^{-1/2} - \alpha_{i-1}\V_{i-1}^{-1/2}}{2}^2 \|m_{i-1}\|_2^2\Bigg] \\
	\overset{(ii)}{\leq}~ & \Big(\frac{\beta_1}{1 - \beta_1}\Big)^2 G_\infty^2 \mathbb{E}\Bigg[\sum\limits_{i=2}^{t} \matnormBigg{\alpha_i \V_i^{-1/2} - \alpha_{i-1} \V_{i-1}^{-1/2}}{2}^2\Bigg]
	\end{align*}
	The reasoning follows
	\begin{enumerate}[label=(\roman*)]
		\item By the matrix norm inequality, we use $\|Ax\|_2 \leq \matnorm{A}{2} \|x\|_2$.
		\item We can obtain the result using $\beta_1 \geq \beta_{1, t} \geq \beta_{1, t+1}$.
	\end{enumerate}
\end{proof}

\begin{lemma}{\label{lemma7}}
	Suppose that the assumptions in Theorem \ref{thm1} hold, The quantity $T_2$ can be bound as
	\begin{align*}
	T_2 & \leq L^2 \Big(\frac{\beta_1}{1 - \beta_1}\Big)^2 T_8 + L^2 \Big(\frac{\beta_1}{1 - \beta_1}\Big)^2 T_9
	+ \frac{1}{2} \mathbb{E}\Bigg[\sum\limits_{i=1}^{t} \|\alpha_i \V_i^{-1/2} g_i \|^2 \Bigg] \\
	& ~~~+ 2G_\infty^2 \mathbb{E} \Bigg[\sum\limits_{i=2}^{t} \matnormBigg{\alpha_i V_i^{-1/2} - \alpha_{i-1} V_{i-1}^{-1/2}}{2}\Bigg] + 2G_\infty^2 \mathbb{E} \Bigg[\matnormBigg{\alpha_1 V_1^{-1/2}}{2}\Bigg] \\
	& ~~~- \mathbb{E}\Bigg[\sum\limits_{i=1}^{t} \alpha_i \Bigg\langle \nabla f(x_i), V_i^{-1/2} \nabla f(x_i)\Bigg\rangle \Bigg]
	\end{align*}
\end{lemma}

\begin{proof}
	First, note that, 
	\begin{align*}
	z_i - x_i = \frac{\beta_{1,i}}{1 - \beta_{1,i}}(x_i - x_{i-1}) = -\frac{\beta_{1,i}}{1 - \beta_{1,i}} \alpha_{i-1}\V_{i-1}^{-1/2}m_{i-1}
	\end{align*}
	By the definition of $T_2$ and $z_1 = x_1$, we have
	\begin{align*}
	T_2 & = -\mathbb{E}\Bigg[\sum\limits_{i=1}^{t} \alpha_i \Bigg\langle \nabla f(z_i), \V_i^{-1/2}g_i \Bigg\rangle \Bigg] \\
	& = -\mathbb{E}\Bigg[\sum\limits_{i=1}^{t} \alpha_i \Bigg\langle \nabla f(x_i), \V_i^{-1/2} g_i \Bigg\rangle \Bigg] - \mathbb{E}\Bigg[\sum\limits_{i=1}^{t} \alpha_i \Bigg\langle \nabla f(z_i) - \nabla f(x_i), \V_i^{-1/2} g_i\Bigg\rangle \Bigg]
	\end{align*}
	The second term can be bounded as
	\begin{align*}
	& - \mathbb{E}\Bigg[\sum\limits_{i=1}^{t} \alpha_i \Bigg\langle \nabla f(z_i) - \nabla f(x_i), \V_i^{-1/2} g_i\Bigg\rangle \Bigg] \\
	\overset{(i)}{\leq} &~ \mathbb{E}\Bigg[\sum\limits_{i=1}^{t} \frac{1}{2} \|\nabla f(z_i) - \nabla f(x_i)\|^2 + \frac{1}{2}\|\alpha_i \V_i^{-1/2}g_i\|^2 \Bigg] \\
	\overset{(ii)}{\leq} &~ \frac{L^2}{2} T_7 + \frac{1}{2} \mathbb{E}\Bigg[\sum\limits_{i=1}^{t} \|\alpha_i \V_i^{-1/2} g_i\|^2 \Bigg]
	\end{align*}
	(i) is due to Cauchy-Schwarz inequality and $ab \leq \frac{1}{2} a^2 + \frac{1}{2} b^2$ for $a, b \geq 0$. (ii) is as follows:
	
	By $L$-Lipschitz continuous gradients, we have
	\begin{align*}
	\|\nabla f(z_i) - \nabla f(x_i) \| \leq L\|z_i - x_i\| = L \Big\|\frac{\beta_{1,t}}{1 - \beta_{1, t}} \alpha_{i-1}V_{i-1}^{-1/2} m_{i-1}\Big\|
	\end{align*}
	Let $T_7$ be
	\begin{align*}
	T_7 = \mathbb{E} \Bigg[\sum\limits_{i=1}^{t} \Big\|\frac{\beta_{1,i}}{1 - \beta_{1,i}} \alpha_{i-1}V_{i-1}^{-1/2} m_{i-1}\Big\|^2\Bigg]
	\end{align*}
	We should bound the quantity $T_7$, by the definition of $m_t$, we have
	\begin{align*}
	m_i = \sum\limits_{k=1}^{i}\big[\big(\prod_{l=k+1}^{i} \beta_{1,l}\big)(1 - \beta_{1,k})g_k\big]
	\end{align*}
	Plugging $m_{i-1}$ into $T_7$ yields 
	\begin{align*}
	T_7 =~ & \mathbb{E} \Bigg[\sum\limits_{i=1}^{t} \Big\|\frac{\beta_{1,i}}{1 - \beta_{1,i}} \alpha_{i-1}V_{i-1}^{-1/2} m_{i-1}\Big\|^2\Bigg] \\
	\overset{(i)}{\leq}~ & \Big(\frac{\beta_1}{1 - \beta_1}\Big)^2 \mathbb{E} \Bigg[\sum\limits_{i=2}^{t} \Bigg\|\alpha_{i-1}V_{i-1}^{-1/2} \sum\limits_{k=1}^{i-1}\Big[\Big(\prod_{l=k+1}^{i-1} \beta_{1,l}\Big)(1-\beta_{1,k})g_k\Big]\Bigg\|^2\Bigg] \\
	=~ & \Big(\frac{\beta_1}{1 - \beta_1}\Big)^2 \mathbb{E} \Bigg[\sum\limits_{i=2}^{t} \Bigg\|\sum\limits_{k=1}^{i-1} \alpha_{i-1}V_{i-1}^{-1/2} \Big[\Big(\prod_{l=k+1}^{i-1} \beta_{1,l}\Big)(1-\beta_{1,k})g_k\Big]\Bigg\|^2\Bigg] \\
	\overset{(ii)}{\leq}~ & 2\Big(\frac{\beta_1}{1 - \beta_1}\Big)^2 \underbrace{\mathbb{E} \Bigg[\sum\limits_{i=2}^{t} \Bigg\|\sum\limits_{k=1}^{i-1} \alpha_{k}V_{k}^{-1/2} \Big[\Big(\prod_{l=k+1}^{i-1} \beta_{1,l}\Big)(1-\beta_{1,k})g_k\Big]\Bigg\|^2\Bigg]}_{T_8} \\
	& + 2\Big(\frac{\beta_1}{1 - \beta_1}\Big)^2 \underbrace{\mathbb{E}\Bigg[\sum\limits_{i=2}^{t} \Bigg\|\sum\limits_{k=1}^{i-1} \Big(\alpha_i V_i^{-1/2} - \alpha_k V_k^{-1/2}\Big) \Big[\Big(\prod_{l=k+1}^{i-1} \beta_{1,l}\Big)(1 - \beta_{1,k})g_k\Big] \Bigg\|^2\Bigg]}_{T_9}
	\end{align*}
	(i) is by $\beta_1 \geq \beta_{1,t}$ and (ii) is by 
	We use the fact $(a + b) \leq 2(\|a\|^2 + \|b\|^2)$ in (i). We first bound $T_8$ as below
	\begin{align*}
	T_8 & = \mathbb{E} \Bigg[\sum\limits_{i=2}^{t} \Bigg\|\sum\limits_{k=1}^{i-1} \alpha_{k}V_{k}^{-1/2} \Big[\Big(\prod_{l=k+1}^{i-1} \beta_{1,l}\Big)(1-\beta_{1,k})g_k\Big]\Bigg\|^2\Bigg] \\
	& = \mathbb{E} \Bigg[\sum\limits_{i=2}^{t} \sum\limits_{j=1}^{d} \Bigg(\sum\limits_{k=1}^{i-1} \alpha_{k}V_{k}^{-1/2} \Big[\Big(\prod_{l=k+1}^{i-1} \beta_{1,l}\Big)(1-\beta_{1,k})g_k\Big]\Bigg)_j^2\Bigg] \\
	& = \mathbb{E} \Bigg[\sum\limits_{i=2}^{t} \sum\limits_{j=1}^{d} \Bigg(\sum\limits_{k=1}^{i-1} \sum\limits_{p=1}^{i-1} \Big(\alpha_{k}V_{k}^{-1/2}g_k\Big)_j \Big(\prod_{l=k+1}^{i-1} \beta_{1,l}\Big)(1-\beta_{1,k})\Big(\alpha_p V_p^{-1/2} g_p\Big)_j \Big(\prod_{q=p+1}^{i-1} \beta_{1,q}\Big)(1 - \beta_{1,p})\Bigg)\Bigg] \\
	& \leq \mathbb{E} \Bigg[\sum\limits_{i=2}^{t} \sum\limits_{j=1}^{d} \Bigg(\sum\limits_{k=1}^{i-1} \sum\limits_{p=1}^{i-1} (\beta_1^{i-1-k}) (\beta_1^{i-1-p}) \frac{1}{2}\Bigg\{\Big(\alpha_{k}V_{k}^{-1/2}g_k\Big)_j^2 + \Big(\alpha_p V_p^{-1/2} g_p\Big)_j^2\Bigg\}\Bigg)\Bigg] \\
	& = \mathbb{E} \Bigg[\sum\limits_{i=2}^{t} \sum\limits_{j=1}^{d} \Bigg(\sum\limits_{k=1}^{i-1} (\beta_1^{i-1-k}) \Big(\alpha_{k}V_{k}^{-1/2}g_k\Big)_j^2 \sum\limits_{p=1}^{i-1} (\beta_1^{i-1-p})\Bigg)\Bigg] \\
	& \leq \frac{1}{1 - \beta_1} \mathbb{E}\Bigg[\sum\limits_{i=2}^{t} \sum\limits_{j=1}^{d} \sum\limits_{k=1}^{i-1} (\beta_1^{i-1-k}) \Big(\alpha_k V_k^{-1/2}g_k \Big)_j^2 \Bigg] \\
	& = \frac{1}{1 - \beta_1} \mathbb{E}\Bigg[\sum\limits_{k=1}^{t-1} \sum\limits_{j=1}^{d} \sum\limits_{i=k+1}^{t} (\beta_1^{i-1-k}) \Big(\alpha_k V_k^{-1/2}g_k \Big)_j^2 \Bigg] \\
	& = \Big(\frac{1}{1 - \beta_1}\Big)^2 \mathbb{E}\Bigg[\sum\limits_{k=1}^{t-1} \sum\limits_{j=1}^{d} \Big(\alpha_k V_k^{-1/2} g_k\Big)_j^2\Bigg] = \Big(\frac{1}{1 - \beta_1}\Big)^2 \mathbb{E}\Bigg[\sum\limits_{i=1}^{t-1} \Big\|\alpha_i V_i^{-1/2} g_i\Big\|^2\Bigg]
	\end{align*}
	For the $T_9$ bound, we have
	\begin{align*}
	T_9 & = \mathbb{E}\Bigg[\sum\limits_{i=2}^{t} \Bigg\|\sum\limits_{k=1}^{i-1} \Big[\Big(\prod_{l=k+1}^{i-1} \beta_{1,l}\Big)(1 - \beta_{1,k})\Big] \Big(\alpha_i V_i^{-1/2} - \alpha_k V_k^{-1/2}\Big) g_k \Bigg\|^2\Bigg] \\
	& \leq \mathbb{E}\Bigg[\sum\limits_{i=2}^{t} \Bigg(\sum\limits_{k=1}^{i-1} \Big[\Big(\prod_{l=k+1}^{i-1} \beta_{1,l}\Big)(1 - \beta_{1,k})\Big] \matnormBigg{\alpha_i V_i^{-1/2} - \alpha_k V_k^{-1/2}}{2} \|g_k\|_2\Bigg)^2\Bigg] \\
	& \leq \mathbb{E}\Bigg[\sum\limits_{i=1}^{t-1} \Bigg(\sum\limits_{k=1}^{i} \Big[\Big(\prod_{l=k+1}^{i} \beta_{1,l}\Big)\Big] \matnormBigg{\alpha_i V_i^{-1/2} - \alpha_k V_k^{-1/2}}{2} \|g_k\|_2\Bigg)^2\Bigg] \\
	& \leq G_\infty^2 \mathbb{E}\Bigg[\sum\limits_{i=1}^{t-1} \Bigg(\sum\limits_{k=1}^{i} \beta_1^{i-k} \matnormBigg{\alpha_i V_i^{-1/2} - \alpha_k V_k^{-1/2}}{2} \Bigg)^2\Bigg] \\
	& \leq G_\infty^2 \mathbb{E}\Bigg[\sum\limits_{i=1}^{t-1} \Bigg(\sum\limits_{k=1}^{i} \beta_1^{i-k} \sum\limits_{l=k+1}^{i} \matnormBigg{\alpha_l V_l^{-1/2} - \alpha_{l-1} V_{l-1}^{-1/2}}{2} \Bigg)^2\Bigg] \\
	& \leq G_\infty^2 \Big(\frac{1}{1 - \beta_1}\Big)^2 \Big(\frac{\beta_1}{1 - \beta_1}\Big)^2 \mathbb{E}\Bigg[\sum\limits_{i=2}^{t-1} \matnormBigg{\alpha_i V_i^{-1/2} - \alpha_{i-1} V_{i-1}^{-1/2}}{2}^2\Bigg]
	\end{align*}
	Then, the remaining term is 
	\begin{align*}
	\mathbb{E}\Bigg[\sum\limits_{i=1}^{t} \alpha_i \Bigg\langle \nabla f(x_i), V_i^{-1/2} g_i \Bigg\rangle \Bigg]
	\end{align*}
	To find the upper bound for this term, we reparameterize $g_t = \nabla f(x_t) + \delta_t$ with $\mathbb{E}[\delta_t] = 0$, and we have
	\begin{align*}
	& \mathbb{E}\Bigg[\sum\limits_{i=1}^{t} \alpha_i \Bigg\langle \nabla f(x_i), V_i^{-1/2} g_i \Bigg\rangle \Bigg] \\
	=~ & \mathbb{E} \Bigg[\sum\limits_{i=1}^{t} \alpha_i \Bigg\langle \nabla f(x_i), V_i^{-1/2} (\nabla f(x_i) + \delta_i) \Bigg\rangle\Bigg] \\
	=~ & \mathbb{E} \Bigg[\sum\limits_{i=1}^{t} \alpha_i \Bigg\langle \nabla f(x_i), V_i^{-1/2} \nabla f(x_i) \Bigg\rangle\Bigg] + \Bigg[\sum\limits_{i=1}^{t} \alpha_i \Bigg\langle \nabla f(x_i), V_i^{-1/2} \delta_i \Bigg\rangle\Bigg]
	\end{align*}
	For the second term of last equation,
	\begin{align*}
	& \mathbb{E} \Bigg[\sum\limits_{i=1}^{t} \alpha_i \Bigg\langle \nabla f(x_i), V_i^{-1/2} \delta_i \Bigg\rangle\Bigg] \\
	=~ & \mathbb{E} \Bigg[\sum\limits_{i=2}^{t} \Bigg\langle \nabla f(x_i), \Big(\alpha_i V_i^{-1/2} - \alpha_{i-1} V_{i-1}^{-1/2}\Big) \delta_i \Bigg\rangle \Bigg] + \mathbb{E} \Bigg[\sum\limits_{i=2}^{t} \alpha_{i-1} \Bigg\langle \nabla f(x_i), V_{i-1}^{-1/2} \delta_i \Bigg\rangle \Bigg] + \mathbb{E} \Bigg[\alpha_1 \Bigg\langle \nabla f(x_1), V_1^{-1/2} \delta_1 \Bigg\rangle \Bigg] \\
	=~ & \mathbb{E} \Bigg[\sum\limits_{i=2}^{t} \Bigg\langle \nabla f(x_i), \Big(\alpha_i V_i^{-1/2} - \alpha_{i-1} V_{i-1}^{-1/2}\Big)\delta_i \Bigg\rangle \Bigg] + \mathbb{E}\Bigg[\alpha_1 \nabla f(x_1)^T V_1^{-1/2} \delta_1 \Bigg] \\
	\overset{(i)}{\geq}~ & \mathbb{E} \Bigg[\sum\limits_{i=2}^{t} \Bigg\langle \nabla f(x_i), \Big(\alpha_i V_i^{-1/2} - \alpha_{i-1}V_{i-1}^{-1/2} \Big)\delta_i \Bigg\rangle\Bigg] - 2G_\infty^2 \mathbb{E} \Bigg[\matnormBigg{\alpha_1 V_1^{-1/2}}{2}\Bigg]
	\end{align*}
	The reasoning is as follows:
	\begin{enumerate}[label=(\roman*)]
		\item The conditional expectation $\mathbb{E}\Big[V_{i-1}^{-1/2}\delta_i \Big| x_i, \widehat V_{i-1}\Big] = 0$ since the $\widehat V_{i-1}$ only depends on the noise variables $\xi_1, \cdots, \xi_{i-1}$ and $\delta_i$ depends on $\xi_i$ with $\mathbb{E}[\xi_k] = 0$ for all $k \in \{1, 2, ..., i\}$. Therefore, they are independent. 
	\end{enumerate}
	Further, we have
	\begin{align*}
	\mathbb{E} \Bigg[\sum\limits_{i=2}^{t} \Bigg\langle \nabla f(x_i), \Big(\alpha_i V_i^{-1/2} - \alpha_{i-1} V_{i-1}^{-1/2}\Big)\delta_i \Bigg\rangle\Bigg] \geq & -\mathbb{E} \Bigg[\sum\limits_{i=2}^{t} \Bigg| \Bigg\langle \nabla f(x_i), \Big(\alpha_i V_i^{-1/2} - \alpha_{i-1} V_{i-1}^{-1/2}\Big)\delta_i \Bigg\rangle \Bigg|\Bigg] \\
	\overset{(ii)}{\geq} & -\mathbb{E}\Bigg[\sum\limits_{i=2}^{t} \Big\|\nabla f(x_i)\Big\|_2 \Big\|\Big(\alpha_i V_i^{-1/2} - \alpha_{i-1}^{-1/2}\Big)\delta_i\Big\|_2\Bigg] \\
	\overset{(iii)}{\geq} & -2G_\infty^2 \mathbb{E}\Bigg[\sum\limits_{i=2}^{t} \matnormBigg{\alpha_i V_i^{-1/2} - \alpha_{i-1} V_{i-1}^{-1/2}}{2}\Bigg]
	\end{align*}
	Therefore, we can bound the first term
	\begin{align*}
	& -\mathbb{E}\Bigg[\sum\limits_{i=1}^{t} \alpha_i \Bigg\langle \nabla f(x_i), V_i^{-1/2} g_i\Bigg\rangle \Bigg] \\
	\leq~ & 2G_\infty^2 \mathbb{E} \Bigg[\sum\limits_{i=2}^{t} \matnormBigg{\alpha_i V_i^{-1/2} - \alpha_{i-1} V_{i-1}^{-1/2}}{2}\Bigg] + 2G_\infty^2 \mathbb{E} \Bigg[\matnormBigg{\alpha_1 V_1^{-1/2}}{2}\Bigg] - \mathbb{E}\Bigg[\sum\limits_{i=1}^{t} \alpha_i \Bigg\langle \nabla f(x_i), V_i^{-1/2} \nabla f(x_i)\Bigg\rangle \Bigg]
	\end{align*}
\end{proof}

\begin{lemma}{\label{lemma8}} (Lemma 6.8 in \cite{chen2019})
	For $a_i \leq 0$, $\beta \in [0, 1)$, and $b_i = \sum_{k=1}^{i} \beta^{i-k} \sum_{l=k+1}^{i} a_l$, we have
	\begin{align*}
	\sum\limits_{i=1}^{t} b_i^2 \leq \Big(\frac{1}{1 - \beta}\Big)^2 \Big(\frac{\beta}{1 - \beta}\Big)^2 \sum\limits_{i=2}^{t} a_i^2
	\end{align*}
\end{lemma}

\subsection{Proof of Theorem \ref{thm1}}
\begin{proof}
	We combine the above lemmas to bound 
	\begin{align*}
	\mathbb{E} [f(z_{t+1}) - f(z_1)] & \leq \sum\limits_{i=1}^{6} T_i \\
	& \leq \underbrace{G_\infty^2 \frac{\beta_1}{1 - \beta_1} \mathbb{E}\Bigg[\sum\limits_{i=2}^{t} \matnormBigg{\alpha_i V_i^{-1/2} - \alpha_{i-1}V_{i-1}^{-1/2}}{2} \Bigg]}_{T_1} \\
	&~~~ + \underbrace{\Bigg(\frac{\beta_1}{1 - \beta_1} - \frac{\beta_{1, t+1}}{1 - \beta_{1, t+1}} \Bigg) (G_\infty^2 + D_\infty^2)}_{T_3} \\
	&~~~ + \underbrace{\Big(\frac{\beta_1}{1 - \beta_1} - \frac{\beta_{1, t+1}}{1 - \beta_{1, t+1}}\Big)^2 D_\infty^2}_{T_4} \\
	&~~~ + \underbrace{\Big(\frac{\beta_1}{1 - \beta_1}\Big)^2 G_\infty^2 \mathbb{E}\Bigg[\sum\limits_{i=2}^{t} \matnormBigg{\alpha_i V_i^{-1/2} - \alpha_{i-1} V_{i-1}^{-1/2}}{2}^2\Bigg]}_{T_5} \\
	&~~~ + \underbrace{\mathbb{E}\Bigg[\sum\limits_{i=1}^{t} \frac{3}{2}L \Bigg\Vert \alpha_i V_i^{-1/2} g_i \Bigg\Vert^2\Bigg]}_{T_6} \\
	&~~~ + \underbrace{2G_\infty^2 \mathbb{E} \Bigg[\sum\limits_{i=2}^{t} \matnormBigg{\alpha_i V_i^{-1/2} - \alpha_{i-1} V_{i-1}^{-1/2}}{2}\Bigg] + 2G_\infty^2 \mathbb{E} \Bigg[\matnormBigg{\alpha_1 V_1^{-1/2}}{2}\Bigg]}_{T_2} \\
	&~~~ \underbrace{- \mathbb{E}\Bigg[\sum\limits_{i=1}^{t} \alpha_i \Bigg\langle \nabla f(x_i), V_i^{-1/2} \nabla f(x_i)\Bigg\rangle \Bigg]}_{T_2} \\
	&~~~ + \underbrace{L^2 \Big(\frac{\beta_1}{1 - \beta_1}\Big)^2 \Bigg(\Big(\frac{1}{1 - \beta_1}\Big)^2 \mathbb{E}\Bigg[\sum\limits_{i=1}^{t-1} \Big\|\alpha_i V_i^{-1/2} g_i\Big\|^2\Bigg]}_{T_2} \\
	&~~~ \underbrace{+ G_\infty^2 \Big(\frac{1}{1 - \beta_1}\Big)^2 \Big(\frac{\beta_1}{1 - \beta_1}\Big)^2 \mathbb{E}\Bigg[\sum\limits_{i=2}^{t-1} \matnormBigg{\alpha_i V_i^{-1/2} - \alpha_{i-1} V_{i-1}^{-1/2}}{2}^2\Bigg]\Bigg)}_{T_2} \\
	&~~~ \underbrace{+ \mathbb{E}\Bigg[\frac{1}{2} \sum\limits_{i=1}^{t} \|\alpha_i V_i^{-1/2} g_i\|^2 \Bigg]}_{T_2}
	\end{align*}
	By merging similar terms, we can have
	\begin{align*}
	\mathbb{E}[f(z_{t+1}) - f(z_1)] & \leq \Bigg(G_\infty^2 \frac{\beta_1}{1 - \beta_1} + 2G_\infty^2\Bigg) \mathbb{E}\Bigg[\sum\limits_{i=2}^{t} \matnormBigg{\alpha_i \widehat{V}_{i}^{-1/2} - \alpha_{i-1} \widehat{V}_{i-1}^{-1/2}}{2}\Bigg] \\
	& + \Bigg(\frac{3}{2}L + \frac{1}{2} + L^2 \big(\frac{\beta_1}{1 - \beta_1}\big)^2\big(\frac{1}{1 - \beta_1}\big)^2\Bigg)\mathbb{E}\Bigg[\sum\limits_{i=1}^{t} \Big\| \alpha_i \widehat{V}_i^{-1/2} g_i \Big\|^2\Bigg] \\ 
	& + \Bigg(1 + L^2 \big(\frac{1}{1 - \beta_1}\big)^2 \big(\frac{\beta_1}{1 - \beta_1}\big)^2\Bigg) \big(\frac{\beta_1}{1 - \beta_1}\big)^2 G_\infty^2 \mathbb{E}\Bigg[\sum\limits_{i=2}^{t-1} \matnormBigg{\alpha_i \widehat{V}_i^{-1/2} - \alpha_{i-1} \widehat{V}_{i-1}^{-1/2}}{2}^2\Bigg] \\
	& + \Bigg(\frac{\beta_1}{1 - \beta_1} - \frac{\beta_{1, t+1}}{1 - \beta_{1, t+1}} \Bigg) (G_\infty^2 + D_\infty^2) + \Big(\frac{\beta_1}{1 - \beta_1} - \frac{\beta_{1, t+1}}{1 - \beta_{1, t+1}}\Big)^2 D_\infty^2 + 2G_\infty^2 \mathbb{E} \Bigg[\matnormBigg{\alpha_1 V_1^{-1/2}}{2}\Bigg] \\
	& - \mathbb{E}\Bigg[\sum\limits_{i=1}^{t} \alpha_i \Bigg\langle \nabla f(x_i), V_i^{-1/2} \nabla f(x_i)\Bigg\rangle \Bigg]
	\end{align*}
	We define constants $C_1, C_2$, and $C_3$ as
	\begin{align*}
	C_1 & = \frac{3}{2}L + \frac{1}{2} + L^2 \big(\frac{\beta_1}{1 - \beta_1}\big)^2\big(\frac{1}{1 - \beta_1}\big)^2 \\
	C_2 & = G_\infty^2 \frac{\beta_1}{1 - \beta_1} + 2G_\infty^2 \\
	C_3 & = \Bigg(1 + L^2 \big(\frac{1}{1 - \beta_1}\big)^2 \big(\frac{\beta_1}{1 - \beta_1}\big)^2\Bigg) \big(\frac{\beta_1}{1 - \beta_1}\big)^2 G_\infty^2
	\end{align*}
	By rearranging terms, we obtain
	\begin{align*}
	\mathbb{E}\Bigg[\sum\limits_{i=1}^{t} \alpha_i \Bigg\langle \nabla f(x_i), V_i^{-1/2} \nabla f(x_i)\Bigg\rangle \Bigg] & \leq \mathbb{E}\Bigg[\sum\limits_{i=1}^{t} C_1 \Big\| \alpha_i \widehat{V}_i^{-1/2} g_i \Big\|^2 + C_2 \sum\limits_{i=2}^{t} \matnormBigg{\alpha_i \widehat{V}_{i}^{-1/2} - \alpha_{i-1} \widehat{V}_{i-1}^{-1/2}}{2} \\
	& ~~~+ C_3 \sum\limits_{i=2}^{t-1} \matnormBigg{\alpha_i \widehat{V}_i^{-1/2} - \alpha_{i-1} \widehat{V}_{i-1}^{-1/2}}{2}^2 \Bigg] \\
	& ~~~+ \Bigg(\frac{\beta_1}{1 - \beta_1} - \frac{\beta_{1, t+1}}{1 - \beta_{1, t+1}} \Bigg) (G_\infty^2 + D_\infty^2) + \Big(\frac{\beta_1}{1 - \beta_1} - \frac{\beta_{1, t+1}}{1 - \beta_{1, t+1}}\Big)^2 D_\infty^2 \\
	& ~~~+ 2G_\infty^2 \mathbb{E} \Bigg[\matnormBigg{\alpha_1 V_1^{-1/2}}{2}\Bigg] \\
	& \leq \mathbb{E}\Bigg[\sum\limits_{i=1}^{t} C_1 \Big\| \alpha_i \widehat{V}_i^{-1/2} g_i \Big\|^2 + C_2 \sum\limits_{i=2}^{t} \matnormBigg{\alpha_i \widehat{V}_{i}^{-1/2} - \alpha_{i-1} \widehat{V}_{i-1}^{-1/2}}{2} \\
	& ~~~+ C_3 \sum\limits_{i=2}^{t-1} \matnormBigg{\alpha_i \widehat{V}_i^{-1/2} - \alpha_{i-1} \widehat{V}_{i-1}^{-1/2}}{2}^2 \Bigg] \\
	& ~~~+ \Bigg(\frac{\beta_1}{1 - \beta_1}\Bigg) (G_\infty^2 + D_\infty^2) + \Big(\frac{\beta_1}{1 - \beta_1}\Big)^2 D_\infty^2 + 2G_\infty^2 \mathbb{E} \Bigg[\matnormBigg{\alpha_1 V_1^{-1/2}}{2}\Bigg]
	\end{align*}
	Finally, we can get
	\begin{align*}
	& \mathbb{E} \Bigg[\sum\limits_{t=1}^{T} \alpha_i \Bigg\langle \nabla f(x_i), \widehat V_i^{-1/2} \nabla f(x_i) \Bigg\rangle \Bigg] \\
	\leq~ & \mathbb{E}\Bigg[C_1 \underbrace{\sum\limits_{t=1}^{T} \Bigg\|\alpha_i \V_i^{-1/2} g_i \Bigg\|^2}_{\text{Term A}} +~ C_2 \underbrace{\sum\limits_{t=2}^{T} \matnormBigg{\alpha_i \V_i^{-1/2} - \alpha_{i-1} \V_{i-1}^{-1/2}}{2}}_{\text{Term B}} +~ C_3 \sum\limits_{t=2}^{T-1} \matnormBigg{\alpha_i \V_i^{-1/2} - \alpha_{i-1} \V_{i-1}^{-1/2}}{2}^2\Bigg] + C_4
	\end{align*}
	with constants
	\begin{align*}
	C_4 & = \Bigg(\frac{\beta_1}{1 - \beta_1}\Bigg) (G_\infty^2 + D_\infty^2) + \Big(\frac{\beta_1}{1 - \beta_1}\Big)^2 D_\infty^2 + 2G_\infty^2 \mathbb{E} \Bigg[\matnormBigg{\alpha_1 V_1^{-1/2}}{2}\Bigg]
	\end{align*}
	with almost same constant for the diagonal version.
\end{proof}

\subsection{Proofs of Corollary \ref{cor1}}
From theorem \ref{thm1}, we first bound the RHS. Since $\widehat{V}_t = (1/t) \sum_{\tau=1}^{t} g_t g_t^T$ and $\alpha_t = 1/\sqrt{t}$, the Term A in the theorem \ref{thm1} is
\begin{align*}
\mathbb{E}\Bigg[\sum\limits_{t=1}^{T} \Bigg\|\alpha_t \widehat{V}_t^{-1/2}g_t\Bigg\|^2\Bigg] & = \mathbb{E}\Bigg[\sum\limits_{t=1}^{T} \Bigg\|\Big(\sum\limits_{\tau=1}^{t} g_\tau g_\tau^T \Big)^{-1/2}g_t\Bigg\|^2\Bigg] \\
& \leq \mathbb{E}\Bigg[\sum\limits_{t=1}^{T} \matnormBigg{\Big(\sum\limits_{\tau=1}^{t} g_\tau g_\tau^T\Big)^{-1/2}}{2}^2 \|g_t\|_2^2 \Bigg] \\
& = \mathbb{E}\Bigg[\sum\limits_{t=1}^{T} \frac{1}{\lambda_{min}\Big(\sum\limits_{\tau=1}^{t} g_\tau g_\tau^T\Big)} \|g_t\|_2^2 \Bigg]
\end{align*}
By Weyl's theorem on eigenvalues, we can obtain $\lambda_{min}(A + B) \geq \lambda_{min}(A) + \lambda_{min}(B)$ for any two Hermitian matrices. Therefore,
\begin{align*}
\mathbb{E}\Bigg[\sum\limits_{t=1}^{T} \frac{1}{\lambda_{min}\Big(\sum\limits_{\tau=1}^{t} g_\tau g_\tau^T\Big)} \|g_t\|_2^2 \Bigg] & \leq \mathbb{E}\Bigg[\sum\limits_{t=1}^{T} \frac{1}{\sum\limits_{\tau=1}^{t} \lambda_{min}\Big(g_\tau g_\tau^T\Big)} \|g_t\|_2^2 \Bigg] \\
& = \mathbb{E}\Bigg[\sum\limits_{t=1}^{T} \frac{1}{\sum\limits_{\tau=1}^{t} \|g_\tau\|_2^2} \|g_t\|_2^2 \Bigg] \\
& \leq \mathbb{E}\Bigg[1 - \log(\|g_1\|_2^2) + \log\sum\limits_{t=1}^{T} \|g_\tau\|_2^2\Bigg] \\
& \leq 1 - 2\log\|g_1\| + 2\log G_\infty + \log T
\end{align*}
For the Term B, we can bound
\begin{align*}
\mathbb{E}\Bigg[\sum\limits_{t=2}^{T} \matnormBigg{\alpha_t \widehat{V}_t^{-1/2} - \alpha_{t-1} \widehat{V}_{t-1}^{-1/2}}{2}\Bigg] & = \mathbb{E}\Bigg[\sum\limits_{t=2}^{T} \matnormBigg{\Big(\sum\limits_{\tau=1}^{t} g_\tau g_\tau^T\Big)^{-1/2} - \Big(\sum\limits_{\tau=1}^{t-1} g_\tau g_\tau^T\Big)^{-1/2}}{2}\Bigg] \\
& \leq \mathbb{E}\Bigg[\sum\limits_{t=2}^{T} \mathrm{tr}\Bigg(\Big(\sum\limits_{\tau=1}^{t-1} g_\tau g_\tau^T\Big)^{-1/2} - \Big(\sum\limits_{\tau=1}^{t} g_\tau g_\tau^T\Big)^{-1/2}\Bigg)\Bigg] \\
& = \mathbb{E}\Bigg[\frac{1}{\|g_1\|_2} - \mathrm{tr}\Bigg(\Big(\sum\limits_{\tau=1}^{T} g_\tau g_\tau^T\Big)^{-1/2}\Bigg)\Bigg] \\
& \leq \frac{1}{\|g_1\|}
\end{align*}
The last term involving the constant $C_3$ can be bound similarly
\begin{align*}
\mathbb{E}\Bigg[\sum\limits_{t=2}^{T} \matnormBigg{\alpha_t \widehat{V}_t^{-1/2} - \alpha_{t-1} \widehat{V}_{t-1}^{-1/2}}{2}^2\Bigg] & = \mathbb{E}\Bigg[\sum\limits_{t=2}^{T} \matnormBigg{\Big(\sum\limits_{\tau=1}^{t} g_\tau g_\tau^T\Big)^{-1/2} - \Big(\sum\limits_{\tau=1}^{t-1} g_\tau g_\tau^T\Big)^{-1/2}}{2}^2\Bigg] \\
& = \mathbb{E}\Bigg[\sum\limits_{t=2}^{T-1} \matnormBigg{\Bigg(\Big(\sum\limits_{\tau=1}^{t-1} g_\tau g_\tau^T\Big)^{-1/2} - \Big(\sum\limits_{\tau=1}^{t} g_\tau g_\tau^T\Big)^{-1/2}\Bigg)^2}{2}\Bigg] \\
& \leq \mathbb{E}\Bigg[\sum\limits_{t=2}^{T-1} \mathrm{tr}\Bigg(\Big(\sum\limits_{\tau=1}^{t-1} g_\tau g_\tau^T\Big)^{-1} - \Big(\sum\limits_{\tau=1}^{t} g_\tau g_\tau^T\Big)^{-1}\Bigg)\Bigg] \\
& \leq \mathbb{E}\Bigg[\frac{1}{\|g_1\|_2^2} - \mathrm{tr}\Bigg(\Big(\sum\limits_{\tau=1}^{T-1} g_\tau g_\tau^T\Big)^{-1}\Bigg)\Bigg] \\
& \leq \frac{1}{\|g_1\|^2}
\end{align*}
To bound the LHS term in the theorem \ref{thm1},
\begin{align*}
\mathbb{E}\Bigg[\sum\limits_{t=1}^{T} \alpha_t \Bigg\langle \nabla f(x_t), \widehat{V}_{t}^{-1/2} \nabla f(x_t)\Bigg\rangle\Bigg] & = \mathbb{E}\Bigg[\sum\limits_{t=1}^{T} \Bigg\langle \nabla f(x_t), \Big(\sum\limits_{\tau=1}^{t} g_\tau g_\tau^T \Big)^{-1/2} \nabla f(x_t)\Bigg\rangle \Bigg] \\
& \geq \mathbb{E}\Bigg[\sum\limits_{t=1}^{T} \lambda_{min}\Big(\sum\limits_{\tau=1}^{t} g_\tau g_\tau^T \Big)^{-1/2} \|\nabla f(x_t)\|^2 \Bigg] \\ 
& = \mathbb{E}\Bigg[\sum\limits_{t=1}^{T} \frac{1}{\lambda_{max}\Big(\sum\limits_{\tau=1}^{t} g_\tau g_\tau^T \Big)^{1/2}} \|\nabla f(x_t)\|^2 \Bigg]
\end{align*}
Again, we use Weyl's theorem to bound the maximum eigenvalues as follows
\begin{align*}
\lambda_{max}\Big(\sum\limits_{\tau=1}^{t} g_\tau g_\tau^T \Big) & \leq \sum\limits_{\tau=1}^{t} \lambda_{max}\Big(g_\tau g_\tau^T\Big) \\
& \leq \sum\limits_{\tau=1}^{T} \lambda_{max}\Big(g_\tau g_\tau^T\Big) \\
& = \sum\limits_{\tau=1}^{T} \|g_\tau\|^2 \\
& \leq T G_\infty^2
\end{align*}
Therefore, the LHS term can be bound as
\begin{align*}
\mathbb{E}\Bigg[\sum\limits_{t=1}^{T} \alpha_t \Bigg\langle \nabla f(x_t), \widehat{V}_{t}^{-1/2} \nabla f(x_t)\Bigg\rangle\Bigg] & \geq \mathbb{E}\Bigg[\frac{1}{\sqrt{T} G_\infty} \sum\limits_{t=1}^{T} \|\nabla f(x_t)\|^2\Bigg] \\
& = \frac{\sqrt{T}}{G_\infty}\min\limits_{t \in [T]} \mathbb{E}\big[\|\nabla f(x_t)\|^2\big]
\end{align*}
Combining all the terms yields
\begin{align*}
\min\limits_{t \in [T]} \mathbb{E}\big[\|\nabla f(x_t)\|^2\big] & \leq \frac{G_\infty}{\sqrt{T}} \Big(C_1(1 - 2\log\|g_1\| + 2\log G_\infty + \log T) + C_2(\frac{1}{\|g_1\|}) + C_3(\frac{1}{\|g_1\|^2}) + C_4\Big)\\
& = \mathcal{O}\Big(\frac{\log T}{\sqrt{T}}\Big)
\end{align*}

%%%%% proofs of the theorem 2 %%%%%
\subsection{Technical Lemmas for Theorem \ref{thm2}}

\begin{lemma}{\label{lemma9}}
	Let $A \succeq B \succeq 0$ be symmetric $d \times d$ PSD matrices. Then, $A^{1/2} \succeq B^{1/2}$. 
\end{lemma}

\begin{lemma}{\label{lemma10}}
	Let $A$ and $B$ be positive semidefinite matrices. Then, the eigenvalues for the product $AB$ is real and non-negative.
\end{lemma}

\begin{proof}
	Since $A$ is PSD, we can compute the square root of the matrix $A$, which we call $A^{1/2}$. Consider the matrix $C = A^{1/2} B A^{1/2}$ which is positive semidefinite. Then, the eigenvalues of $C$ is equal to the eigenvalues of $AB = A^{1/2}(A^{1/2}B)$. Therefore, all the eigenvalues of $AB$ is non-negative. 
\end{proof}

\begin{lemma}{\label{lemma11}}
	For a PSD matrix $A$, $\lambda_{\mathrm{min}}(A^{1/2}) = \lambda_{\mathrm{min}}(A)^{1/2}$.
\end{lemma}

\begin{lemma}{\label{lemma12}} (From Auxiliary Lemma 1 in~\cite{zaheer2018})
	For the iterates $x_t$ where $t \in [T]$ in Algorithm \ref{alg:adaptive_full}, the following inequality holds:
	\begin{align*}
	\mathbb{E}_t\big[ \|g_t\|_2^2 \big] = \mathbb{E}_t\big[ \sum\limits_{i=1}^{d} g_{t,i}^2 \big] \leq \sum\limits_{i=1}^{d} \Big[\frac{\sigma_i^2}{M} + [\nabla f(x_t)]_i^2\Big] = \frac{\sigma^2}{M} + \|\nabla f(x_t)\|_2^2
	\end{align*}
\end{lemma}

\begin{lemma}{\label{lemma13}}
	For $g_t$ and $m_t$, we have
	\begin{align*}
	\mathbb{E}_t\big[ \|g_t\|_2 \|m_t\|_2 \big]  \leq \beta_{1,t} G_\infty^2 + (1 - \beta_{1,t}) \Big(\frac{\sigma^2}{M} + \| \nabla f(x_t) \|_2^2\Big) \\
	\mathbb{E}_t\big[\|m_t\|_2^2\big] \leq \beta_{1,t}^2 G_\infty^2 + (1 - \beta_{1,t})^2 \Big(\frac{\sigma^2}{M} + \| \nabla f(x_t) \|_2^2\Big)
	\end{align*}
\end{lemma}

\begin{proof}
	\begin{align*}
	\mathbb{E}_t\big[ \|g_t\|_2 \|m_t\|_2 \big] & = \mathbb{E}_t\big[ \|g_t\|_2 \|\beta_{1,t} m_{t-1} + (1 - \beta_{1,t}) g_t\|_2 \big] \\
	& \leq \mathbb{E}_t\big[ \beta_{1,t} \|g_t\|_2 \| m_{t-1} \|_2 \big] + \mathbb{E}_t\big[ (1 - \beta_{1,t}) \|g_t\|_2^2 \big] \\
	& \leq \beta_{1,t} G_\infty^2 + (1 - \beta_{1,t}) \mathbb{E}_t\big[\|g_t\|_2^2\big] \\
	& \leq \beta_{1,t} G_\infty^2 + (1 - \beta_{1,t}) \Big(\frac{\sigma^2}{M} + \|\nabla f(x_t)\|_2^2 \Big)
	\end{align*}
	and
	\begin{align*}
	\mathbb{E}_t\big[\|m_t\|_2^2\big] & = \mathbb{E}_t\big[ \| \beta_{1,t} m_{t-1} + (1 - \beta_{1,t}) g_t \|_2^2 \big] \\
	& \leq \mathbb{E}_t\big[ \| \beta_{1,t} m_{t-1} \|_2^2\big] + \mathbb{E}_t\big[ \|(1 - \beta_{1,t}) g_t \|_2^2 \big] \\
	& \leq \beta_{1,t}^2 G_\infty^2 + (1 - \beta_{1,t})^2 \mathbb{E}_t\big[ \| g_t \|_2^2 \big] \\
	& \leq \beta_{1,t}^2 G_\infty^2 + (1 - \beta_{1,t})^2 \Big(\frac{\sigma^2}{M} + \|\nabla f(x_t)\|_2^2 \Big)
	\end{align*}
\end{proof}

\begin{lemma}{\label{lemma14}}
	The term $\nabla f(x_t)^T \Big(\beta_2^{1/2} \widehat{V}_{t-1}^{1/2} + \delta I\Big)^{-1} m_t$ can be bound as
	\begin{align*}
	\nabla f(x_t)^T \Big(\beta_2^{1/2} \widehat{V}_{t-1}^{1/2} + \delta I\Big)^{-1} m_t \geq (1 - \beta_{1,t}) \nabla f(x_t)^T \Big(\beta_2^{1/2} \widehat{V}_{t-1}^{1/2} + \delta I\Big)^{-1} g_t - \frac{\beta_{1,t} G_\infty^2}{\delta}
	\end{align*}
\end{lemma}

\begin{proof}
	\begin{align*}
	\nabla f(x_t)^T \Big(\beta_2^{1/2} \widehat{V}_{t-1}^{1/2} + \delta I\Big)^{-1} m_t =~ & \nabla f(x_t)^T \Big(\beta_2^{1/2} \widehat{V}_{t-1}^{1/2} + \delta I\Big)^{-1} \big(\beta_{1,t} m_{t-1} + (1 - \beta_{1,t}) g_t\big) \\
	=~ & \beta_{1,t} \nabla f(x_t)^T \Big(\beta_2^{1/2} \widehat{V}_{t-1}^{1/2} + \delta I\Big)^{-1} m_{t-1} + (1 - \beta_{1,t}) \nabla f(x_t)^T \Big(\beta_2^{1/2} \widehat{V}_{t-1}^{1/2} + \delta I\Big)^{-1} g_t \\
	\overset{(i)}{\geq}~ & (1 - \beta_{1,t}) \nabla f(x_t)^T \Big(\beta_2^{1/2} \widehat{V}_{t-1}^{1/2} + \delta I\Big)^{-1} g_t - \beta_{1,t} \Big| \nabla f(x_t)^T \Big(\beta_2^{1/2} \widehat{V}_{t-1}^{1/2} + \delta I\Big)^{-1} m_{t-1} \Big| \\
	\overset{(ii)}{\geq}~ & (1 - \beta_{1,t}) \nabla f(x_t)^T \Big(\beta_2^{1/2} \widehat{V}_{t-1}^{1/2} + \delta I\Big)^{-1} g_t - \beta_{1,t} G_\infty^2 \matnormBigg{\Big(\beta_2^{1/2} \widehat{V}_{t-1}^{1/2} + \delta I\Big)^{-1}}{2} \\
	\overset{(iii)}{\geq}~ & (1 - \beta_{1,t}) \nabla f(x_t)^T \Big(\beta_2^{1/2} \widehat{V}_{t-1}^{1/2} + \delta I\Big)^{-1} g_t - \frac{\beta_{1,t} G_\infty^2}{\delta}
	\end{align*}
	The reasoning follows
	\begin{enumerate}[label=(\roman*)]
		\item For a scalar $a$, the relation $-|a| \leq a \leq |a|$.
		\item By Cauchy-Schwarz inequality and matrix norm inequality, $\Big| \nabla f(x_t)^T \Big(\beta_2^{1/2} \widehat{V}_{t-1}^{1/2} + \delta I\Big)^{-1} m_{t-1} \Big| \leq \|\nabla f(x_t)\| \Big\|\Big(\beta_2^{1/2} \widehat{V}_{t-1}^{1/2} + \delta I\Big)^{-1} m_{t-1}\Big\| \leq \|\nabla f(x_t)\| \|m_{t-1}\| \matnormBig{\Big(\beta_2^{1/2} \widehat{V}_{t-1}^{1/2} + \delta I\Big)^{-1}}{2} \leq G_\infty^2 \matnormBig{\Big(\beta_2^{1/2} \widehat{V}_{t-1}^{1/2} + \delta I\Big)^{-1}}{2}$ holds by our bounded gradient assumptions.
		\item We use the relation $\Big(\beta_2^{1/2} \widehat{V}_{t-1}^{1/2} + \delta I\Big)^{-1} \preceq (\delta I)^{-1}$.
	\end{enumerate}
\end{proof}

\subsection{Proofs of Theorem \ref{thm2}}
The update rule for \textsc{Adam} with full matrix adaptations is
\begin{align*}	
x_{t+1} = x_t - \alpha_t (\widehat V_t^{1/2} + \delta I)^{-1}m_t 
\end{align*}
We assume that $\widehat V_t$ is full-rank after $t \geq t_0$ steps. For notational convenience, we let $\widetilde m_t = ( \widehat{V}_t^{1/2} + \delta I)^{-1} m_t$.
Since $f$ is $L$-smooth, we have the following:
\begin{align*}
f(x_{t+1}) & \leq f(x_t) + \langle \nabla f(x_t), x_{t+1} - x_t \rangle + \frac{L}{2} \|x_{t+1} - x_t\|^2 \\
& = f(x_t) - \langle \nabla f(x_t), \alpha_t (\widehat V_t^{1/2} + \delta I)^{-1} m_t \rangle + \frac{L}{2} \| \alpha_t (\widehat V_t^{1/2} + \delta I)^{-1} m_t \|^2 \\
& = f(x_t) - \alpha_t \sum\limits_{i=1}^{d} \Big(\nabla [f(x_t)]_i \times \widetilde m_{t,i} \Big) + \frac{\alpha_t^2 L}{2} \sum\limits_{i=1}^{d} \widetilde m_{t,i}^2 
\end{align*}
We take the expectation of $f(x_{t+1})$ in the above inequality,
\begin{align*}
\mathbb{E}_t[f(x_{t+1})] \leq~ & f(x_t) - \alpha_t \sum\limits_{i=1}^{d} \Big( [\nabla f(x_t)]_i \times \mathbb{E}_t \Big[\big[(\widehat V_t^{1/2} + \delta I)^{-1} m_t\big]_i\Big]	\Big) + \frac{\alpha_t^2 L}{2} \sum\limits_{i=1}^{d} \mathbb{E}\big[\widetilde m_{t,i}^2\big] \\
=~ & f(x_t) - \alpha_t \sum\limits_{i=1}^{d} \Big( [\nabla f(x_t)]_i \times \mathbb{E}_t \Big[\big[(\widehat V_t^{1/2} + \delta I)^{-1} m_t\big]_i - \big[(\beta_2^{1/2} \widehat V_{t-1}^{1/2} + \delta I)^{-1} m_t\big]_i + \big[(\beta_2^{1/2} \widehat V_{t-1}^{1/2} + \delta I)^{-1} m_t\big]_i \Big]	\Big) \\
& ~~~~+ \frac{\alpha_t^2 L}{2} \sum\limits_{i=1}^{d} \mathbb{E}\big[\widetilde m_{t,i}^2\big] \\
=~ & f(x_t) - \alpha_t \mathbb{E}_t\Big[\nabla f(x_t)^T  (\beta_2^{1/2} \widehat V_{t-1}^{1/2} + \delta I)^{-1} m_t\Big] \\
& ~~~~- \alpha_t \sum\limits_{i=1}^{d} \Big([\nabla f(x_t)]_i \times \mathbb{E}_t \Big[\big[(\widehat V_t^{1/2} + \delta I)^{-1} m_t\big]_i - \big[(\beta_2^{1/2} \widehat V_{t-1}^{1/2} + \delta I)^{-1} m_t\big]_i \Big]\Big) + \frac{\alpha_t^2 L}{2} \sum\limits_{i=1}^{d} \mathbb{E}\big[\widetilde m_{t,i}^2\big] \\
\overset{(i)}{\leq}~ & f(x_t) - \alpha_t \mathbb{E}_t\Big[(1 - \beta_{1,t}) \nabla f(x_t)^T \Big(\beta_2^{1/2} \widehat{V}_{t-1}^{1/2} + \delta I\Big)^{-1} g_t \Big] + \frac{\alpha_t \beta_{1,t} G^2}{\delta} \\
& ~~~~- \alpha_t \sum\limits_{i=1}^{d} \Big([\nabla f(x_t)]_i \times \mathbb{E}_t \Big[\big[(\widehat V_t^{1/2} + \delta I)^{-1} m_t\big]_i - \big[(\beta_2^{1/2} \widehat V_{t-1}^{1/2} + \delta I)^{-1} m_t\big]_i \Big]\Big) + \frac{\alpha_t^2 L}{2} \sum\limits_{i=1}^{d} \mathbb{E}\big[\widetilde m_{t,i}^2\big] \\
=~ & f(x_t) - \alpha_t (1 - \beta_{1,t}) \|\nabla f(x_t)\|_{(\beta_2^{1/2} \widehat V_{t-1}^{1/2} + \delta I)^{-1}}^2 - \alpha_t \nabla f(x_t)^T \mathbb{E}_t \Big[\big[(\widehat V_t^{1/2} + \delta I)^{-1} - (\beta_2^{1/2} \widehat V_{t-1}^{1/2} + \delta I)^{-1}\big]m_t \Big] \\
& ~~~~+ \frac{\alpha_t^2 L}{2} \sum\limits_{i=1}^{d} \mathbb{E}\big[\widetilde m_{t,i}^2\big] + \frac{\alpha_t \beta_{1,t} G^2}{\delta}\\
\overset{(ii)}{\leq}~ & f(x_t) - \alpha_t (1 - \beta_{1,t}) \|\nabla f(x_t)\|_{(\beta_2^{1/2} \widehat V_{t-1}^{1/2} + \delta I)^{-1}}^2 + \alpha_t \|\nabla f(x_t)\|_2 \mathbb{E}_t \underbrace{\Big\|\big[(\widehat V_t^{1/2} + \delta I)^{-1} - (\beta_2^{1/2} \widehat V_{t-1}^{1/2} + \delta I)^{-1}\big]m_t \Big\|_2}_{T_1} \\
& ~~~~+ \frac{\alpha_t^2 L}{2} \sum\limits_{i=1}^{d} \mathbb{E}\big[\widetilde m_{t,i}^2\big] + \frac{\alpha_t \beta_{1,t} G^2}{\delta}
\end{align*}
The reasoning follows
\begin{enumerate}[label=(\roman*)]
	\item We use the lemma \ref{lemma14}.
	\item For any scalar $a$, we have $-|a| \leq a \leq |a|$.
\end{enumerate}
Bounding the term $T_1$ using a matrix norm inequality,
\begin{align*}
T_1 & = \Big\|\big[(\widehat V_t^{1/2} + \delta I)^{-1} - (\beta_2^{1/2} \widehat V_{t-1}^{1/2} + \delta I)^{-1}\big]m_t \Big\|_2 \\
& \leq \underbrace{\matnormBigg{\big(\widehat V_t^{1/2} + \delta I\big)^{-1} - \big(\beta_2^{1/2} \widehat V_{t-1}^{1/2} + \delta I\big)^{-1}}{2}}_{T_2} \|m_t\|_2
\end{align*}
By definition of $\widehat V_t = \beta_2 \widehat{V}_{t-1} + (1 - \beta_2) g_t g_t^T$, we have $\widehat V_t \succeq \beta_2 \widehat{V}_{t-1}$. Therefore, by the lemma \ref{lemma1}, we have $\widehat V_t^{1/2} \succeq \beta_2^{1/2} \widehat{V}_{t-1}^{1/2}$, and moreover we can obtain $\widehat V_t^{1/2} + \delta I \succeq \beta_2^{1/2} \widehat{V}_{t-1}^{1/2} + \delta I$. Finally, we arrive at
\begin{align*}
\big(\beta_2^{1/2} \widehat{V}_{t-1}^{1/2} + \delta I\big)^{-1} \succeq \big(\widehat{V}_t^{1/2} + \delta I\big)^{-1} 
\end{align*}
Therefore, we can bound $T_2$ as
\begin{align*}
T_2 & = \matnormBigg{\big(\widehat V_t^{1/2} + \delta I\big)^{-1} - \big(\beta_2^{1/2} \widehat V_{t-1}^{1/2} + \delta I\big)^{-1}}{2} \\
=~ & \matnormBigg{\big(\widehat V_t^{1/2} + \delta I\big)^{-1} \Bigg(\big(\beta_2^{1/2} \widehat V_{t-1}^{1/2} + \delta I\big) - \big(\widehat V_{t-1}^{1/2} + \delta I\big)\Bigg) \big(\beta_2^{1/2} \widehat V_{t-1}^{1/2} + \delta I\big)^{-1}}{2} \\
\overset{(i)}{\leq}~ & \matnormBigg{\big(\widehat V_t^{1/2} + \delta I\big)^{-1}}{2} \matnormBigg{\Bigg(\big(\beta_2^{1/2} \widehat V_{t-1}^{1/2} + \delta I\big) - \big(\widehat V_{t-1}^{1/2} + \delta I\big)\Bigg)}{2} \matnormBigg{\big(\beta_2^{1/2} \widehat V_{t-1}^{1/2} + \delta I\big)^{-1}}{2} \\
=~ & \matnormBigg{\big(\widehat V_t^{1/2} + \delta I\big)^{-1}}{2} \matnormBigg{\widehat{V}_t^{1/2} - \beta_2^{1/2} \widehat{V}_{t-1}^{1/2}}{2} \matnormBigg{\big(\beta_2^{1/2} \widehat V_{t-1}^{1/2} + \delta I\big)^{-1}}{2} \\
\overset{(ii)}{=}~ & \matnormBigg{\big(\widehat V_t^{1/2} + \delta I\big)^{-1}}{2} \matnormBigg{\Big(\widehat{V}_t^{1/2} - \beta_2^{1/2} \widehat{V}_{t-1}^{1/2}\Big)\Big(\widehat{V}_t^{1/2} + \beta_2^{1/2} \widehat{V}_{t-1}^{1/2}\Big)\Big(\widehat{V}_{t}^{1/2} + \beta_2^{1/2}\widehat{V}_{t-1}^{1/2}\Big)^{-1}}{2} \matnormBigg{\big(\beta_2^{1/2} \widehat V_{t-1}^{1/2} + \delta I\big)^{-1}}{2} \\
\overset{(iii)}{\leq}~ & \matnormBigg{\big(\widehat V_t^{1/2} + \delta I\big)^{-1}}{2} \matnormBigg{\big(\beta_2^{1/2} \widehat V_{t-1}^{1/2} + \delta I\big)^{-1}}{2} \matnormBigg{\Big(\widehat{V}_{t}^{1/2} + \beta_2^{1/2}\widehat{V}_{t-1}^{1/2}\Big)^{-1}}{2} \matnormBigg{\Big(\widehat{V}_t^{1/2} - \beta_2^{1/2} \widehat{V}_{t-1}^{1/2}\Big)\Big(\widehat{V}_t^{1/2} + \beta_2^{1/2} \widehat{V}_{t-1}^{1/2}\Big)}{2} \\
\overset{(iv)}{\leq}~ & \matnormBigg{\big(\widehat V_t^{1/2} + \delta I\big)^{-1}}{2} \matnormBigg{\big(\beta_2^{1/2} \widehat V_{t-1}^{1/2} + \delta I\big)^{-1}}{2} \matnormBigg{\Big(\widehat{V}_{t}^{1/2} + \beta_2^{1/2}\widehat{V}_{t-1}^{1/2}\Big)^{-1}}{2} \mathrm{tr}\Bigg(\Big(\widehat{V}_t^{1/2} - \beta_2^{1/2} \widehat{V}_{t-1}^{1/2}\Big)\Big(\widehat{V}_t^{1/2} + \beta_2^{1/2} \widehat{V}_{t-1}^{1/2}\Big)\Bigg) \\
=~ & \matnormBigg{\big(\widehat V_t^{1/2} + \delta I\big)^{-1}}{2} \matnormBigg{\big(\beta_2^{1/2} \widehat V_{t-1}^{1/2} + \delta I\big)^{-1}}{2} \matnormBigg{\Big(\widehat{V}_{t}^{1/2} + \beta_2^{1/2}\widehat{V}_{t-1}^{1/2}\Big)^{-1}}{2} \mathrm{tr}\Big(\widehat{V}_t - \beta_2 \widehat{V}_{t-1}\Big) \\
=~ & \matnormBigg{\big(\widehat V_t^{1/2} + \delta I\big)^{-1}}{2} \matnormBigg{\big(\beta_2^{1/2} \widehat V_{t-1}^{1/2} + \delta I\big)^{-1}}{2} \matnormBigg{\Big(\widehat{V}_{t}^{1/2} + \beta_2^{1/2}\widehat{V}_{t-1}^{1/2}\Big)^{-1}}{2} (1-\beta_2) \|g_t\|_2^2
\end{align*}
The reasoning follows
\begin{enumerate}[label=(\roman*)]
	\item We use the subordinate property of matrix norms, $\matnorm{AB}{2} \leq \matnorm{A}{2} \matnorm{B}{2}$.
	\item Here is the point we need an assumption $\widehat{V}_t$ should be full-rank after finite time $t_0 \in \mathbb{N}$.
	\item The same reason as (i).
	\item Since the eigenvalues of the matrix $\Big(\widehat{V}_t^{1/2} - \beta_2^{1/2} \widehat{V}_{t-1}^{1/2}\Big)\Big(\widehat{V}_t^{1/2} + \beta_2^{1/2} \widehat{V}_{t-1}^{1/2}\Big)$ are all non-negative by the lemma \ref{lemma10}, we can use the inequality $\matnorm{A}{2} \leq \mathrm{tr}(A)$.
\end{enumerate}
Then, we can bound $T_1$ as
\begin{align*}
T_1 \leq~ & \matnormBigg{\big(\widehat V_t^{1/2} + \delta I\big)^{-1}}{2} \matnormBigg{\big(\beta_2^{1/2} \widehat V_{t-1}^{1/2} + \delta I\big)^{-1}}{2} \matnormBigg{\Big(\widehat{V}_{t}^{1/2} + \beta_2^{1/2}\widehat{V}_{t-1}^{1/2}\Big)^{-1}}{2} (1-\beta_2) \|g_t\|_2^2 \|m_t\|_2 \\
=~ & \matnormBigg{\big(\widehat V_t^{1/2} + \delta I\big)^{-1}}{2} \matnormBigg{\big(\beta_2^{1/2} \widehat V_{t-1}^{1/2} + \delta I\big)^{-1}}{2} \matnormBigg{\Big((\beta_2 \widehat{V}_{t-1} + (1 - \beta_2)g_t g_t^T)^{1/2} + \beta_2^{1/2}\widehat{V}_{t-1}^{1/2}\Big)^{-1}}{2} (1-\beta_2) \|g_t\|_2^2 \|m_t\|_2 \\
\overset{(i)}{\leq}~ & \frac{1}{\delta} \matnormBigg{\big(\beta_2^{1/2} \widehat V_{t-1}^{1/2} + \delta I\big)^{-1}}{2} \matnormBigg{\Big((\beta_2 \widehat{V}_{t-1} + (1 - \beta_2)g_t g_t^T)^{1/2}\Big)^{-1}}{2} (1-\beta_2) \|g_t\|_2^2 \|m_t\|_2 \\
=~ & \frac{1}{\delta} \matnormBigg{\big(\beta_2^{1/2} \widehat V_{t-1}^{1/2} + \delta I\big)^{-1}}{2} \frac{1}{\lambda_{\mathrm{min}}\Big((\beta_2 \widehat{V}_{t-1} + (1 - \beta_2)g_t g_t^T)^{1/2}\Big)} (1-\beta_2) \|g_t\|_2^2 \|m_t\|_2 \\
=~ & \frac{1}{\delta} \matnormBigg{\big(\beta_2^{1/2} \widehat V_{t-1}^{1/2} + \delta I\big)^{-1}}{2} \frac{1}{\lambda_{\mathrm{min}}\Big(\beta_2 \widehat{V}_{t-1} + (1 - \beta_2)g_t g_t^T\Big)^{1/2}} (1-\beta_2) \|g_t\|_2^2 \|m_t\|_2 \\
\overset{(ii)}{\leq}~ & \frac{1}{\delta} \matnormBigg{\big(\beta_2^{1/2} \widehat V_{t-1}^{1/2} + \delta I\big)^{-1}}{2} \frac{1}{\Big[\lambda_{\mathrm{min}}\Big(\beta_2 \widehat{V}_{t-1}\Big) + \lambda_{\mathrm{min}}\Big((1 - \beta_2)g_t g_t^T\Big)\Big]^{1/2}} (1-\beta_2) \|g_t\|_2^2 \|m_t\|_2 \\
\overset{(iii)}{\leq}~ & \frac{1}{\delta} \matnormBigg{\big(\beta_2^{1/2} \widehat V_{t-1}^{1/2} + \delta I\big)^{-1}}{2} \sqrt{1 - \beta_2} \|g_t\|_2 \|m_t\|_2
\end{align*}
The reasoning follows
\begin{enumerate}[label=(\roman*)]
	\item We use the fact $(\widehat{V}_t^{1/2} + \delta I)^{-1} \preceq (\delta I)^{-1}$ and $\Big((\beta_2 \widehat{V}_{t-1} + (1 - \beta_2)g_t g_t^T)^{1/2} + \beta_2^{1/2}\widehat{V}_{t-1}^{1/2}\Big)^{-1} \preceq \Big((\beta_2 \widehat{V}_{t-1} + (1 - \beta_2)g_t g_t^T)^{1/2}\Big)^{-1}$.
	\item By Weyl's theorem~\cite{horn2012} on eigenvalues, we have $\lambda_{\mathrm{min}}(A + B) \geq \lambda_{\mathrm{min}}(A) + \lambda_{\mathrm{min}}(B)$.
	\item We use the fact, $\lambda_{\mathrm{min}}\Big(\beta_2 \widehat{V}_{t-1}\Big) + \lambda_{\mathrm{min}}\Big((1 - \beta_2)g_t g_t^T\Big) \geq \lambda_{\mathrm{min}}\Big((1 - \beta_2)g_t g_t^T\Big) = (1 - \beta_2) \|g_t\|^2$.
\end{enumerate}
Therefore, we can bound
\begin{align*}
\mathbb{E}_t\big[f(x_{t+1})\big] \leq~ & f(x_t) - \alpha_t (1 - \beta_{1,t}) \|\nabla f(x_t)\|_{(\beta_2^{1/2} \widehat V_{t-1}^{1/2} + \delta I)^{-1}}^2 \\
& ~~~~+ \alpha_t \|\nabla f(x_t)\|_2 \mathbb{E}_t \underbrace{\Big\|\big[(\widehat V_t^{1/2} + \delta I)^{-1} - (\beta_2^{1/2} \widehat V_{t-1}^{1/2} + \delta I)^{-1}\big]m_t \Big\|_2}_{T_1} \\
& ~~~~+ \frac{\alpha_t^2 L}{2} \sum\limits_{i=1}^{d} \mathbb{E}\big[\widetilde m_{t,i}^2\big] + \frac{\alpha_t \beta_{1,t} G_\infty^2}{\delta}\\
\overset{(i)}{\leq}~ & f(x_t) - \alpha_t (1 - \beta_{1,t}) \|\nabla f(x_t)\|_{(\beta_2^{1/2} \widehat V_{t-1}^{1/2} + \delta I)^{-1}}^2 + \frac{\alpha_t G \sqrt{1 - \beta_2}}{\delta} \mathbb{E}_t \Bigg[\matnormBigg{\big(\beta_2^{1/2} \widehat V_{t-1}^{1/2} + \delta I\big)^{-1}}{2} \|g_t\|_2 \|m_t\|_2\Bigg] \\
& ~~~~+ \frac{\alpha_t^2 L}{2} \sum\limits_{i=1}^{d} \mathbb{E}\big[\widetilde m_{t,i}^2\big] + \frac{\alpha_t \beta_{1,t} G_\infty^2}{\delta} \\
\overset{(ii)}{\leq}~ & f(x_t) - \lambda_{\mathrm{min}}\Big((\beta_2^{1/2} \widehat V_{t-1}^{1/2} + \delta I)^{-1}\Big)\alpha_t (1 - \beta_{1,t}) \|\nabla f(x_t)\|^2 \\
& ~~~~+ \frac{\alpha_t G_\infty \sqrt{1 - \beta_2}}{\delta} \matnormBigg{\big(\beta_2^{1/2} \widehat V_{t-1}^{1/2} + \delta I\big)^{-1}}{2} \mathbb{E}_t \big[\|g_t\|_2 \|m_t\|_2\big] \\
& ~~~~+ \frac{\alpha_t^2 L}{2} \sum\limits_{i=1}^{d} \mathbb{E}\big[\widetilde m_{t,i}^2\big] + \frac{\alpha_t \beta_{1,t} G_\infty^2}{\delta}
\end{align*}
The reasoning follows
\begin{enumerate}[label=(\roman*)]
	\item We use our bound derivation of $T_1$.
	\item For any vector $x$ and positive definite matrix $A$, we have $x^T Ax \geq \lambda_{\mathrm{min}}(A) \|x\|^2$.
\end{enumerate}
Lastly, we bound
\begin{align*}
\frac{\alpha_t^2 L}{2} \sum\limits_{i=1}^{d} \mathbb{E}\big[\widetilde m_{t,i}^2\big] & = \frac{\alpha_t^2 L}{2} \mathbb{E}\big[\| (\widehat V_t^{1/2} + \delta I)^{-1} m_t \|^2\big] \\
& \leq \frac{\alpha_t^2 L}{2} \mathbb{E}\Bigg[\underbrace{\matnormBigg{(\widehat V_t^{1/2} + \delta I)^{-1}}{2}^2}_{T_3} \| m_t \|_2^2\Bigg]
\end{align*}
The $T_3$ can be bound
\begin{align*}
T_3 & = \matnormBigg{(\widehat V_t^{1/2} + \delta I)^{-1}}{2}^2 \\
& = \matnormBigg{(\widehat V_t^{1/2} + \delta I)^{-1}}{2} \matnormBigg{(\widehat V_t^{1/2} + \delta I)^{-1}}{2} \\
& \leq \matnormBigg{(\delta I)^{-1}}{2} \matnormBigg{\Big(\big(\beta_2\widehat{V}_{t-1} + (1 - \beta_2)g_t g_t^T\big)^{1/2} + \delta I\Big)^{-1}}{2} \\
& \leq \frac{1}{\delta} \matnormBigg{\Big(\beta_2^{1/2} \widehat{V}_{t-1}^{1/2} + \delta I\Big)^{-1}}{2}
\end{align*}
Therefore, we can obtain
\begin{align}{\label{tilde_bound}}
\frac{\alpha_t^2 L}{2} \sum\limits_{i=1}^{d} \mathbb{E}\big[\widetilde g_{t,i}^2\big] \leq \frac{\alpha_t^2 L}{2 \delta} \matnormBigg{\Big(\beta_2^{1/2} \widehat{V}_{t-1}^{1/2} + \delta I\Big)^{-1}}{2} \mathbb{E}_t\big[ \| m_t \|_2^2 \big]
\end{align}
By putting altogether, we have
\begin{align*}
\mathbb{E}_t\big[f(x_{t+1})\big] \leq~ & f(x_t) - \lambda_{\mathrm{min}}\Big((\beta_2^{1/2} \widehat V_{t-1}^{1/2} + \delta I)^{-1}\Big)\alpha_t (1 - \beta_{1,t}) \|\nabla f(x_t)\|^2 \\
& ~~~~+ \frac{\alpha_t G_\infty \sqrt{1 - \beta_2}}{\delta} \matnormBigg{\big(\beta_2^{1/2} \widehat V_{t-1}^{1/2} + \delta I\big)^{-1}}{2} \mathbb{E}_t \big[\|g_t\|_2 \|m_t\|_2\big] + \frac{\alpha_t^2 L}{2} \sum\limits_{i=1}^{d} \mathbb{E}\big[\widetilde m_{t,i}^2\big] + \frac{\alpha_t \beta_{1,t} G_\infty^2}{\delta} \\
\overset{(i)}{\leq}~ & f(x_t) - \lambda_{\mathrm{min}}\Big((\beta_2^{1/2} \widehat V_{t-1}^{1/2} + \delta I)^{-1}\Big)\alpha_t (1 - \beta_{1,t}) \|\nabla f(x_t)\|^2 \\
& ~~~~+ \frac{\alpha_t G_\infty \sqrt{1 - \beta_2}}{\delta} \matnormBigg{\big(\beta_2^{1/2} \widehat V_{t-1}^{1/2} + \delta I\big)^{-1}}{2} \mathbb{E}_t \big[\|g_t\|_2 \|m_t\|_2\big] + \frac{\alpha_t^2 L}{2 \delta} \matnormBigg{\Big(\beta_2^{1/2} \widehat{V}_{t-1}^{1/2} + \delta I\Big)^{-1}}{2} \mathbb{E}_t\big[ \| m_t \|_2^2 \big] \\
& ~~~~+ \frac{\alpha_t \beta_{1,t} G_\infty^2}{\delta} \\
\overset{(ii)}{\leq}~ & f(x_t) - \lambda_{\mathrm{min}}\Big((\beta_2^{1/2} \widehat V_{t-1}^{1/2} + \delta I)^{-1}\Big)\alpha_t (1 - \beta_{1,t}) \|\nabla f(x_t)\|^2 \\
& ~~~~+ \frac{\alpha_t G_\infty \sqrt{1 - \beta_2}}{\delta} \matnormBigg{\big(\beta_2^{1/2} \widehat V_{t-1}^{1/2} + \delta I\big)^{-1}}{2} \Bigg(\beta_{1,t} G_\infty^2 + (1 - \beta_{1,t}) \Big(\frac{\sigma^2}{M} + \|\nabla f(x_t)\|_2^2 \Big)\Bigg) \\
& ~~~~+ \frac{\alpha_t^2 L}{2 \delta} \matnormBigg{\Big(\beta_2^{1/2} \widehat{V}_{t-1}^{1/2} + \delta I\Big)^{-1}}{2} \Bigg(\beta_{1,t}^2 G_\infty^2 + (1 - \beta_{1,t})^2 \Big(\frac{\sigma^2}{M} + \|\nabla f(x_t)\|_2^2 \Big)\Bigg) + \frac{\alpha_t \beta_{1,t} G_\infty^2}{\delta} \\
=~ & f(x_t) - \Bigg(\alpha_t (1 - \beta_{1,t}) \lambda_{\mathrm{min}}\Big((\beta_2^{1/2} \widehat V_{t-1}^{1/2} + \delta I)^{-1}\Big) - \frac{\alpha_t G_\infty (1 - \beta_{1,t})\sqrt{1 - \beta_2}}{\delta} \matnormBigg{\big(\beta_2^{1/2} \widehat V_{t-1}^{1/2} + \delta I\big)^{-1}}{2} \\
& ~~~~~~~~~~~~~~~~- \frac{\alpha_t^2 (1 - \beta_{1,t})^2 L}{2\delta} \matnormBigg{\big(\beta_2^{1/2} \widehat V_{t-1}^{1/2} + \delta I\big)^{-1}}{2} \Bigg) \|\nabla f(x_t)\|^2 \\
& ~~~~~~~~+ \frac{\sigma^2}{M} \Bigg( \frac{\alpha_t (1 - \beta_{1,t}) G_\infty \sqrt{1 - \beta_2}}{\delta} \matnormBigg{\big(\beta_2^{1/2} \widehat V_{t-1}^{1/2} + \delta I\big)^{-1}}{2} + \frac{\alpha_t^2 (1 - \beta_{1,t})^2 L}{2\delta} \matnormBigg{\big(\beta_2^{1/2} \widehat V_{t-1}^{1/2} + \delta I\big)^{-1}}{2} \Bigg) \\
& ~~~~~~~~+ \frac{\alpha_t \beta_{1,t} G_\infty^2}{\delta} + \frac{\alpha_t \beta_{1,t} G_\infty^3 \sqrt{1 - \beta_2}}{\delta} \matnormBigg{\Big(\beta_2^{1/2} \widehat{V}_{t-1}^{1/2} + \delta I\Big)^{-1}}{2} + \frac{\alpha_t^2 \beta_{1,t}^2 G_\infty^2 L}{2\delta} \matnormBigg{\Big(\beta_2^{1/2} \widehat{V}_{t-1}^{1/2} + \delta I\Big)^{-1}}{2} \\
\overset{(iii)}{\leq}~ & f(x_t) - \alpha_t(1 - \beta_{1,t})\Bigg(1 - \frac{G_\infty \sqrt{1 - \beta_2}}{\delta} \kappa_{\text{max}} - \frac{\alpha_t (1 - \beta_{1,t}) L}{2\delta} \kappa_{\text{max}} \Bigg) \lambda_{\text{min}}\Big((\beta_2^{1/2} \widehat V_{t-1}^{1/2} + \delta I)^{-1}\Big) \|\nabla f(x_t)\|^2 \\
& ~~~~+ \frac{\sigma^2}{M} \Bigg( \frac{\alpha_t (1 - \beta_{1,t}) G_\infty \sqrt{1 - \beta_2}}{\delta} \matnormBigg{\big(\beta_2^{1/2} \widehat V_{t-1}^{1/2} + \delta I\big)^{-1}}{2} + \frac{\alpha_t^2 (1 - \beta_{1,t})^2 L}{2\delta} \matnormBigg{\big(\beta_2^{1/2} \widehat V_{t-1}^{1/2} + \delta I\big)^{-1}}{2} \Bigg) \\
& ~~~~+ \frac{\alpha_t \beta_{1,t} G^2}{\delta} + \frac{\alpha_t \beta_{1,t} G_\infty^3 \sqrt{1 - \beta_2}}{\delta} \matnormBigg{\Big(\beta_2^{1/2} \widehat{V}_{t-1}^{1/2} + \delta I\Big)^{-1}}{2} + \frac{\alpha_t^2 \beta_{1,t}^2 G_\infty^2 L}{2\delta} \matnormBigg{\Big(\beta_2^{1/2} \widehat{V}_{t-1}^{1/2} + \delta I\Big)^{-1}}{2} \\
\overset{(iv)}{\leq}~ & f(x_t) - \alpha_t(1 - \beta_{1,t})\Bigg(1 - \frac{G_\infty \sqrt{1 - \beta_2}}{\delta} \kappa_{\text{max}} - \frac{\alpha_t (1 - \beta_{1,t}) L}{2\delta} \kappa_{\text{max}} \Bigg) \lambda_{\text{min}}\Big((\beta_2^{1/2} \widehat V_{t-1}^{1/2} + \delta I)^{-1}\Big) \|\nabla f(x_t)\|^2 \\
& ~~~~+ \frac{\sigma^2}{M} \Bigg(\frac{\alpha_t (1 - \beta_{1,t}) G_\infty \sqrt{1 - \beta_2}}{\delta^2} + \frac{\alpha_t^2 (1 - \beta_{1,t})^2 L}{2\delta^2} \Bigg) + \frac{\alpha_t \beta_{1,t} G_\infty^2}{\delta} + \frac{\alpha_t \beta_{1,t} G_\infty^3 \sqrt{1 - \beta_2}}{\delta^2} + \frac{\alpha_t^2 \beta_{1,t}^2 G_\infty^2 L}{2\delta^2}
\end{align*}
\begin{enumerate}[label=(\roman*)]
	\item We use the derived bound for \ref{tilde_bound}.
	\item We use the Lemma \ref{lemma12} and \ref{lemma13}.
	\item This is the key part. By our assumption on the $\kappa(\beta_2^{1/2} \widehat{V}_{t-1}^{1/2} + \delta I) \leq \kappa_{\mathrm{max}}$, we have
	\begin{align*}
	\frac{\matnormBig{(\beta_2^{1/2} \widehat{V}_{t-1}^{1/2} + \delta I)^{-1}}{2}}{\lambda_{\mathrm{min}}\Big((\beta_2^{1/2} \widehat{V}_{t-1}^{1/2} + \delta I)^{-1}\Big)} \leq \kappa_{\mathrm{max}}
	\end{align*}
	\item We use $(\beta_2^{1/2} \widehat{V}_{t-1}^{1/2} + \delta I)^{-1} \preceq (\delta I)^{-1}$.
\end{enumerate}
By our assumptions on $\kappa_{\text{max}}$, we have
\begin{align}{\label{kappa_max}}
\frac{G_\infty \sqrt{1 - \beta_2}}{\delta} \kappa_{\text{max}} \leq \frac{1}{3} \\
\frac{\alpha_t L}{2\delta} \kappa_{\text{max}} \leq \frac{1}{3}
\end{align}
Then, we have
\begin{align*}
\mathbb{E}_t\big[f(x_{t+1})\big] \leq~ & f(x_t) - \alpha_t(1 - \beta_{1,t})\Bigg(1 - \frac{G_\infty \sqrt{1 - \beta_2}}{\delta} \kappa_{\text{max}} - \frac{\alpha_t (1 - \beta_{1,t}) L}{2\delta} \kappa_{\text{max}} \Bigg) \lambda_{\mathrm{min}}\Big((\beta_2^{1/2} \widehat V_{t-1}^{1/2} + \delta I)^{-1}\Big) \|\nabla f(x_t)\|^2 \\
& ~~~~+ \frac{\sigma^2}{M} \Bigg( \frac{\alpha_t (1 - \beta_{1,t}) G_\infty \sqrt{1 - \beta_2}}{\delta^2} + \frac{\alpha_t^2 (1 - \beta_{1,t})^2 L}{2\delta^2} \Bigg) + \frac{\alpha_t \beta_{1,t} G_\infty^2}{\delta} + \frac{\alpha_t \beta_{1,t} G_\infty^3 \sqrt{1 - \beta_2}}{\delta^2} + \frac{\alpha_t^2 \beta_{1,t}^2 G_\infty^2 L}{2\delta^2} \\
\overset{(i)}{\leq}~ & f(x_t) - \frac{\alpha_t (1 - \beta_{1,t})}{3} \lambda_{\mathrm{min}}\Big((\beta_2^{1/2} \widehat V_{t-1}^{1/2} + \delta I)^{-1}\Big) \|\nabla f(x_t)\|^2 \\
& ~~~~+ \frac{\sigma^2}{M} \Bigg( \frac{\alpha_t G_\infty \sqrt{1 - \beta_2}}{\delta^2} + \frac{\alpha_t^2 L}{2\delta^2} \Bigg) + \frac{\alpha_t \beta_{1,t} G_\infty^2}{\delta} + \frac{\alpha_t \beta_{1,t} G^3 \sqrt{1 - \beta_2}}{\delta^2} + \frac{\alpha_t^2 \beta_{1,t}^2 G_\infty^2 L}{2\delta^2} \\
=~ & f(x_t) - \frac{\alpha_t (1 - \beta_{1,t})}{3} \frac{1}{\lambda_{\mathrm{max}}\Big(\beta_2^{1/2} \widehat V_{t-1}^{1/2} + \delta I\Big)} \|\nabla f(x_t)\|^2 \\
& ~~~~+ \frac{\sigma^2}{M} \Bigg( \frac{\alpha_t G_\infty \sqrt{1 - \beta_2}}{\delta^2} + \frac{\alpha_t^2 L}{2\delta^2} \Bigg) + \frac{\alpha_t \beta_{1,t} G_\infty^2}{\delta} + \frac{\alpha_t \beta_{1,t} G_\infty^3 \sqrt{1 - \beta_2}}{\delta^2} + \frac{\alpha_t^2 \beta_{1,t}^2 G_\infty^2 L}{2\delta^2} \\
\overset{(ii)}{\leq}~ & f(x_t) - \frac{\alpha_t (1 - \beta_{1,t})}{3(\sqrt{\beta_2}G + \delta)} \|\nabla f(x_t)\|^2 + \frac{\sigma^2}{M} \Bigg( \frac{\alpha_t G \sqrt{1 - \beta_2}}{\delta^2} + \frac{\alpha_t^2 L}{2\delta^2} \Bigg) \\
& ~~~~+ \frac{\alpha_t \beta_{1,t} G^2}{\delta} + \frac{\alpha_t \beta_{1,t} G^3 \sqrt{1 - \beta_2}}{\delta^2} + \frac{\alpha_t^2 \beta_{1,t}^2 G^2 L}{2\delta^2}
\end{align*}
The reasoning follows
\begin{enumerate}[label=(\roman*)]
	\item From our assumptions \ref{kappa_max} on $\kappa_{\mathrm{max}}$, we can get the desired inequality.
	\item By Weyl's theorem on eigenvalues, we have $\lambda_{\mathrm{max}}(A + B) \leq \lambda_{\mathrm{max}}(A) + \lambda_{\mathrm{max}}(B)$.
\end{enumerate}
For the constant stepsize $\alpha_t = \alpha$, we telescope from $t = 1$ to $t_0$,
\begin{align*}
\frac{\alpha (1 - \beta_1)}{3(\sqrt{\beta_2}G + \delta)} \sum\limits_{t=1}^{T} \|\nabla f(x_t)\|^2 & \leq \frac{\alpha}{3(\sqrt{\beta_2}G + \delta)} \sum\limits_{t=1}^{T} (1 - \beta_{1,t})\|\nabla f(x_t)\|^2 \\
& \leq f(x_{t_0}) - \mathbb{E}\big[f(x_{T+1})\big] + \frac{T \sigma^2}{M} \Bigg( \frac{\alpha G \sqrt{1 - \beta_2}}{\delta^2} + \frac{\alpha^2 L}{2\delta^2} \Bigg) \\
& ~~~~~~~~+ \sum\limits_{t=1}^{T} \Bigg[\frac{\alpha \beta_{1,t} G^2}{\delta} + \frac{\alpha \beta_{1,t} G^3 \sqrt{1 - \beta_2}}{\delta^2} + \frac{\alpha \beta_{1,t}^2 G^2 L}{2\delta^2}\Bigg] + \sum\limits_{t=1}^{t_0 - 1} \|\nabla f(x_t)\|^2 \\
& \leq f(x_{t_0}) - f(x^{*}) + \frac{T \sigma^2}{M} \Bigg( \frac{\alpha G \sqrt{1 - \beta_2}}{\delta^2} + \frac{\alpha^2 L}{2\delta^2} \Bigg) + C
\end{align*}
where $C$ is defined as
\begin{align*}
C = \frac{\alpha \beta_1 G_\infty^2}{\delta (1 - \lambda)} \Big(1 + \frac{G_\infty^2 \sqrt{1 - \beta_2}}{\delta} + \frac{\beta_1 G_\infty L}{\delta (1 + \lambda)}\Big) + \sum\limits_{t=1}^{t_0 - 1} \|\nabla f(x_t) \|^2
\end{align*}
which is independent of $T$. Dividing both sides by $\frac{\alpha (1 - \beta_1)}{3(\sqrt{\beta_2}G_\infty + \delta)}$ yields 
\begin{align*}
\frac{1}{T} \sum\limits_{t=1}^{T} \| \nabla f(x_t) \|^2 & \leq \frac{3(\sqrt{\beta_2} G_\infty + \delta)}{1 - \beta_1} \Bigg[ \frac{f(x_1) - f(x^{*})}{\alpha T} + \frac{\sigma^2}{M}\Big(\frac{G \sqrt{1 - \beta_2}}{\delta^2} + \frac{\alpha L}{2\delta^2}\Big) + \frac{C}{T}\Bigg] \\
& = \mathcal{O}\Big(\frac{f(x_{t_0}) - f(x^{*})}{\alpha T} + \frac{\sigma^2}{M}\Big)
\end{align*}

\end{document}